\documentclass{article}

\PassOptionsToPackage{authoryear,round}{natbib}
\usepackage[preprint]{neurips_2026}

\usepackage[utf8]{inputenc} % allow utf-8 input
\usepackage[T1]{fontenc}    % use 8-bit T1 fonts
\usepackage[hypertexnames=false]{hyperref}       % hyperlinks
\usepackage{url}            % simple URL typesetting
\usepackage{booktabs}       % professional-quality tables
\usepackage{amsfonts}       % blackboard math symbols
\usepackage{nicefrac}       % compact symbols for 1/2, etc.
\usepackage{microtype}      % microtypography
\usepackage{xcolor}         % colors

\usepackage{subcaption}

\usepackage{amsmath}
\usepackage{amssymb}
\usepackage{mathtools}
\usepackage{amsthm}
\usepackage{wrapfig}

\usepackage{bbm}
\usepackage{csquotes}
\usepackage{enumitem}  

\usepackage[capitalize,noabbrev]{cleveref}

\usepackage[textsize=tiny]{todonotes}

\usepackage{graphicx}
\usepackage{apptools}
\usepackage[flushleft]{threeparttable}
\usepackage{array,booktabs,makecell}
\usepackage{multirow}
\usepackage{nicefrac}
\usepackage{amsthm}
\usepackage{amsmath,amsfonts,bm,amssymb,mathtools}
\usepackage{caption}
\usepackage{siunitx}
\usepackage{thm-restate}
\usepackage{nccmath}
\usepackage{empheq}
\usepackage{suffix}
\usepackage{tabularx}
\usepackage{mathrsfs}
\usepackage{derivative}
\usepackage[dvipsnames]{xcolor}

\usepackage{aligned-overset}
\newcommand{\oversetref}[2]{\overset{{\scriptscriptstyle\mathrm{#1}\!~\text{#2}}}}
\newcommand{\oversetlab}[1]{\overset{{\scriptscriptstyle\text{#1}}}}

\newcounter{relctr} %% <- counter for relations
\newcounter{relgroup} %% <- counter for relations

\everydisplay\expandafter{\the\everydisplay\setcounter{relctr}{0}}%% <- reset every eq
\AtBeginDocument{\let\originallabel\label} %% <- store original definition

\makeatletter
\newcommand\oversetrel[1]{%
    \ifnum\value{relctr}=0\relax%
      \stepcounter{relgroup}%
    \fi%
    \refstepcounter{relctr}%
    \phantomsection%
    \originallabel{rel:#1@\therelgroup}%
    \overset{\scriptscriptstyle\text{(\alph{relctr})}}%
}

\newcommand\relref[1]{%
    \@ifundefined{r@rel:#1@\arabic{relgroup}}{%
        \eqref{}%
    }{%
        \eqref{rel:#1@\arabic{relgroup}}% Always point to (a)
    }%
}
\makeatother

\makeatletter
\DeclareRobustCommand\widecheck[1]{{\mathpalette\@widecheck{#1}}}
\def\@widecheck#1#2{%
    \setbox\z@\hbox{\m@th$#1#2$}%
    \setbox\tw@\hbox{\m@th$#1%
       \widehat{%
          \vrule\@width\z@\@height\ht\z@
          \vrule\@height\z@\@width\wd\z@}$}%
    \dp\tw@-\ht\z@
    \@tempdima\ht\z@ \advance\@tempdima2\ht\tw@ \divide\@tempdima\thr@@
    \setbox\tw@\hbox{%
       \raise\@tempdima\hbox{\scalebox{1}[-1]{\lower\@tempdima\box
\tw@}}}%
    {\ooalign{\box\tw@ \cr \box\z@}}}
\makeatother

\usepackage{algorithmic}
\usepackage{color}
\usepackage{colortbl}
\definecolor{bgcolor}{rgb}{0.76,0.88,0.50}
\definecolor{bgcolor0}{rgb}{0.93,0.99,1}
\definecolor{bgcolor1}{rgb}{0.8,1,1}
\definecolor{bgcolor2}{rgb}{0.8,1,0.8}
\definecolor{bgcolor3}{rgb}{0.50,0.90,0.50}
\usepackage{tcolorbox}
\usepackage{pifont}

\definecolor{mydarkgreen}{RGB}{39,130,67}
\definecolor{mydarkorange}{RGB}{236,147,14}
\definecolor{mydarkred}{RGB}{192,47,25}
\definecolor{ruby}{RGB}{155,17,30}
\definecolor{chili}{RGB}{191,0,0}
\definecolor{sangria}{RGB}{146,0,10}
\definecolor{burgundy}{RGB}{128,0,32} 
\definecolor{darkred}{RGB}{132,0,0} 
\definecolor{cherry}{RGB}{192,0,0} 

\definecolor{blue}{RGB}{0,0,255}

\usepackage[textsize=tiny]{todonotes}

\usepackage{xspace}

\newcommand{\norm}[1]{\left\| #1 \right\|}

\newcommand{\abs}[1]{\left| #1 \right|}

\newcommand{\R}{\mathbb{R}} % reals
\newcommand{\N}{\mathbb{N}} % reals
\newcommand{\E}[1]{\mathbb{E}\left[#1\right]}

\newcommand{\cA}{\mathcal{A}}

\newcommand{\cE}{\mathcal{E}}

\newcommand{\cK}{\mathcal{K}}
\newcommand{\cL}{\mathcal{L}}

\newcommand{\cP}{\mathcal{P}}
\newcommand{\cQ}{\mathcal{Q}}

\newcommand{\cU}{\mathcal{U}}

\newcommand{\eqdef}{:=}
\WithSuffix\newcommand\eqdef*{:\!&=}

\makeatletter
\newcommand{\vast}{\bBigg@{4}}

\def\<{\left\langle}
\def\>{\right\rangle}
\def\({\left(}
\def\){\right)}

\usepackage{thmtools}
\usepackage{thm-restate}

\usepackage{amsthm}
\theoremstyle{plain}
\newtheorem{theorem}{Theorem}[section]

\newtheorem{lemma}{Lemma}[section]
\newtheorem{corollary}[theorem]{Corollary}
\theoremstyle{definition}
\newtheorem{definition}{Definition}[section]

\theoremstyle{remark}
\newtheorem{remark}{Remark}[section]

\usepackage{cleveref}

\crefname{assumption}{assumption}{assumptions}
\Crefname{assumption}{Assumption}{Assumptions}
\creflabelformat{assumption}{#2#1#3}

\crefname{condition}{condition}{conditions}
\creflabelformat{condition}{#2#1#3}

\crefname{observation}{observation}{observations}
\creflabelformat{observation}{#2#1#3}

\theoremstyle{plain}
\newtheorem{innercustomthm}{Theorem}
\newenvironment{restate-theorem}[1]
{\renewcommand\theinnercustomthm{#1}\innercustomthm}
{\endinnercustomthm}

\newtheorem{innercustomlemma}{Lemma}
\newenvironment{restate-lemma}[1]
{\renewcommand\theinnercustomlemma{#1}\innercustomlemma}
{\endinnercustomlemma}

\newtheorem{innercustomcorollary}{Corollary}
\newenvironment{restate-corollary}[1]
{\renewcommand\theinnercustomcorollary{#1}\innercustomcorollary}
{\endinnercustomcorollary}

\newtheorem{innercustomproposition}{Proposition}
\newenvironment{restate-proposition}[1]
{\renewcommand\theinnercustomproposition{#1}\innercustomproposition}
{\endinnercustomproposition}

\theoremstyle{definition}
\newtheorem{innercustomdef}{Definition}
\newenvironment{restate-definition}[1]
{\renewcommand\theinnercustomdef{#1}\innercustomdef}
{\endinnercustomdef}

\newcommand*{\sketchproofname}{Sketch of Proof}

\usepackage{longtable}

\newcommand{\supp}{\mathrm{supp}\,}

\DeclareMathSymbol{\mh}{\mathord}{operators}{`\-}

\newcommand*{\lbd}{\lambda}% lambda grec.
\newcommand*{\eps}{\varepsilon}% epsilon grec.

\newcommand*{\Floor}[1]{\left\lfloor #1 \right\rfloor}% partie entière inférieure.
\newcommand*{\Ceil}[1]{\left\lceil #1 \right\rceil}% partie entière supérieure.
\newcommand*{\T}[1]{{#1}^{\scriptscriptstyle\smash{\top}}}% transposée.

\newcommand*{\Tr}{\mathop{\mathrm{Tr}}}% trace.

\newcommand*{\intervalle}[4]{\mathopen{#1}#2\mathclose{}\mathpunct{},#3\mathclose{#4}}%
\newcommand*{\intff}[2]{\intervalle{[}{#1}{#2}{]}}% intervalle fermé.
\newcommand*{\intof}[2]{\intervalle{(}{#1}{#2}{]}}% intervalle ouvert-fermé.
\newcommand*{\intfo}[2]{\intervalle{[}{#1}{#2}{)}}% intervalle ferme-ouvert.
\newcommand*{\intoo}[2]{\intervalle{(}{#1}{#2}{)}}% intervalle ouvert.
\newcommand*{\ens}[1]{\left\{#1\right\}}%
\newcommand*{\enstq}[2]{\left\{#1\,:\,#2\right\}}%
\newcommand*{\Card}[1]{\left|#1\right|}% cardinal.
\renewcommand*{\N}{\mathbb{N}}% entiers naturels.
\newcommand*{\Z}{\mathbb{Z}}% entiers relatifs.
\renewcommand*{\R}{\mathbb{R}}% réels.
\renewcommand*{\T}{\mathbb{T}}% réels.
\newcommand*{\C}{\mathbb{C}}% complexes.
\newcommand*{\K}{\mathbb{K}}% corps, R ou C en général.
\newcommand{\diam}{\delta}
\newcommand{\normop}[1]{{\left\vert\kern-0.25ex\left\vert\kern-0.25ex\left\vert #1 \right\vert\kern-0.25ex\right\vert\kern-0.25ex\right\vert}}% operator norm

\newcommand{\adh}[1]{\overline{#1}}% Adhérence d'une partie.
\newcommand*{\Conv}{\mathop{\mathrm{Conv}}}% Enveloppe convexe.

\newcommand{\ps}[2]{\left\langle #1 , #2 \right\rangle}% produit scalaire.

\renewcommand*{\abs}[1]{\left\lvert #1 \right\rvert}% module / valeur absolue.
\renewcommand*{\lim}{\mathop{\mathrm{lim}}\limits}% limite.

\newcommand\numberthis{\addtocounter{equation}{1}\tag{\theequation}}%number 
\definecolor{jogreen}{rgb}{0.15,.7,0.15}

\newcommand{\johan}[1]{\textcolor{jogreen}{{#1}}}
\newcommand{\abde}[1]{{\color{blue}{\textbf{Abderrahim:} #1}}}
\newcommand{\adri}[1]{{\color{blue}{\textbf{Adrien:} #1}}}

\newcommand{\normgridvec}[1]{\norm{#1}_{\ell^2(\T^d_N)}}
\newcommand{\normgridmattwo}[1]{\normop{#1}_{\ell^2(\T^d_N)}}
\newcommand{\normgridmatone}[1]{\normop{#1}_{\ell^1(\T^d_N)}}
\newcommand{\normsobolev}[1]{\norm{#1}_{H^s}}

\usepackage{xcolor}

\definecolor{mplblue}{RGB}{31,119,180}
\definecolor{mplgreen}{RGB}{44,160,44}

\title{Stability and Discretization Error of State Space Model Neural Operators}

\author{%
  Abderrahim Bendahi %\thanks{Use footnote for providing further information about author (webpage, alternative address)---\emph{not} for acknowledging funding agencies.} 
    \\
  École polytechnique \\
  Paris, France\\
  \texttt{abderrahim.bendahi@polytechnique.edu} \\
   \And
  Adrien Fradin \\
  École polytechnique \\
  Paris, France \\
  \texttt{adrien.fradin@polytechnique.edu} \\
  \AND
  Johan Peralez \\
  Université Claude Bernard Lyon 1, CNRS, LAGEPP UMR 5007 \\
  Villeurbanne, France \\
  \texttt{johan.peralez@univ-lyon1.fr} \\
   \And
  Julie Digne \\
  CNRS, Université Lyon 1, INSA Lyon, LIRIS \\
  Lyon, France \\
  \texttt{julie.digne@cnrs.fr} \\
   \And
  Madiha Nadri \\
  Université Claude Bernard Lyon 1, CNRS, LAGEPP UMR 5007 \\
  Villeurbanne, France \\
  \texttt{madiha.nadri-wolf@univ-lyon1.fr} \\
}

\begin{document}

\maketitle

\begin{abstract}
%Neural operators provide a discretization-invariant framework for learning solution operators of PDEs, but the relation between their continuous formulation and discrete implementation remains incomplete. We study this continuous-to-discrete gap for convolution-based neural operators, with a focus on State Space Neural Operators (SS-NOs). We prove discretization-error bounds linking the approximation error to input regularity, kernel regularity, and activation regularity, extending prior FNO analyses to fractional Sobolev regimes and non-smooth activations. We also establish continuous and discrete stability estimates, including an input-to-state stability bound that quantifies the joint effect of discretization and input perturbations. Experiments in 1D and 2D validate the predicted resolution scaling and illustrate the robustness of SS-NOs under grid refinement.

  Neural operators have emerged as a powerful, discretization-invariant framework for solving partial differential equations (PDEs). Although established approaches like the Deep Operator Network (DeepONet) have successfully achieved universal approximation for operators, and architectures such as Fourier Neural Operators (FNOs) have shown algebraic convergence rates, a precise theoretical connection between the continuous theory and its discrete numerical implementation remains a challenge. Specifically, the relationship between the continuous formulation and the discrete numerical stability has yet to be fully explored. In this paper, we address this gap by establishing theoretical guarantees for the discretization error and stability of neural operator approximation schemes. We prove analytical bounds that link solution regularity to input discretization, providing a formal quantification of neural operator accuracy under real-world numerical constraints. We derive these bounds to the specific cases of State Space Model-based Neural Operators (SS-NOs) and FNOs, thus providing a new discretization error theorem for these models. Additionally, through an input-to-state stability (ISS) analysis, we formally assess the impact of discretization on the stability of SS-NOs results obtained in the continuous domain. 
%We show empirically\mn{ca casse l'ambiance ce Ampirically! ON ne peut pas dire
    Our empirical experiments on 1D and 2D benchmarks validate our theoretical bounds and show the robustness of SS-NOs under varying resolutions.
\end{abstract}

\section{Introduction}

Scientific machine learning increasingly uses neural networks to solve partial differential equations~\cite{brunton2023machine}. Since solving PDEs amounts to learning solution operators, neural operators (NOs)~\cite{Kovachki-NeuralOperator:20223} have emerged as models between infinite-dimensional function spaces that can generalize across grid resolutions. Architectures such as DeepONet~\cite{learning-nonlinear-operators-deeponets}, Fourier Neural Operators (FNOs)~\cite{lifourier}, and State Space Neural Operators (SS-NOs) are defined in the continuous domain, but their implementation requires discretizing the convolution kernel. Most analyses focus on the continuous setting~\cite{kovachki2021universalapproximationerrorbounds}, leaving open how discretization affects accuracy and stability. While discretization error has recently been studied for FNOs~\cite{lanthaler2025discretizationerrorfourierneural}, analogous guarantees for SS-NOs are still missing.

In this work, we bridge this gap through three contributions:
\begin{itemize}
    \setlength{\itemsep}{0pt}
    \item \emph{Continuous stability.} We prove Lipschitz stability of deep SS-NO architectures in infinite-dimensional Hilbert spaces.
    \item \emph{Discretization error.} We derive explicit continuous-to-discrete error bounds that account for the kernel and activation regularity. Compared with prior FNO results~\cite{lanthaler2025discretizationerrorfourierneural}, our analysis covers fractional Sobolev regularity, more general convolution kernels, and non-smooth activations such as \textsc{ReLU} and \textsc{LeakyReLU}.
    \item \emph{Discrete stability.} We prove a global stability bound for the discretized model, including both discretization error and input perturbations.
\end{itemize}

\section{Related Work}

Neural operators aim to approximate mappings between infinite-dimensional function spaces \cite{boulle2024mathematical, lanthaler2025discretizationerrorfourierneural, kovachki2024operator}. Unlike classical neural networks designed for finite-dimensional Euclidean spaces \cite{boulle2024mathematical, Kovachki-NeuralOperator:20223, lifourier}, neural operators are discretization-invariant, allowing them to share model parameters across different grid resolutions and providing significant speedups over traditional numerical solvers \cite{boulle2024mathematical, li2023fourier, lifourier}.

%\paragraph{Foundational Approaches in Neural Operator Learning} A seminal approach in this domain is the Deep Operator Network (DeepONet) \cite{learning-nonlinear-operators-deeponets}, which builds upon the universal approximation theorem for operators using a dual-network architecture consisting of a branch net for input functions and a trunk net for output locations. DeepONets have demonstrated high-order error convergence and significantly reduced generalization errors compared to fully-connected networks when learning dynamic systems and PDEs \cite{learning-nonlinear-operators-deeponets}. Theoretical work has focused so far on proving universal approximation theorems for neural operators formulated as a composition of linear integral operators and non-linear activations \cite{Kovachki-NeuralOperator:20223,boulle2024mathematical}. Since their introduction, many neural operators have been proposed either based on graph kernels~\cite{li2020multipole,anandkumar2020neural}, learning the kernel directly in physical space~\cite{gin2020deepgreendeeplearninggreens}, incorporating physics priors~\cite{li2024physics}, or adapting to different geometries~\cite{Li23-gino}. We focus on two types of neural operators, described below, that are particularly interesting theoretically.

\paragraph{Foundational Approaches in Neural Operator Learning}
Deep Operator Networks (DeepONets) are a seminal neural-operator architecture based on the universal approximation theorem for operators, using a branch net for input functions and a trunk net for output locations~\cite{learning-nonlinear-operators-deeponets}. They achieve high-order error convergence and reduced generalization error compared to fully-connected networks on dynamic systems and PDEs~\cite{learning-nonlinear-operators-deeponets}. Theory has mainly studied universal approximation for neural operators composed of linear integral operators and nonlinear activations~\cite{Kovachki-NeuralOperator:20223,boulle2024mathematical}. Subsequent architectures use graph kernels~\cite{li2020multipole,anandkumar2020neural}, learned physical-space kernels~\cite{gin2020deepgreendeeplearninggreens}, physics priors~\cite{li2024physics}, or geometry-aware designs~\cite{Li23-gino}. We focus on two theoretically relevant neural-operator families below.

\paragraph{Fourier Neural Operators}
The Fourier Neural Operators (FNOs) introduced by \cite{lifourier} parameterize the integral kernel in Fourier space, enabling the modeling of turbulent flows with high efficiency. To handle complex geometries beyond uniform grids, \cite{li2023fourier} proposed Geo-FNO, which uses a learnable deformation to map irregular domains into a canonical latent space. On the theoretical side, \cite{lanthaler2025discretizationerrorfourierneural} established convergence rates for FNOs, quantifying the discretization error between the continuous operator and its numerical implementation. We extend their results to SS-NOs, and beyond discretization error, we study the discrete stability of SS-NOs.

\paragraph{State Space Models and SS-NO}
In parallel, \cite{gu2022parameterization} introduced S4D, a simplified diagonal state space model (SSM) that captures long-range dependencies with minimal computational overhead.
Building on this, the State Space Neural Operator (SS-NO) \cite{anonymous2025merging} extends SSMs to joint spatiotemporal modeling.
This architecture introduces adaptive damping and learnable frequency modulation, proving a universality theorem for full field-of-view convolutional architectures while maintaining high parameter efficiency on benchmarks like Navier-Stokes \cite{anonymous2025merging}.
Related efforts in computational efficiency also include generative frameworks like LD3 \cite{tong2025learningdiscretizedenoisingdiffusion}, which optimize time discretization to improve sampling efficiency in diffusion probabilistic models without retraining.
However, a rigorous theoretical quantification of the discretization error and numerical stability for SS-NO architectures is currently missing. We address this gap by establishing explicit error bounds linked to the kernel parameters.

\section{Preliminaries and Continuous Stability}
We introduce the neural-operator architecture and the FNO/SS-NO kernels, then prove continuous stability for SS-NO layers on the torus.

%In this section, we briefly introduce the foundational definitions and architectural framework for %structured  state-space models 
%SSMs, with a specific focus on the SS-NO and 
%Fourier neural operator (FNO)
%FNO architectures. Then we provide a stability analysis for SS-NOs kernel under periodic boundary conditions.

\subsection{Definitions and Notation}

\paragraph{Notations.} 
We let $\N := \ens{1, 2, \ldots}$, $\T^d \cong [0, 1)^d$ the $d$-dimensional torus, and $\mathcal{A}$, $\mathcal{U}$ are the Banach spaces of input and output functions respectively. For $M \in \R^{H \times H}$, $\normop{M} := \sup_{\norm{x}_2 = 1} \norm{Mx}_2$ is the operator norm and $\norm{M}_F := \sqrt{\Tr( M M^\top )}$ is the Frobenius norm. Standard $L^p(\T^d)$ and Sobolev $H^s(\T^d)$ norms are denoted $\norm{\cdot}_{L^p}$ and $\norm{\cdot}_{H^s}$. We define the discrete torus as the uniform grid $\T^d_N := \frac{1}{N} \ens{0, \ldots, N-1}^d = \frac{1}{N} [N]^d$, which we equip, for $1 \le p < +\infty$, with the norm $\norm{f}_{\ell^p(\T_N^d)} := \left( \sum_{x \in \T_N^d} \norm{f(x)}^p \right)^{1/p}$. See~\Cref{appdx-sec:function-spaces} for more details on the function spaces considered.% in this work.

Building on the work of~\citet{lanthaler2025discretizationerrorfourierneural}, we consider neural operators acting on the torus $\T^d$:
% Consistent with recent advancements in operator learning, the following definitions extend the mathematical formalism of \citet{lanthaler2025discretizationerrorfourierneural} 
% to convolution-based neural operator architectures, specifically the SS-NO architecture.
% For now on, unless otherwise stated, we exclusively work on the $d$--dimensional torus $\T^d$ (assuming periodic boundary conditions).

\begin{definition}[Neural Operator]\label{def:neural-operator}
    Let $\cA$ and $\cU$ be Banach spaces of functions on $\T^d$. The operator $\Psi_{\theta} \colon \cA \to \cU$ is composed of lifting ($\cP$), projection ($\cQ$), and $T$ layers:
    \begin{align}
            \Psi_{\theta} :\!&= \cQ \circ \cL_{T - 1} \circ \cdots \circ \cL_0 \circ \cP \numberthis\label{eq:psi-def} \\
            v_{t + 1} &= \cL_t v_t := \sigma_t \left( W_t v_t + \cK_t * v_t + b_t \right), \numberthis\label{def:ssno-layer}
        \end{align}
    for $t = 0, \ldots, T - 1$, with $v_0 = \cP(a)$. $\cP \colon \R^{d_a} \to \R^{d_0}$ and $\cQ \colon \R^{d_T} \to \R^{d_u}$ are shallow networks with smooth ($\mathscr{C}^{\infty}$) Lipschitz activations, $\sigma_t$ is a $L_{\sigma}$--Lipschitz continuous activation function, $W_t \in \R^{d_{t + 1} \times d_t}$ and $b_t \in \R^{d_{t + 1}}$ are weights and biases, $\cK_t$ is the convolution kernel, %It is defined via its Fourier transform $\widehat{\cK}_t(\xi) = \int_{\T^d} \cK_t(z) e^{-2 i \pi \xi \cdot z} \odif{z}$ for $\xi \in \Z^d$, with the inverse $\cK_t(z) = \sum_{\xi} \widehat{\cK}_t(\xi) e^{2 i \pi \xi \cdot z}$.
    and $\theta$ denotes the collection of all learnable parameters.
\end{definition}

We now specify the kernel parameterizations for FNO and SS-NO architectures.

\begin{definition}[FNO Kernel]\label{def:fno}
    For $1 \le K \le N/2$ and learnable matrices $P_t^{(k)} \in \C^{d_{t + 1} \times d_t}$, we let:
    \begin{equation}\label{def:kernel-fnos}
        \cK_t^{\text{FNO}}(z) := \sum_{k \in \ens{-K, \ldots, K}^d} P_t^{(k)} e^{2i \pi k \cdot z}.
    \end{equation}
\end{definition}

\begin{definition}[SS-NO Kernels]\label{def:ss-no-kernels}
    SS-NO kernels rely on 1D directional components. We distinguish the \textbf{Product-Form} (left) and the \textbf{Sum-Form} (right) in \eqref{eq:kernel-ssnos}:
    \begin{align}
        \cK_t^{\text{SS-NO}}(z) := \prod_{j = 1}^d \left( \cK^{(j)}_{t, +}(z_j) + \cK^{(j)}_{t, -}(z_j) \right) \quad
        \cK_t^{\text{SS-NO-FF}}(z) := \sum_{j = 1}^d \left( \cK^{(j)}_{t, +}(z_j) + \cK^{(j)}_{t, -}(z_j) \right), \label{eq:kernel-ssnos}
    \end{align}
    % The components are defined for each $j \in \{1, \dots, d\}$ by:
    % \begin{equation}
    %     \begin{split}
    %         \cK^{(j)}_{t, \pm}(z_j) := \mathbbm{1}_{\ens{\pm z_j \geq 0}} \sum_{k = 1}^K c_{t, k, j} & e^{-\rho_{t, k, j} \abs{z_j} + i \omega_{t, k, j} z_j} \\
    %         & \times C_{t, \pm}^{(k)} \left( B_{t, \pm}^{(k)} \right)^{\top},
    %     \end{split}
    % \end{equation}
    % where $c_{t, k, j} \in \R$, $\rho_{t, j, k} \in \R_+$, $\omega_{t, j, k} \in \R$, and matrices $C_{t, \pm}^{(k)}, B_{t, \pm}^{(k)}$ are learnable parameters.
    where each directional component $\cK^{(i)}_{t, \pm}(z_i)$ is defined as:
    \begin{align*}        &\hspace{-0.5cm}\cK^{(i)}_{t, \pm}(z_i) := \mathbbm{1}_{\ens{\pm z_i \geq 0}}   \sum_{k = 1}^K c_{t, k, i} e^{-\rho_{t, k, i} \abs{z_i}} e^{i \omega_{t, k, i} z_i} C_{t, \pm}^{(k)} \left( B_{t, \pm}^{(k)} \right)^{\top}.
    \end{align*}
    The variables $c_{t, k, i} \in \R$, $\rho_{t, i, k} \in \R_+, \omega_{t, i, k} \in \R$, and $C_{t, \pm}^{(k)} \in \R^{d_{t + 1}}, B_{t, \pm}^{(k)} \in \R^{d_t}$ are learnable parameters. $K$ denotes the fixed mode truncation parameter.
\end{definition}

\iffalse
\begin{remark}\label{rem:fno-vs-ssno}
In 1D, the class of FNO kernels is contained in the family of SS-NO kernels. Indeed, by setting $\rho_{t,k,i} = 0$, $\omega_{t,k,i} = 2\pi k_i$, with $k_i \in \ens{-K, \ldots, K}$, and choosing appropriate finite sums of rank-one factors $C_{t,\pm}^{(k)} (B_{t,\pm}^{(k)})^\top$, the SS-NO kernel reproduces the Fourier atoms $e^{2\pi i k z}P_t^{(k)}$, and further allow nonzero decay rates (with the damping parameters $\{\rho_{t,k,i}\}_{k = 1}^K$), and learnable frequencies. SS-NOs therefore extend the representational capacity of FNOs.% (e.g., it can model spatially localized and anisotropic kernels that FNO cannot).
\end{remark}
\fi

\begin{remark}\label{rem:fno-vs-ssno}
In 1D, FNO kernels are contained in the SS-NO family: setting $\rho_{t,k,i}=0$, $\omega_{t,k,i}=2\pi k_i$, and choosing rank-one sums $C_{t,\pm}^{(k)}(B_{t,\pm}^{(k)})^\top$ recovers the Fourier atoms $e^{2\pi i kz}P_t^{(k)}$. Nonzero damping and learnable frequencies therefore strictly extend the FNO parameterization.
\end{remark}

Throughout this work, \textbf{we establish and prove our results for the \emph{sum-form} SS-NOs architecture}. Results for the \emph{product-form} are given in Section \ref{sec:prodform} and the proofs are omitted as they follow from straightforward adaptation of sum-form proofs. For conciseness, we use $\cK_t^{\text{SS-NO}}$ to denote $\cK_t^{\text{SS-NO-FF}}$.

\iffalse
\subsection{Diagonal State Space Models (S4D)} \label{sec:S4D}
\textbf{Spatial state space system}: \\
 $\partial_x \tilde{v}(x) = A\, \tilde{v}(x) + B\, v(x)$, with $x \in [a, b]$ \\
 $(Kv)(x) = C\, \tilde v(x)$ with $\quad \tilde{v}(a) = 0$. \\
 Integral solution: \\
 $\tilde{v}(x) = \int_{a}^{x} e^{A (x - y)} B\, v(y)\, dy$ \\
 $(Kv)(x) = \int_{a}^{x} C\, e^{A (x - y)} B\, v(y)\, dy$. \\
\textbf{Kernel form}: \\
$(Kv)(x) = \int_{a}^{x} \cK(x-y)\, v(y)\, dy$, \\
$\cK(x-y) := C\, e^{A (x-y)} B$.\\
\textbf{S4D parametrization from \citep{gu2022parameterization}, used in \cite{anonymous2025merging} :} \\
The vector-valued function $v(x) \in \R^H$ is treated as $H$ independent channels. Denoting $ v(x)=[v_1(x), ..., v_h(x), ..., v_H(x)]$, \\
$\partial_x \tilde{v}_h(x) = A_h\, \tilde{v}_h(x) + B_h\, v(x_h)$, $A_h \in \C^{K \times K}, B_h \in \C^{K \times 1}$ \\
 $(Kv_h)(x) = C_h\, \tilde v_h(x)$, $C_h \in \C_h^{1 \times K}$ \\
 Moreover, $A_h$ is chosen to be diagonal: $A_h = \Lambda(r_1 + i \omega_1, ..., r_K + i \omega_K)$. \\
 As a result, the number of (complex) parameters, for the kernel of a single layer is about $H \times K$, while the number of (real) parameters for the matrix $W$ is $H^2$. \\ 
 This aligns with the claim in \cite[section 4]{anonymous2025merging} claim: \textit{spatial-SSM \#Parameter = $L(H^2+HK)d$} with $L$ the number of layers. 
 
\fi

In the following, we analyze the stability of SS-NO layers as nonlinear operators on $L^2(\mathbb{T}^d)$. Our approach follows a  two-step procedure: we first establish stability results in the continuous domain to derive explicit Lipschitz constants as a function of model parameters (a similar analysis in the case of FNOs is done in \cite{lanthaler2025discretizationerrorfourierneural}. This continuous-time analysis will later be used as a baseline to evaluate the approximation error induced by discretization, using ISS analysis.

\subsection{Some Stability Results}

We now establish basic stability properties of SS-NO layers viewed as nonlinear operators acting on functions over the torus $\T^d$. Our goal is twofold: $(i)$ to show that a single SS-NO layer is Lipschitz continuous in $L^2(\T^d)$ under a mild integrability condition on its kernel; and $(ii)$ to propagate this stability through compositions of multiple layers. We then verify that SS-NO kernels indeed satisfy the required integrability assumption, yielding explicit stability constants that depend on the model parameters.

We begin by analyzing a single SS-NO layer. The following lemma shows that, provided the operator norm of the kernel is integrable over $\T^d$, a single SS-NO layer defines a Lipschitz map on $L^2(\T^d, \R^H)$. For the following results, we defer the proof to~\Cref{apx:stability-proofs} for readability.

\begin{lemma}[Stability of a Single SS-NO layer]\label{lem:stability-layer-ssm}
    Assume that $g \colon z \mapsto \normop{\cK^{\textnormal{SS-NO}} (z)} \in L^1(\T^d, \R)$, then for any input functions $v, w \in L^2(\T^d, \R^H)$, we have
        \begin{align}
             \norm{\mathcal{L}v - \mathcal{L}w}_{L^2} \leq C_{\sigma} \norm{v - w}_{L^2}, 
        \end{align}
    where $\displaystyle C_{\sigma} := L_\sigma \left( \normop{W} + \int_{\T^d} \normop{\cK^{\textnormal{SS-NO}}(z)} \odif{z} \right)$, with $\normop{\cdot}$ the operator norm of a matrix, and $\mathcal{L}$ is the map defined in~\eqref{def:ssno-layer} (we omit the layer index $t$) $\mathcal{L} \colon v \mapsto \sigma\left( W v + \cK^{\textnormal{SS-NO}} * v + b \right)$.
\end{lemma}

Stability of individual layers naturally extends to deep architectures via composition. The next corollary quantifies how Lipschitz constants accumulate across a stack of SS-NO layers, yielding an explicit bound for the full operator.

\begin{corollary}[Stability of Layer Stacks]\label{cor:stability-full-ssm}
    Let $\mathscr{L} := \mathcal{L}_{T-1} \circ \dots \circ \mathcal{L}_0$ be a stack of $T > 0$ layers of the form \eqref{eq:psi-def}, with $\cK_t^{\textnormal{SS-NO}}$ satisfying $z \mapsto \normop{\cK_t^{\textnormal{SS-NO}}(z)} \in L^1(\T^d)$, then we have for any input functions $v, w \in L^2(\T^d, \R^H)$: $\norm{\mathscr{L}v - \mathscr{L}w}_{L^2} \leq C_{\sigma, T} \norm{v - w}_{L^2}$, where 
        \[ C_{\sigma, T} := L_\sigma^T \left[ \prod_{t = 0}^{T-1} \left( \normop{W_t} + \int_{\T^d} \normop{\cK_t^{\textnormal{SS-NO}}(z)} \odif{z} \right) \right]. \]
\end{corollary}

Both previous results require that SS-NO kernels are integrable on $L^1(\T^d)$, a property that is ensured by the following lemma.
%The remaining question is whether SS-NO kernels satisfy the integrability condition required in the previous results. The following lemma answers this in the affirmative: the SS-NO kernels $\cK^{\textnormal{SS-NO}}$ satisfy the desired $L^1(\T^d)$ integrability assumption required in both~\Cref{lem:stability-layer-ssm} and~\Cref{cor:stability-full-ssm}. 

\begin{lemma}\label{lem:finite-bound-L1-op-norm}
    For SS-NO kernels defined in \Cref{def:ss-no-kernels}, 
    the assumption of \Cref{lem:stability-layer-ssm} on $\cK^{\textnormal{SS-NO}}$ holds, and we have
        \begin{align*}
            &\int_{\T^d} \normop{\cK^{\textnormal{SS-NO}}(z)} \odif{z} \leq \sum_{i = 1}^d \sum_{k=1}^K \left[ \dfrac{1 - e^{-\rho_{k, i}}}{\abs{\rho_{k, i}}} \norm{C_+^{(k)}}_2 \norm{B^{(k)}_+}_2 \right],
        \end{align*}
    where we omit the layer index $t$ for simplicity.
    \iffalse
    more precisely, if $D = [-a, b] \subset \R$ with $a, b \geq 0$ then
        \begin{align}
            \int_{\T^d} \norm{\cK(z)}_{\mathrm{op}} \odif{z} &\leq \sum_{k=1}^K \left[ \dfrac{2 - e^{r_k b} - e^{r_k a}}{\abs{r_k}} \norm{C_k}_2 \norm{B_k}_2 \right] \\
            &\le \begin{cases} 
            \displaystyle\diam(D) \sum_{k = 1}^K \norm{C_k}_2 \norm{B_k}_2  < \infty, \\[7pt]
            \displaystyle 2 \sum_{k = 1}^K \dfrac{\norm{C_k}_2 \norm{B_k}_2}{\abs{r_k}}
        \end{cases}  \numberthis\label{30ba552b-cb1a-44c5-b9d2-732306caae3d} 
        \end{align} 
    where $\diam(D) := b + a$ is the diameter of the space domain.
    \fi
\end{lemma}

Taken together, \Cref{lem:stability-layer-ssm,cor:stability-full-ssm,lem:finite-bound-L1-op-norm} show that the continuous SS-NO architectures define stable nonlinear operators on $L^2(\T^d)$, with Lipschitz constants that are explicitly parametrized by the kernel parameters. In next Section, we shift our analysis to the discrete setting which is the main contribution of this work. In~\Cref{thm:global_stability_layers}, we extend these stability guarantees to finite-resolution operators, providing a rigorous framework to \enquote{quantify} the discretization error.

\iffalse
Denote $f(r_k) := \dfrac{2 - e^{r_k b} - e^{r_k a}}{\abs{r_k}} = \dfrac{e^{r_k b} + e^{r_k a} - 2}{r_k}$ the quantity that appears in the upper-bound above. Notice that
    \begin{align}
        f'(r_k) = \dfrac{(b e^{r_k b} + a e^{r_k a}) r_k - (e^{r_k b} + e^{r_k a} - 2)}{r_k^2} = \dfrac{N(r_k)}{r_k^2},
    \end{align}
    where
    \begin{align}
        N(r) := (b e^{rb} + a e^{r a})r - (e^{r a} + e^{r b} - 2).
    \end{align}
    Now differentiate $N(r)$:
        \begin{align}
            N'(r) &= r (b^2 e^{r b} + a^2 e^{r a}) + (b e^{r b} + a e^{r a}) - (b e^{rb} + a e^{ra}) \\
            &= r (b^2 e^{r b} + a^2 e^{ra})  < 0 
        \end{align}
        
    This yields 
        \[ \forall r < 0, \quad N(r) > N(0) = 0.\]
    And this implies that the stability Lipschitz upper-bound $f(r_k)$ is increasing in $r_k$, \textbf{which suggests that in order to have better stability, it may be useful to have greater damping (equivalently, very negative $r_k$)}. This aligns with the findings of~\citet{anonymous2025merging} when they say:
    \begin{itemize}
        \item \enquote{\emph{lower-capacity models rely critically on explicit damping for stability.}}
        \item \enquote{\emph{when frequency adaptation is disabled, models intensify their damping mechanisms to maintain stability.}}
        \item \enquote{\emph{This progressive dependency relationship underscores damping's role as a crucial stabilization mechanism in capacity-constrained settings.}}
    \end{itemize}
\fi

\section{Discrete Framework}

The architecture of a neural operator can be defined independently of discretization, as it is based on a convolution operation $\cK*v$.
However, in practice, to make such architectures computationally feasible (e.g., for FNOs or SS-NOs models, and more generally, any \emph{discretization invariant} neural operator based on integral convolution), one needs to discretize the convolution operation. As done in~\citet{lanthaler2025discretizationerrorfourierneural} and briefly discussed in~\citet{anonymous2025merging}, the approximation is done over the uniform grid $\T_N^d$\footnote{For simplicity, we assume the grid to be a square. Our results can be easily extended to a general, non-squared, uniform grid.}, and for every $x \in \T^d_N$ we define the \emph{discrete convolution} (a Monte Carlo approximation of the true integral)
    \begin{align*}
        (\cK^N_t * v_t^N)(x) :&\!= \frac{1}{N^d} \sum_{y \in \T^d_N} \cK_t(x - y) v_t^N(y) \approx \int_{\T^d} \cK_t(x - y) v_t(y) \odif{y}, \numberthis\label{9754cccd-2dc7-4de6-a4a3-bdb8b9bbf6c9-2}
    \end{align*}
where $v_t^N$ is the function passed to the $t^{\textnormal{th}}$ layer when using only discrete (i.e., approximate) convolutions, resulting in a different function than the ground truth $v_t$ if integral (i.e., exact) convolutions were used. % Equality~\eqref{9754cccd-2dc7-4de6-a4a3-bdb8b9bbf6c9-2} shows that the true integral convolution is approximated by a Monte Carlo sum on the grid $\T^d_N$.
%Recall the grid $\T^d_N$ is identified to the additive group $(\frac{1}{N} (\nicefrac{\Z}{N \Z})^d, +)$, which corresponds to the discretization of the $d$--dimensional torus $\T^d$. 
It is worth mentioning that we \emph{do not} discretize the function space. In both scenarios, whether involving the integral or discrete convolution, the architecture outlined in ~\Cref{def:neural-operator} operates directly on functions defined over a continuous domain.

\subsection{Discretization Error}

To analyze the discretization error induced by~\eqref{9754cccd-2dc7-4de6-a4a3-bdb8b9bbf6c9-2}, it is convenient to work in the Fourier domain.  
We therefore begin by characterizing the Fourier coefficients of the convolution kernels used by FNO and SS-NO architectures.

Note that FNO kernels are defined directly in the frequency domain via a hard truncation of Fourier modes. This yields the following explicit form of their Fourier transform.
\begin{lemma}[Fourier Transform of FNO Kernel]\label{lem:fourier-fno}
    For any $\xi \in \Z^d$, we have:
    \begin{equation}
        \widehat{\cK}_t^{\textnormal{FNO}}(\xi) = \sum_{k \in \ens{-K, \ldots, K}^d} P_t^{(k)} \delta_k(\xi) = P_t^{(\xi)} \mathbbm{1}_{\ens{\xi = k}},
    \end{equation}
    where we let $P_t^{(\xi)} = 0$ for any $\xi \in \Z^d \setminus \ens{-K, \ldots, K}^d$.
\end{lemma}

In contrast, SS-NO kernels are defined in the spatial domain via sums of damped oscillatory components. Their Fourier transform therefore exhibits a fundamentally different structure, with non-compact spectral support.

\begin{lemma}[Fourier Transform of SS-NO Kernel]\label{lem:fourier-ss-no}
    For any $\xi = (\xi_1, \dots, \xi_d) \in \Z^d$, we have:
    \begin{align*}
        &\widehat{\cK}_t^{\textnormal{SS-NO}}(\xi) = \sum_{i=1}^d \sum_{k=1}^K c_{t, k, i} \left[ F_{+, k, i}(\xi_i) A_{k, +} + F_{-, k, i}(\xi_i) A_{k, -} \right],
    \end{align*}
    where $A_{k, \eps} := C_{t, \eps}^{(k)} \left( B_{t, \eps}^{(k)} \right)^\top$ for $\eps \in \ens{+, -}$, and
%\johan{\\typo ? $C_{t, \eps}^{(k)} \left( B_{t, \eps}^{(k)} \right)$}    
    \begin{align*}
         F_{+, k, i}(\xi_i) &:= \frac{1 - e^{-(\rho_{t, k, i} - i(\omega_{t, k, i} - 2 \pi \xi_i) ) / 2}}{\rho_{t, k, i} - i(\omega_{t, k, i} - 2 \pi \xi_i)},  \\  \quad
         F_{-, k, i}(\xi_i) &:= \frac{1 - e^{-(\rho_{t, k, i} + i(\omega_{t, k, i} - 2 \pi \xi_i) ) / 2}}{\rho_{t, k, i} + i(\omega_{t, k, i} - 2 \pi \xi_i)}.
         %&\int_{-\frac{1}{2}}^0 e^{- \rho_{t, k} \abs{z}} e^{i z (\omega_{t, k} - 2 n \pi)} \odif{z} = \int_{-\frac{1}{2}}^0 e^{z (\rho_{t, k} + i (\omega_{t, k} - 2 n \pi))} \odif{z} = \frac{e^{z (\rho_{t, k} + i (\omega_{t, k} - 2 n \pi))}}{\rho_{t, k} + i (\omega_{t, k} - 2 n \pi)} \Bigg]_{z = -\frac{1}{2}}^{z = 0},
    \end{align*}
\end{lemma}

The proofs of lemmas \ref{lem:fourier-fno} and \ref{lem:fourier-ss-no} are deferred to \Cref{apx:bounds-fno-ss-no}.

%\begin{remark}
%   FNO kernels, while having \emph{fixed $d$--dimensional} frequencies $\ens{1, 2, 3, \ldots, K}^d$, are $\mathscr{C}^{\infty}$ $1$--periodic functions thus, by~\Cref{lem:todo}, the Fourier coefficients $\widehat{\cK}_t^{\text{FNO}}(\xi)$, $\xi \in \Z^d$ will decay faster than any polynomial as $\norm{\xi} \to +\infty$. This \emph{is not the case} for the SS-NO kernel; while the frequencies (which are only \emph{one dimensional}) are learnable, this result in lack of regularity (\emph{a priori} no periodicity at the boundaries of $\T^d$) and the Fourier coefficients of $\widehat{\cK}_t^{\text{SS-NO}}(\xi) \asymp \norm{\xi}^{-1}$ as $\norm{\xi} \to +\infty$; such kernel typically does not have finite Sobolev norm.
%\end{remark}

%\begin{remark}\label{rem:regularity-fno-ssno}
FNO kernels are finite linear combinations of $d$--dimensional Fourier modes with frequencies in $\{-K,\ldots,K\}^d$. As such, they define $\mathscr{C}^{\infty}$, $1$--periodic functions on $\mathbb{T}^d$. According to~\Cref{lem:fourier-fno}, their Fourier coefficients $\widehat{\mathcal K}_t^{\mathrm{FNO}}(\xi)$ decay faster than any polynomial as $\|\xi\|\to\infty$. In particular, FNO kernels possess finite Sobolev norms of arbitrary order.

This behavior contrasts sharply with that of SS-NO kernels. While SS-NO employs learnable (and only one--dimensional) frequencies, the resulting kernels are generally not periodic on $\mathbb{T}^d$ and lack smoothness across the boundary. Consequently, their Fourier coefficients exhibit in the worse case only algebraic decay
\[ \widehat{\mathcal K}_t^{\mathrm{SS\text{-}NO}}(\xi) \asymp \|\xi\|^{-1}
\,\, \text{ as } \,\, \|\xi\|\to\infty, \]
and the associated kernels only have finite Sobolev norm up to a constant $s < + \infty$. This distinction is illustrated in \Cref{fig:operator-norm-ssno-vs-fno}. As shown in \Cref{fig:ssno-kernel}, the SS-NO kernel exhibits a continuous, full-band spectral profile consistent with a power-law decay, whereas the FNO kernel in \Cref{fig:fno-kernel} displays compact spectral support induced by the hard frequency cutoff at $K$. These qualitative differences reflect the fundamentally different regularity of the two architectures.
%\end{remark}

%\adri{please someone correct Parseval's identity or provide correct norms so as to get bounds on $\normgridmatone{\cdot}$, thank you.}
%the discrete Parseval's formula is $\sum_{x \in \T^d_N} \abs{f(x)}^2 = \frac{1}{N^d} \sum_{\xi \in \ens{0, 1, \ldots, N - 1}^d} \abs{\widehat{f}(\xi)}^2$

%We first recall the Sobolev's embedding theorem
%\begin{theorem}
    
%\end{theorem}

%A consequence of this is the following lemma
%\begin{lemma}[{Sobolev's Inequality on $\T^d$~\citep{BenyiOh2013}}]
%    For any integer $s > \frac{d}{2}$ and any function (\adri{check regularity}) $u \colon \T^d \to \R$, we have
%        \[ \norm{u}_{L^{\infty}(\T^d)} \lesssim \norm{u}_{H^s(\T^d)}. \]
%\end{lemma}

%\adri{Above lemma can be used to give similar bounds as in the paper~\citet{lanthaler2025discretizationerrorfourierneural}.}

%\adri{Draft}

\begin{wrapfigure}{r}{0.3\textwidth}
\vspace{-0.2cm}
    \centering
    \begin{subfigure}[b]{0.3\columnwidth}
        \includegraphics[width=\linewidth,trim={25.4cm 0 0 0},clip]{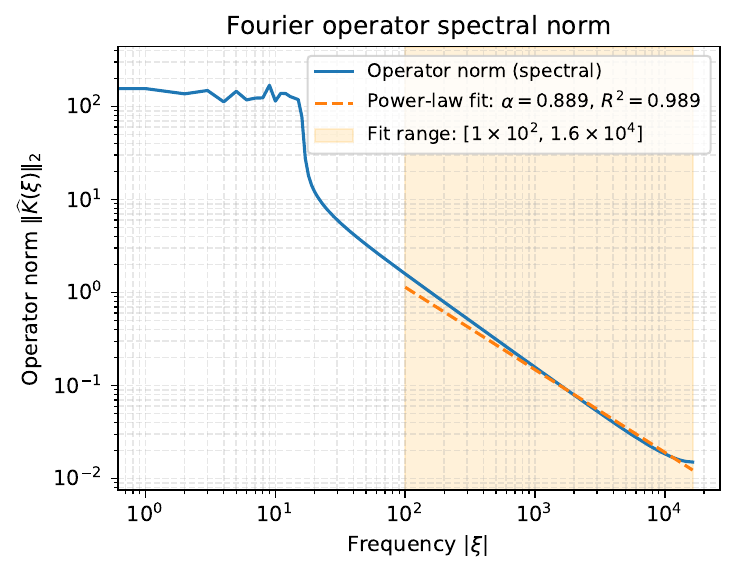}
        \caption{SSNO (forward) kernel}
        \label{fig:ssno-kernel}
    \end{subfigure}
    \hfill
    \begin{subfigure}[b]{0.3\columnwidth}
        \includegraphics[width=1.035\linewidth,trim={25cm 0 0 0},clip]{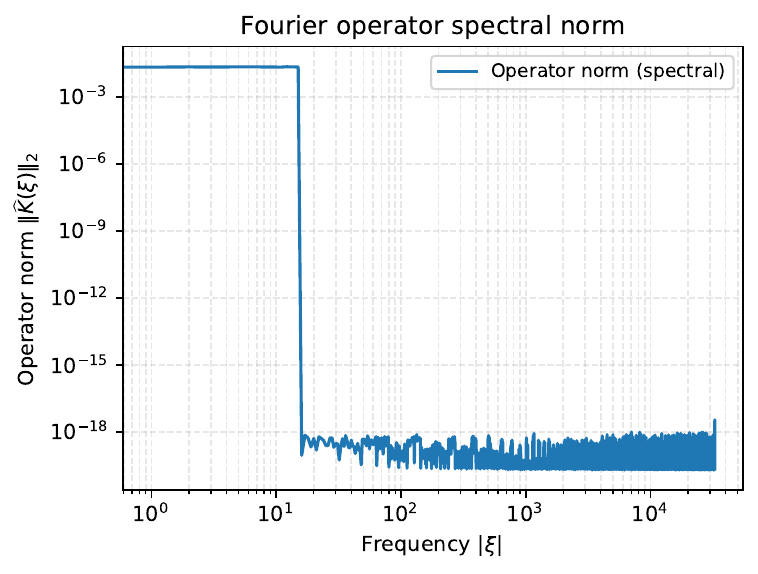}
        \caption{FNO kernel}
        \label{fig:fno-kernel}
    \end{subfigure}
    \caption{
        Comparison of a SS-NO and a FNO kernel (both \emph{randomly initialized}) evaluated at spatial resolution $L = 2^{15}$ on $\intff{0}{1}$. SS-NO exhibits continuous, full-band spectral structure without mode truncation, while FNO enforces a hard cutoff frequency (at $K = 16$), resulting in compact spectral support.
    }
    \label{fig:operator-norm-ssno-vs-fno}
    \vspace{-2cm}
\end{wrapfigure}

As an intermediate step toward establishing discretization error bounds, we derive explicit estimates on the discrete $L^2$-norm of the convolution kernels associated with each architecture. We begin with the FNO case.
\begin{lemma}[Bound on the $L^2$-norm of $\cK_t$ for FNO]\label{lem:bound-fno}
     For the FNO kernel in~\eqref{def:kernel-fnos}, there exists a constant $C_d \ge 0$ such that
        \[ \normgridmattwo{\cK_t^{\textnormal{FNO}}} \le C_d (N K)^{\frac{d}{2}} \sup_{k \in \ens{1, 2, \ldots, K}^d} \normop{P_t^{(k)}}. \]
\end{lemma}

We now turn to SS-NO kernels. In contrast to FNOs, which rely on an explicit truncation over a $d$--dimensional frequency grid and therefore involve an exponential number of modes in $d$, SS-NOs employ a parameter-efficient representation based on sums of one--dimensional spectral components. This structural difference leads to a distinct scaling behavior for the discrete kernel norm.

\begin{lemma}[Bound on the $L^2$-norm of $\cK_t$ for SS-NO]\label{lem:bound-ssno}
    For the SS-NO kernel $\cK_t^{\textnormal{SS-NO}}$ in~\eqref{eq:kernel-ssnos} we have
     %   \[ \normgridmattwo{\cK_t^{\textnormal{SS-NO}}} \leq K d N^{\frac{d}{2}} \sup_{i \in [d], k \in [K]} \left\{ \abs{c_{t, k, i}} \alpha_{t, k, i}  \left(  \norm{C_{t, +}^{(k)}}^2  \norm{B_{t, +}^{(k)}}^2 + \norm{C_{t, -}^{(k)}}^2  \norm{B_{t, -}^{(k)}}^2 \right)^{\nicefrac{1}{2}} \right\} .  \]
    %Alternatively,
        \begin{align*}
            \normgridmattwo{\cK_t^{\textnormal{SS-NO}}} \le C_d Kd N^{\frac{d}{2}},
        \end{align*}
    where $C_d := \sup_{i \in [d], \, k \in [K]} \ens{\abs{c_{t, k, i}} \left(  A_{k, +} + A_{k, -} \right)}$ and $A_{k, \pm} := \normop{C_{t, \pm}^{(k)}} \, \normop{B_{t, \pm}^{(k)}}$.
        
\end{lemma}

The proofs of~\Cref{lem:bound-fno,lem:bound-ssno} are in~\Cref{apx:bounds-fno-ss-no}.

We now state our main result on the discretization error of SS-NOs. We measure the error at the network output, after $T$ layers, defined by $\smash{\cE_T^{(0)}(x) := v_T^N(x) - v_T(x), \ x \in \T_N^d}$. For simplicity, we omit the networks $\mathcal{P}$ and $\mathcal{Q}$ in our bound to focus on the layers $\{\mathcal{L}_t\}_{t = 0}^{T - 1}$. $\mathcal{P}$ and $\mathcal{Q}$ only affect our bound up to some multiplicative constant independent of $N$, $s$ and the kernels.% (which are the quantities of interest).

\begin{theorem}[Discretization Error of SS-NO]\label{thm:discretization-error-ssnos}
    The discretization error of SS-NOs verifies:
        \[ \normgridvec{\cE_T^{(0)}} \le N^\beta B \dfrac{A^T - 1}{A - 1}, \numberthis\label{thm:discretization-error-ssnos-result} \]
    with $A := L_{\sigma} \sup_{t \in \ens{0, \ldots, T - 1}} \left( \normop{W_t} + K d \, C(\cK_t) \right)$, and
        \begin{align*}
            C(\cK_t) :&\!= \sup_{i \in [d], \, k \in [K]} \left\{\abs{c_{t, k, i}} \left( \normop{C_{t, +}^{(k)}} \, \normop{B_{t, +}^{(k)}}  +  \normop{C_{t, -}^{(k)}} \, \normop{B_{t, -}^{(k)}} \right)\right\},
        \end{align*}
    and, if there exists an integer $K_{\textnormal{cutoff}} > 0$ such that $\widehat{\cK}_t(\xi) = 0$ for all $\xi \in \Z^d$ with $\norm{\xi}_{\infty} \ge K_{\textnormal{cutoff}}$. Then,~\eqref{thm:discretization-error-ssnos-result} holds with
            \begin{align*}
                \beta & := \frac{d}{2} -s < 0,  \quad B := L_\sigma C_{d, s} K_{\textnormal{cutoff}}^{\frac{d}{2} + s} \sup_{t \in \ens{0, 1, \ldots T - 1}} \normop{\widehat{\cK}_t}_{\infty} \normsobolev{v_t},
            \end{align*}
        where $0 \le C_{d, s} < \infty$ is a constant depending on $d$ and $s$.
\end{theorem}

\begin{proof}[Proof (Sketch)]
For each layer $t$, subtract the discrete and continuous updates at a grid point $x \in \T_N^d$ and use that each $\sigma_t$ is $L_\sigma$-Lipschitz to obtain:
    \begin{align*}
        &\norm{\cE_{t+1}^{(0)}(x)}_{\ell^2(\T^d_N)} \le L_\sigma\Big(\normop{W_t} \norm{\cE_t^{(0)}(x)}_{\ell^2(\T^d_N)} + \norm{(\cK_t *_N v_t^N)(x) - (\cK_t * v_t)(x)} \Big).
    \end{align*} 
Then, split the convolution discrepancy into (i) the propagation of the current error through the discrete convolution and (ii) the pure aliasing error of replacing $\int_{\T^d}$ by the grid average; take the $\ell^2(\T_N^d)$--norm and apply (a) the discrete Young's inequality (\Cref{appdx-lem:young-convolution-inequality}) to bound the propagated term by $Kd\,C(\cK_t)\,\|\cE_t^{(0)}\|_{\ell^2(\T^d_N)}$, and (b) the Fourier cutoff assumption combined with $v_t\in H^s$ to bound the aliasing term by $C_{d,s}K_{\mathrm{cutoff}}^{\frac d2+s}\|\widehat{\cK}_t\|_\infty \|v_t\|_{H^s}\,N^{\frac d2-s}$. This yields the one-step recursion
    \[ \|\cE_{t+1}^{(0)}\|_{\ell^2(\T^d_N)} \le A \|\cE_t^{(0)}\|_{\ell^2(\T^d_N)} + N^\beta B, \,\, \beta=\frac d2-s<0, \]
with $A$ and $B$ as in the theorem. Iterating this inequality from $t=0$ to $T-1$ (and using $\cE_0^{(0)}=0$ when the inputs coincide) gives
$\|\cE_T^{(0)}\|_{\ell^2(\T^d_N)} \le  N^\beta B\frac{A^T-1}{A-1}$, which is~\eqref{thm:discretization-error-ssnos-result}.
\end{proof}

Note that, if there exists some real number $\alpha > d$ and a finite real constant $C_{d, \alpha} \ge 0$ such that
    \[ \normop{\widehat{\cK}_t(\xi)} \le C_{d, \alpha} (1 + \norm{\xi})^{-\alpha}, \numberthis\label{} \]
for all $\xi \in \Z^d$. Then, inequality~\eqref{thm:discretization-error-ssnos-result} holds with
        \begin{align*}
            \beta & := \max \ens{\frac{d}{2} -s, d - \alpha} < 0, \quad  B := L_\sigma C_{d, s, \alpha} \sup_{t \in \ens{0, 1, \ldots T - 1}}  \normsobolev{v_t},
        \end{align*}
where $0 \le C_{d, s, \alpha} < \infty$ is a universal constant depending only on $d$, $s$ and $\alpha$,

The only risk in~\Cref{thm:discretization-error-ssnos} is that some of the $v_t$, $t \in \ens{0, \dots, T-1}$ may have an infinite $H^s$ norm. We show that this \emph{cannot} happen, as long as the activation function $\sigma$ is smooth, i.e., $\sigma \in \mathscr{C}^\infty(\R, \R)$.

%We start with the next lemma from~\citet{lanthaler2025discretizationerrorfourierneural} to bound the $H^s$-norm of a point-wise composition $\sigma \circ v$ using the $H^s$-norm of $v$, albeit some regularity assumptions on $\sigma$.

\begin{corollary}[{\citet[Theorem~2.87]{Bahouri2011chap2}}]\label{lem:Moser}
    Let $\sigma \in \mathscr{C}^{\infty}(\R, \R)$ be a smooth activation function vanishing at $0$, let $s > \frac{d}{2}$ be a real number and $v \in H^s(\T^d, \R^H)$ then $\sigma \circ v \in H^s(\T^d, \R^H)$.
\end{corollary}

This result is more general than~\citet[Lemma~D.1]{lanthaler2025discretizationerrorfourierneural}, which applies only for integers $s > \frac{d}{2}$. 
\iffalse
\begin{lemma}\label{lem:Moser}
Assume $\varphi \colon \T^d \to \T^d$ possesses continuous derivatives up to order $r$ which are bounded by $B$ and let $v \in H^r(\T^d)$. Then $\varphi \circ v \in H^r(\T^d)$ and furthermore
\begin{equation*}
    \|\varphi \circ v\|_{H^r} \leq Bc \left(1+ \|v\|_{\infty}^{r-1}\right)\|v\|_{H^r}
\end{equation*}
where $c$ is a constant depending on $r$ and $d$. 
\end{lemma}
\fi
We now show that across all layers $t \in \ens{0, 1, \ldots, T - 1}$, the output $v_t$ stays in $H^s(\T^d)$:
\begin{corollary}\label{lem:heredity-H-s}
    If $\sigma$ is a smooth function ($\mathscr{C}^\infty$) and $v_t \in H^s(\T^d)$ with $s > \frac{d}{2}$, then $v_{t+1} \in H^s(\T^d)$.
\end{corollary}
Therefore, as long as $\sigma$ is smooth, $\sigma(0) = 0$, and the input function $v_0$ is in $H^s(\T^d)$, then $v_t \in H^s(\T^d)$ for every $t \in \ens{0, \dots, T-1}$, and the final bound on the discretization error is finite.

We now extend this result to less regular $\sigma$, e.g., \textsc{ReLU}. This result significantly strengthen the regularity framework underlying discretization error bounds for neural operators.
\begin{lemma}[A Regularity Lemma]\label{lem:regularity-lemma}
    Let $d, H$ be positive integers, $s > \frac{d}{2}$ be a real number. Assume $v \in H^s(\T^d, \R^H)$, and the activation function $\sigma$ is globally Lipschitz, $\sigma(0) = 0$ and $\sigma'$ has bounded variations on $\R$. Then, for any $0 < t < \min \ens{ \frac{3}{2}, s }$, we have $\sigma \circ v \in H^t(\T^d)$.
\end{lemma}

% > > > > > > > > > > > > > > > > >       < < < < < < < < < < < < < < < <      

The proofs of~\Cref{lem:heredity-H-s,lem:regularity-lemma} are in~\Cref{appdx-sec:regularity-lemma-proof}. Also,~\Cref{lem:regularity-lemma} holds for $\sigma = \textsc{ReLU}$ as it is $1$-Lipschitz, $\sigma(0) = 0$, and $\sigma' = \mathbf{1}_{(0, \infty)}$ has bounded variations.

While~\Cref{lem:regularity-lemma} provides an \emph{a priori} bound on the Sobolev regularity of the composition $\sigma \circ v$, our empirical results in~\Cref{sec:exp:disc-error-ssno} strongly indicate that this bound is in fact \emph{sharp} for \textsc{ReLU}, and more generally for piecewise linear activation functions vanishing at $0$. In particular, the observed convergence rates are consistent with $\sigma \circ v$ belonging to $H^t(\T^d)$ for all $0 < t < \min\ens{\frac{3}{2}, s}$, but failing to exhibit higher Sobolev regularity beyond this threshold.

As a result, our core~\Cref{thm:discretization-error-ssnos} holds under markedly weaker assumptions, for a broad and practically relevant class of \emph{non-smooth} activations $\sigma$ than those considered in previous analyses.%, while at the same time sharply characterizing the intrinsic regularity limitations induced by such non-smooth nonlinearities.

%Notice that \Cref{lem:regularity-lemma} guarantees Sobolev regularity for  $\sigma \circ v$ up to $\min \ens{ \frac{3}{2}, s }$, and in particular, \emph{a priori}, it does not show that $\sigma \circ v \notin H^r$ for $r \ge \min \ens{ \frac{3}{2}, s }$. However interestingly, our experiments in \Cref{sec:exp:disc-error-ssno} show \emph{empirically} that this result should be tight, at least for \textsc{ReLU}, i.e., that the best Sobolev regularity is exactly $\min \ens{ \frac{3}{2}, s }^-$.

% > > > > > > > > > > > > > > > > >       < < < < < < < < < < < < < < < <      

\Cref{thm:discretization-error-ssnos} and above lemmas show that the discretized operator remains a faithful approximation of its continuous counterpart as the grid resolution $N$ increases. However, a low approximation error does not guarantee the numerical stability of the implemented algorithm. To ensure the discrete model is robust against perturbations, we perform an input-to-state stability (ISS) analysis. 

% > > > > > > > > > > > > > > > > >       < < < < < < < < < < < < < < < <      

\subsection{Global ISS under Discretization Errors}
Based on the previous approximation error analysis, we now turn to a global stability assessment. Specifically, we examine the operator’s behavior across different latent states, through an Input-to-State Stability (ISS) analysis. To do so, we establish the stability properties of the discretized SS-NO layers, viewed as nonlinear operators acting on grid functions over the discrete torus $\mathbb{T}_N^d$. Our goal is to demonstrate that a multi-layer SS-NO is globally Lipschitz continuous in $L^2(\mathbb{T}_N^d, \mathbb{R}^H)$ provided that the discrete operator norm of its kernel remains bounded. The following theorem shows that, under this condition, the composition of $T$ nonlinear layers defines a stable Lipschitz map:% on the discrete latent space:

\begin{theorem}[Global Stability of Discretized SS-NO]
\label{thm:global_stability_layers}
Let $\mathcal{L}_N^{(T)} = \mathcal{L}_{N, T-1} \circ \dots \circ \mathcal{L}_{N, 0}$ be the operator representing the composition of $T$ discretized SS-NO layers. For any latent states $v, w \in \ell^2(\mathbb{T}_N^d, \mathbb{R}^{H})$, we have: $\|\mathcal{L}_N^{(T)}(v) - \mathcal{L}_N^{(T)}(w)\|_{\ell^2(\mathbb{T}_N^d)} \le \mathbf{C}_{N, T} \|v - w\|_{\ell^2(\mathbb{T}_N^d)}$, where %the global discrete Lipschitz constant $\mathbf{C}_{N, T}$ is:
\begin{equation}
    \mathbf{C}_{N, T} := \prod_{t=0}^{T-1} L_{\sigma_t} \left( \|W_t\| + \|\mathcal{K}_{N,t}\|_{\ell^2(\mathbb{T}_N^d)} \right).
\end{equation}
\end{theorem}

The proof starts by characterizing the Lipschitz property of a single layer. The overall stability then follows by induction. The complete proof can be found in~\Cref{apx:stability-proofs}.

Let $\Psi_N^{(T)}$ be the composition of $T$ discretized SS-NO layers, i.e., $\Psi_N^{(T)} = \mathcal{L}_{N, T-1} \circ \dots \circ \mathcal{L}_{N, 0}$ as defined in \ref{def:neural-operator}. The stability results established for the sequence of layers $\mathcal{L}_N^{(T)}$ in~\Cref{thm:global_stability_layers} can be naturally extended to the full architecture $\Psi_N^T$. Since the lifting operator $\mathcal{P}$ and the projection operator $\mathcal{Q}$ are typically Lipschitz continuous, often implemented as point-wise linear layers or shallow MLPs, with constants $L_{\mathcal{P}}$ and $L_{\mathcal{Q}}$, the global Lipschitz constant for the full operator $\Psi_N^T$ is given by:
\begin{equation}
    L_{\Psi} = L_{\mathcal{Q}} \cdot \mathbf{C}_{N, T} \cdot L_{\mathcal{P}}.
\end{equation}
Consequently, the numerical observability and stability properties derived for the hidden dynamics $\mathcal{L}$ hold for the end-to-end operator $\Psi$, scaled by the regularity of the input and output transformations.

Based on~\Cref{thm:global_stability_layers}, we can now state the main stability result, which bounds the total error on the grid,i.e., the discrepancy between the ground truth continuous operator $\smash{\Psi^{(T)}}$ applied to a clean input $v$, and its discretized counterpart $\Psi_N^{(T)}$ applied to $\smash{v^\delta}$, a discretized and noisy version of $v$.

\begin{theorem}[ISS under Discretization and Noise]
\label{thm:global_iss}
Let $v \in H^s(\mathbb{T}^d)$ be a continuous input function (with $s > d/2$) and $v^\delta \in \ell^2(\mathbb{T}_N^d)$ be a noisy discrete observation of $v$ such that $\smash{\|v^\delta - v|_{\mathbb{T}_N^d}\|_{l^2} \le \delta}$. The total error between the discrete output and the grid-projected ground truth satisfies:
\begin{align}
    \| \Psi_N^{(T)}(v^\delta) - \Psi^{(T)}&(v)|_{\mathbb{T}_N^d} \|_{\ell^2(\mathbb{T}_N^d)} \le \underbrace{\mathbf{C}_{N, T} \cdot \delta}_{\text{Input Perturbation}} + \underbrace{N^\beta B \frac{A^T - 1}{A - 1}}_{\text{Discretization Error}} \nonumber,
\end{align}
where $A$ and $B$ are the constants defined in ~\Cref{thm:discretization-error-ssnos}, $\beta < 0$ is the convergence rate derived in \Cref{lem:bounding-E1-polynomial-decay,lem:bounding-E1-truncation}, and $\mathbf{C}_{N, T}$ is defined in~\Cref{thm:global_stability_layers}.
\end{theorem}

\section{Experiments}
\begin{figure*}[t]
\vspace{-1.5cm}
    \centering
    \begin{minipage}[t]{0.95\textwidth}
    % \begin{subfigure}[t]{\linewidth}
    %     \includegraphics[width=\linewidth]{imgs/stacked_SSM1D_same_instance_relL2_multiGRF_s_0p2-0p5-1p0-1p5-2p0_N50 (1).pdf}
    %     \caption{1D with $\sigma = \textsc{GELU}$.}
    %     \label{fig:1D-gelu}
    % \end{subfigure}
    % \begin{subfigure}[t]{\linewidth}
    %     \centering
    %     \includegraphics[width=\linewidth]{imgs/stacked_SSM1D_same_instance_relL2_multiGRF_s_0p2-0p5-1p0-1p5-2p0_N50.pdf}
    %     \caption{1D with $\sigma = \textsc{ReLU}$.}
    %     \label{fig:1D-relu}
    % \end{subfigure}
    \begin{subfigure}[t]{\linewidth}
        \centering
        \includegraphics[width=0.9\linewidth]{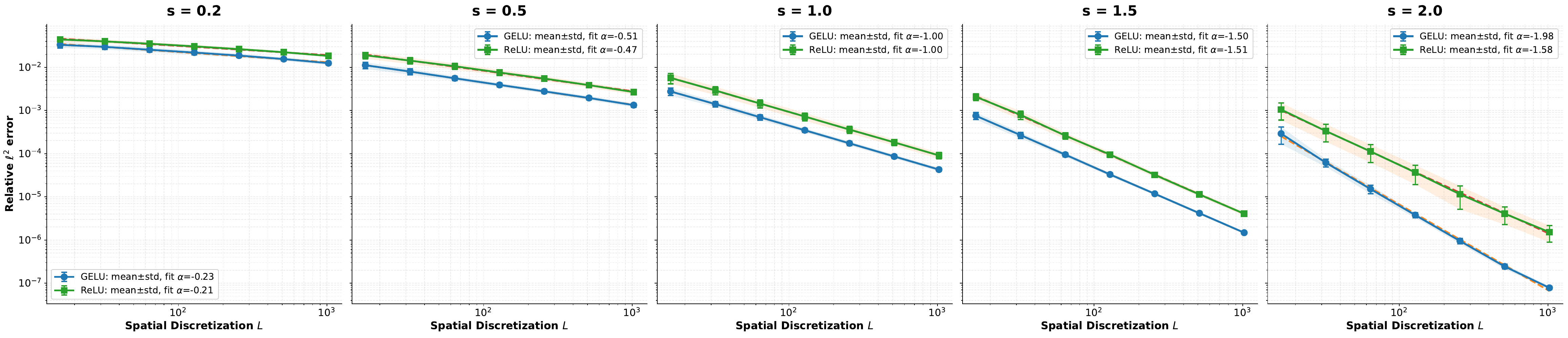}
        \caption{1D with $\sigma = \textsc{GELU}$ (in {\color{mplblue} blue}) and $\sigma = \textsc{ReLU}$ (in {\color{mplgreen} green}).}
        \label{fig:1D}
    \end{subfigure}
    \begin{subfigure}[t]{\linewidth}
        \centering
        \includegraphics[width=0.9\linewidth]{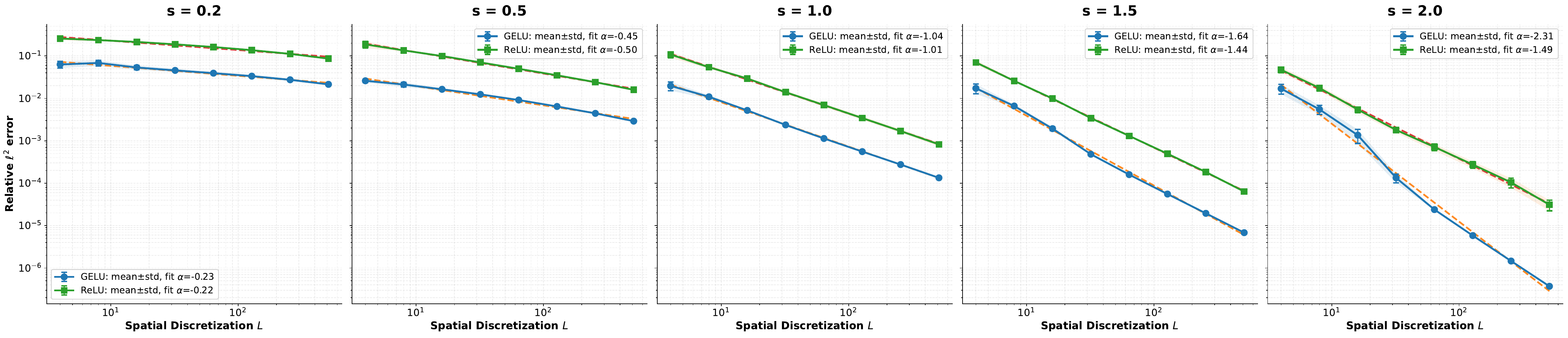}
        \caption{2D with $\sigma = \textsc{GELU}$ (in {\color{mplblue} blue}) and $\sigma = \textsc{ReLU}$ (in {\color{mplgreen} green}).}
        \label{fig:2D}
    \end{subfigure}
    \end{minipage}
    \caption{Relative $\ell^2$ error vs.\ resolution $L$ for GRF inputs of varying smoothness $s$ (mean $\pm$ std over $N = 50$ samples for $(a), (b)$ and $N = 5$ samples for $(c), (d)$).} %\mn{est ce qu'il possible de plus étoffer la description de la figure dans la caption?  et le rouge fit, on en dit rien !
    \label{fig:discretization-error}
    \vspace{-0.2cm}
\end{figure*}
%\abde{Uniform Grid.}

We now assess our theoretical bounds on discretization and stability error on numerical examples in 1D and 2D. Additional experiments on the 1D Burgers benchmark, including the effect of training and network depth, are provided in \Cref{appdx-sec:trained-burgers-depth}. These experiments show that the discretization-scaling behavior, predicted by our theoretical analysis, remains clearly visible even in the trained setting.

\subsection{Discretization Error of SS-NOs}
\label{sec:exp:disc-error-ssno}

We first turn to the empirical validation of the discretization-error predictions of
\Cref{thm:discretization-error-ssnos} in a controlled setting where the only source of variation is the spatial sampling of the input function. In particular, we isolate the \emph{numerical} error induced by approximating the integral convolution with its grid quadrature \eqref{9754cccd-2dc7-4de6-a4a3-bdb8b9bbf6c9-2}, separately from training effects, finite data, or optimization noise.

\paragraph{Ground truth and layerwise comparison.}
As is standard in numerical analysis when the true continuous quantity is unavailable, we define a high-resolution discrete SS-NO evaluation as our reference (“ground truth”), and measure discrepancies when the same operator is evaluated from coarser discretizations. Concretely, for each experiment we fix a single SS-NO architecture $\Psi_\theta$ and a single set of parameters
$\theta$ throughout. For a given input $u$ we compute the sequence of hidden states $\{v_\ell^{\mathrm{ref}}\}_{\ell=0}^{T}$ on a finest grid of size $N_{\mathrm{full}}^{\mathrm{1D}}=8196$ in 1D and $N_{\mathrm{full}}^{\mathrm{2D}}=2048$ in 2D\footnote{All grids are uniform on $\T^d$, identified with $\T^d_N \cong \frac1N [N]^d$.}.
For each subsampling factor $s\in\mathcal{S}$ we construct the coarsened input $u^{(s)}$ by uniform sampling with a stride of $s$ and evaluate the \emph{same} SS-NO parameters on the corresponding coarse grid of size $N^{(s)}=N_{\mathrm{full}}/s$.
To compare outputs across resolutions, we represent all states on the same set of spatial resolution (coarse-grid states are lifted to the finest grid).

\paragraph{Error metric.}
For the last layer $T-1$ we compute the relative $\ell^2$ error on the finest grid:
\begin{equation*}
    \mathrm{RelErr}_{T}(N^{(s)}) :=  \frac{\bigl\|\,\widetilde v_T^{(s)} - v_T^{\mathrm{ref}}\,\bigr\|_{\ell^2({\T^d_{N_{\mathrm{full}}}})}}
    {\bigl\|\,v_T^{\mathrm{ref}}\,\bigr\|_{\ell^2({\T^d_{N_{\mathrm{full}}}})}},
    \label{eq:relerr_layerwise}
\end{equation*}
where $\widetilde v_T^{(s)}$ denotes the interpolated version of the coarse-grid state $\smash{v_T^{(s)}}$. We average results over multiple independent input realizations ($50$ in 1D and $5$ in 2D), showing mean $\pm$ standard deviation.

\begin{wrapfigure}{r}{0.35\textwidth}
%\vspace{-1.5cm}
    \centering
    \begin{subfigure}[t]{0.3\columnwidth}
        \centering%
        \includegraphics[width=\linewidth]{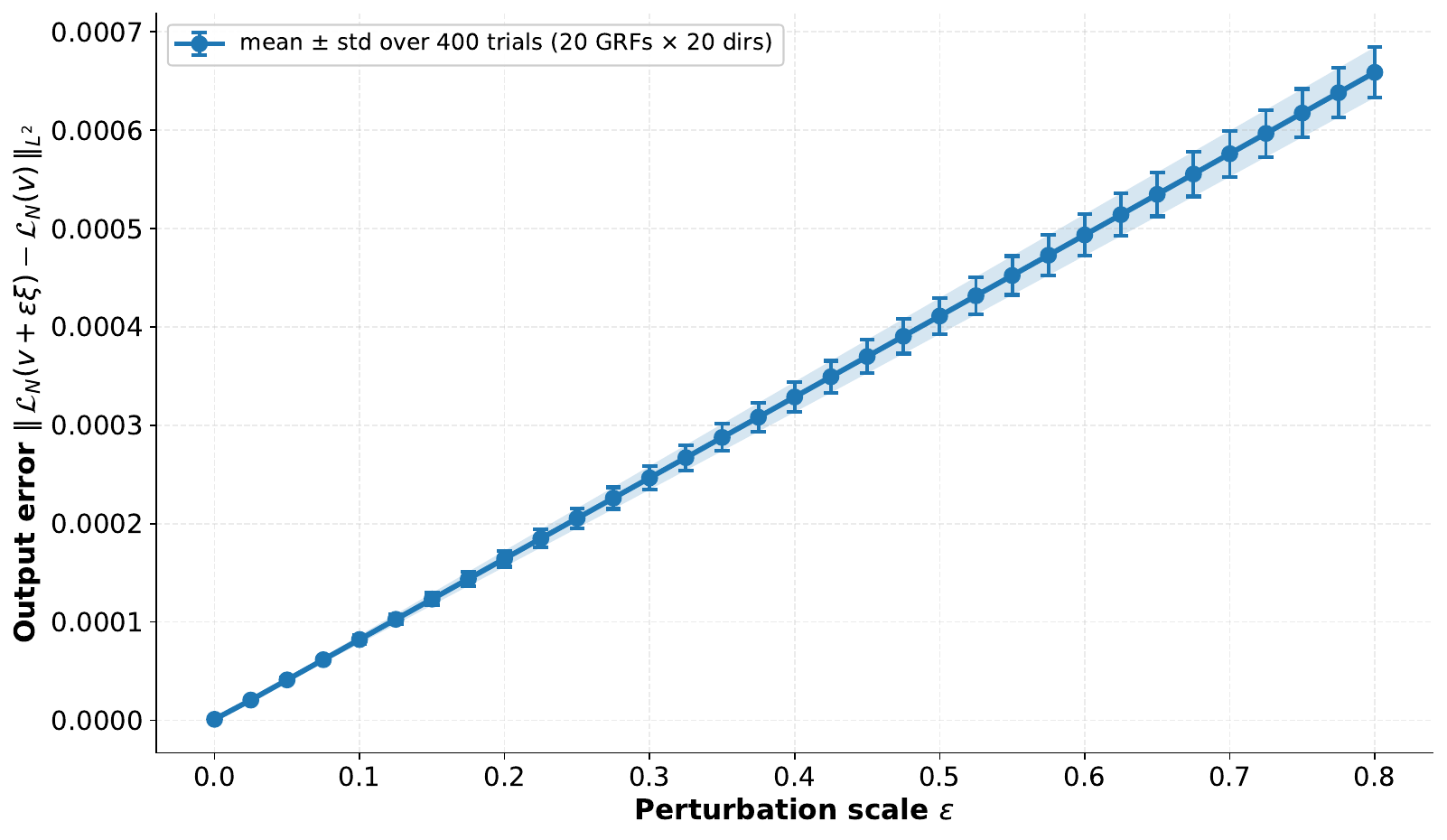}
        \caption{1D case.}
        \label{fig:1D-gelu-stability}
    \end{subfigure}
    \begin{subfigure}[t]{0.3\columnwidth}
        \centering
        \includegraphics[width=\linewidth]{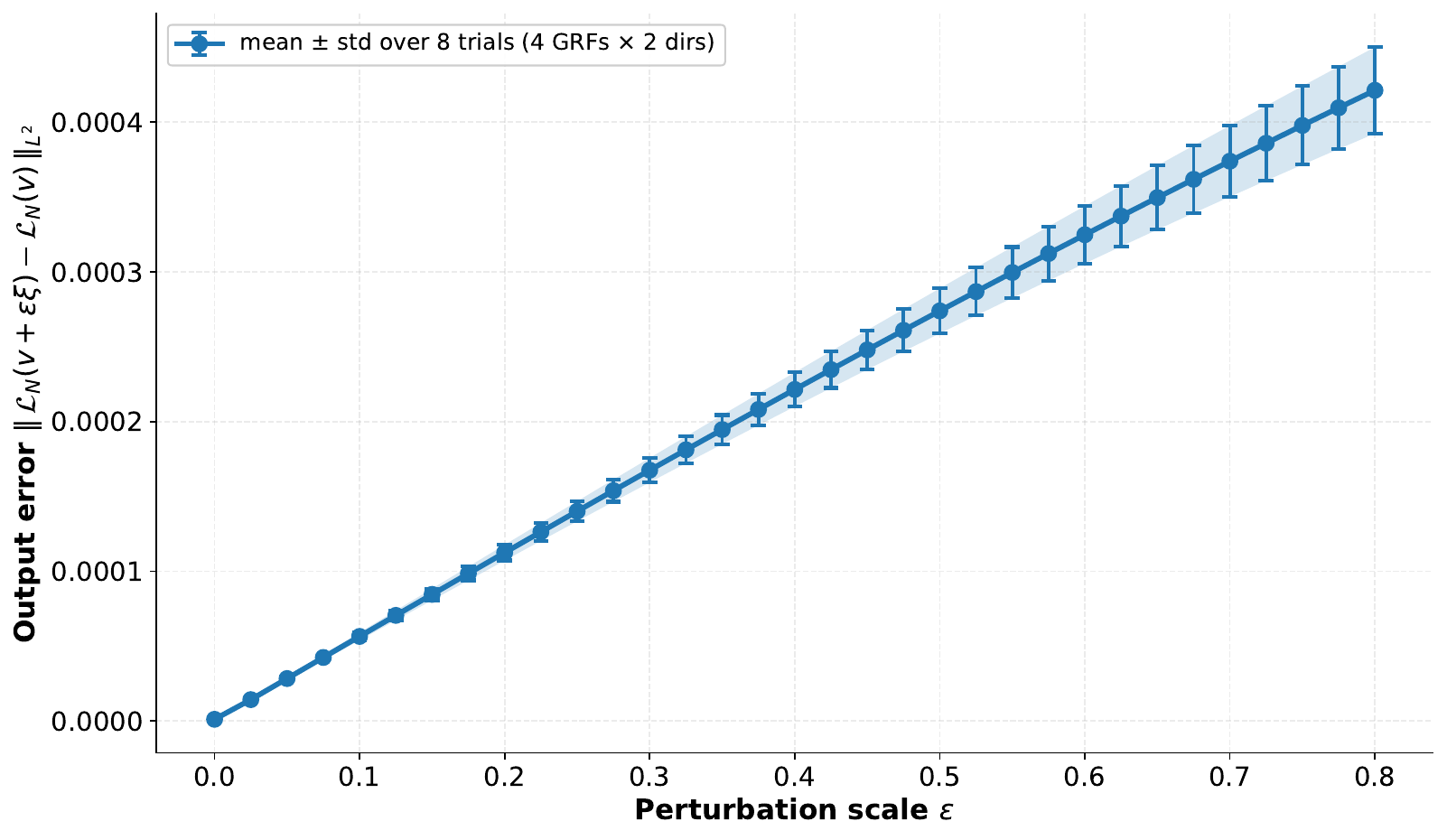}
        \caption{2D case.}
        \label{fig:2D-gelu-stability}
    \end{subfigure}
    \caption{Stability of a single discretized SS-NO layer (mean $\pm$ std over $N = 400$ samples ($(a)$) and $N = 8$ samples ($(b)$).}
    \label{fig:1D-2D-stability}
   s \vspace{-1cm}
\end{wrapfigure}

\paragraph{Input regularity via Gaussian random fields.}
To assess the dependency on Sobolev regularity, we generate Gaussian random fields (GRFs) inputs with fixed smoothness (\Cref{sec:datasets}). We evaluate on several smoothness levels (as shown in the figure caption), spanning low-regularity fields to highly smooth fields. This targets the exponent $\beta$ in \Cref{thm:discretization-error-ssnos}, which predicts algebraic decay in $N$ governed by (i) the input regularity $s$ and (ii) the high-frequency decay of the kernel Fourier coefficients (see \Cref{lem:bounding-E1-polynomial-decay,lem:bounding-E1-truncation}), and (iii) the regularity of the activation function (see~\Cref{lem:Moser}).

\paragraph{Models: random weights and activation regularity.}
Unless specified otherwise, SS-NO parameters are \emph{randomly initialized} and then kept fixed. This eliminates any confounding between discretization and training, aligning with the theoretical setting where $\theta$ is fixed while discretization varies. To assess the effect of the activation function regularity on discretization error propagation across layers, we compare the default smooth \textsc{GELU} activation with the non-smooth \textsc{ReLU}. This comparison is motivated by the error decomposition $\smash{\cE_{\ell+1}^{(0)}}$ in \eqref{b06ec7af-e2f9-46dc-ac2e-244f90be7c6d-3} and the role of composition estimates (\Cref{lem:Moser,lem:heredity-H-s}) in controlling the growth of Sobolev norms across layers.

\paragraph{Results.} \Cref{fig:1D,fig:2D} jointly confirm~\Cref{thm:discretization-error-ssnos}: in both 1D and 2D, $\mathrm{RelErr}_T(N)$ decays with the resolution $N$, improving with input smoothness at first (consistent with the discretization term being controlled by input regularity, e.g.\ the $N^{\frac d2-s}$ contribution in \Cref{lem:bounding-E1-truncation}) but eventually saturating in the case of activation functions with finite Sobolev regularity (e.g., \textsc{ReLU}).

\subsection{Stability in 1D and 2D}

\iffalse
To empirically illustrate the Lipschitz-type stability bound at the level of a single discretized SS-NO layer, we perform a controlled perturbation experiment in which the \emph{only} varying quantity is the input function. In both 1D and 2D we fix a discretized SS-NO layer $\mathcal{L}_N$ (i.e., a fixed architecture and a single frozen set of parameters) and sample a input function $v$ as a Gaussian random field on the uniform grid $\T_N^d$ with prescribed smoothness $s$ (here $s=2$; cf.\ \Cref{sec:datasets}). For each $v$ we generate independent Gaussian perturbation directions $\xi$ (normalized to unit discrete $\ell^2$ norm) and evaluate the perturbed inputs $v_\varepsilon := v + \varepsilon \xi$ for a linearly spaced range of amplitudes $\varepsilon \in [0, 0.8]$. We then record the output discrepancy
    $  f(\varepsilon) := \normgridvec{ \mathcal{L}_N(v_\varepsilon)  - \mathcal{L}_N(v) }, $ 
and report mean $\pm$ standard deviation of $f(\varepsilon)$ over multiple trials (in 1D: $20$ independent GRFs $\times$ $20$ directions; in 2D: $4$ GRFs $\times$ $2$ directions), shown in Figures~\ref{fig:1D-gelu-stability}--\ref{fig:2D-gelu-stability}.
\fi

To empirically illustrate \Cref{thm:global_stability_layers}, we perturb the input of a fixed discretized SS-NO layer. In 1D and 2D, we sample a GRF input $v$ with smoothness $s=2$, draw normalized Gaussian directions $\xi$, and evaluate $v_\varepsilon=v+\varepsilon\xi$ for $\varepsilon\in[0,0.8]$. We report
\[
f(\varepsilon):=\normgridvec{\mathcal{L}_N(v_\varepsilon)-\mathcal{L}_N(v)}
\]
as mean $\pm$ std over multiple trials; see Figures~\ref{fig:1D-gelu-stability}--\ref{fig:2D-gelu-stability}.

\iffalse
In both 1D and 2D we observe an essentially \emph{linear} growth of the output error as a function of the perturbation amplitude $\varepsilon$, with modest dispersion across realizations; this is exactly the qualitative behavior predicted by \Cref{thm:global_stability_layers}, which bounds $\normgridvec{\mathcal{L}_N v - \mathcal{L}_N w}$ by a constant times an input-distance term. In particular, the near-affine dependency on $f(\varepsilon)$ indicates that, for the explored perturbation magnitudes, the discretized layer behaves as a stable (locally Lipschitz) map on typical GRF inputs, and that its sensitivity does not exhibit blow-up as $\varepsilon$ increases on this range. 
%The larger spread in 2D is consistent with the increased variability of high-dimensional inputs and the heavier interaction between spatial mixing and channel dynamics, but the same linear scaling trend persists.

A dedicated empirical analysis of the interplay between depth and stability on 1D Gaussian random fields is provided in~\Cref{appdx-sec:depth-GRF-stability}.
\fi

In both 1D and 2D, the output error grows nearly linearly with $\varepsilon$, with modest dispersion across realizations. This agrees with \Cref{thm:global_stability_layers}, predicting Lipschitz-type dependence on input perturbations. Thus, over the tested range, the discretized SS-NO layer behaves as a stable map on typical GRF inputs, without visible sensitivity blow-up. A depth-stability study is provided in~\Cref{appdx-sec:depth-GRF-stability}.

\iffalse
\begin{figure}[t]
    \centering
    \begin{subfigure}[b]{0.49\columnwidth}
        \includegraphics[width=\linewidth]{imgs/kernel_imgs/SSNO_Fwd_layer-idx=0_L=2^15_random.png}
        \caption{SSNO (forward) kernel}
        \label{fig:ssno-kernel}
    \end{subfigure}
    \hfill
    \begin{subfigure}[b]{0.49\columnwidth}
        \includegraphics[width=\linewidth]{imgs/kernel_imgs/FNO_layer-idx=0_L=2^15_random.png}
        \caption{FNO kernel}
        \label{fig:fno-kernel}
    \end{subfigure}
    \caption{
        Comparison of the SSNO kernel and a Fourier Neural Operator (FNO) kernel
        (both \emph{randomly initialized}) evaluated at spatial resolution $L = 2^{15}$ on $\intff{0}{1}$.  
        SSNO exhibits continuous, full-band spectral structure without mode truncation,
        while FNO enforces a hard cutoff frequency (here at $K = 16$), resulting in compact spectral support.
    }
    \label{fig:operator-norm-ssno-vs-fno}
\end{figure}
\fi

%\vfill

\section{Conclusion}
This work develops a rigorous continuous-to-discrete theory for convolution-based neural operators, with a focus on SS-NOs. We derive explicit discretization-error bounds that yield an algebraic convergence rate in the spatial resolution under standard Sobolev regularity assumptions and mild spectral conditions on the kernels (e.g., frequency cutoff or decay). In parallel, we establish stability estimates both in the continuous model and for its discretized implementation, that quantify how input perturbations and numerical errors propagate across layers. Numerical experiments in 1D and 2D corroborate our theory and highlight the practical role of regularity assumptions. Overall, these results provide a rigorous pathway to reason about resolution generalization and stability of neural operators. They also suggest principled guidelines for choosing kernels, activations, and depth, to control discretization effects in operator learning.

\clearpage
\bibliography{draft/bib}

\begin{thebibliography}{27}
\providecommand{\natexlab}[1]{#1}
\providecommand{\url}[1]{\texttt{#1}}
\expandafter\ifx\csname urlstyle\endcsname\relax
  \providecommand{\doi}[1]{doi: #1}\else
  \providecommand{\doi}{doi: \begingroup \urlstyle{rm}\Url}\fi

\bibitem[Anandkumar et~al.(2020)Anandkumar, Azizzadenesheli, Bhattacharya,
  Kovachki, Li, Liu, and Stuart]{anandkumar2020neural}
Anima Anandkumar, Kamyar Azizzadenesheli, Kaushik Bhattacharya, Nikola
  Kovachki, Zongyi Li, Burigede Liu, and Andrew Stuart.
\newblock Neural operator: Graph kernel network for partial differential
  equations.
\newblock In \emph{ICLR 2020 workshop on integration of deep neural models and
  differential equations}, 2020.

\bibitem[Bahouri et~al.(2011{\natexlab{a}})Bahouri, Chemin, and
  Danchin]{Bahouri2011}
Hajer Bahouri, Jean-Yves Chemin, and Rapha{\"e}l Danchin.
\newblock \emph{Basic Analysis}, pages 1--50.
\newblock Springer Berlin Heidelberg, Berlin, Heidelberg, 2011{\natexlab{a}}.
\newblock ISBN 978-3-642-16830-7.
\newblock \doi{10.1007/978-3-642-16830-7_1}.
\newblock URL \url{https://doi.org/10.1007/978-3-642-16830-7_1}.

\bibitem[Bahouri et~al.(2011{\natexlab{b}})Bahouri, Chemin, and
  Danchin]{Bahouri2011chap2}
Hajer Bahouri, Jean-Yves Chemin, and Rapha{\"e}l Danchin.
\newblock \emph{Littlewood--Paley Theory}, pages 51--121.
\newblock Springer Berlin Heidelberg, Berlin, Heidelberg, 2011{\natexlab{b}}.
\newblock ISBN 978-3-642-16830-7.
\newblock \doi{10.1007/978-3-642-16830-7_2}.
\newblock URL \url{https://doi.org/10.1007/978-3-642-16830-7_2}.

\bibitem[Boull{\'e} and Townsend(2024)]{boulle2024mathematical}
Nicolas Boull{\'e} and Alex Townsend.
\newblock A mathematical guide to operator learning.
\newblock In \emph{Handbook of Numerical Analysis}. Elsevier, 2024.

\bibitem[Bourdaud and Kateb(1992)]{Bourdaud1992}
G.~Bourdaud and M.~E.~D. Kateb.
\newblock Calcul fonctionnel dans l'espace de sobolev fractionnaire.
\newblock \emph{Mathematische Zeitschrift}, 210\penalty0 (1):\penalty0
  607--613, Dec 1992.
\newblock ISSN 1432-1823.
\newblock \doi{10.1007/BF02571817}.
\newblock URL \url{https://doi.org/10.1007/BF02571817}.

\bibitem[Bourdaud et~al.(2014)Bourdaud, Moussai, and Sickel]{Bourdaud2014}
G{\'e}rard Bourdaud, Madani Moussai, and Winfried Sickel.
\newblock Composition operators acting on besov spaces on the real line.
\newblock \emph{Annali di Matematica Pura ed Applicata (1923 -)}, 193\penalty0
  (5):\penalty0 1519--1554, Oct 2014.
\newblock ISSN 1618-1891.
\newblock \doi{10.1007/s10231-013-0342-x}.
\newblock URL \url{https://doi.org/10.1007/s10231-013-0342-x}.

\bibitem[Brunton and Kutz(2023)]{brunton2023machine}
Steven~L Brunton and J~Nathan Kutz.
\newblock Machine learning for partial differential equations.
\newblock \emph{arXiv preprint arXiv:2303.17078}, 2023.

\bibitem[Dai et~al.(2020)Dai, Hu, Wu, and Xiao]{doi:10.1142/S0219530519500234}
Yichen Dai, Weiwei Hu, Jiahong Wu, and Bei Xiao.
\newblock The littlewood–paley decomposition for periodic functions and
  applications to the boussinesq equations.
\newblock \emph{Analysis and Applications}, 18\penalty0 (04):\penalty0
  639--682, 2020.
\newblock \doi{10.1142/S0219530519500234}.
\newblock URL \url{https://doi.org/10.1142/S0219530519500234}.

\bibitem[Folland(1999)]{folland1999real}
G.B. Folland.
\newblock \emph{Real Analysis: Modern Techniques and Their Applications}.
\newblock Pure and Applied Mathematics: A Wiley Series of Texts, Monographs and
  Tracts. Wiley, 1999.
\newblock ISBN 9780471317166.
\newblock URL \url{https://books.google.fr/books?id=N8jVDwAAQBAJ}.

\bibitem[Gin et~al.(2020)Gin, Shea, Brunton, and
  Kutz]{gin2020deepgreendeeplearninggreens}
Craig~R. Gin, Daniel~E. Shea, Steven~L. Brunton, and J.~Nathan. Kutz.
\newblock Deepgreen: Deep learning of green's functions for nonlinear boundary
  value problems, 2020.
\newblock URL \url{https://arxiv.org/abs/2101.07206}.

\bibitem[Gu et~al.(2022)Gu, Goel, Gupta, and R{\'e}]{gu2022parameterization}
Albert Gu, Karan Goel, Ankit Gupta, and Christopher R{\'e}.
\newblock On the parameterization and initialization of diagonal state space
  models.
\newblock \emph{Advances in Neural Information Processing Systems}, 2022.

\bibitem[Koren and Lanthaler(2025)]{anonymous2025merging}
Nodens Koren and Samuel Lanthaler.
\newblock Merging memory and space: A state space neural operator.
\newblock \emph{Submitted to Transactions on Machine Learning Research}, 2025.
\newblock URL \url{https://openreview.net/forum?id=SwLxxz0x58}.
\newblock Under review.

\bibitem[Kovachki et~al.(2021)Kovachki, Lanthaler, and
  Mishra]{kovachki2021universalapproximationerrorbounds}
Nikola Kovachki, Samuel Lanthaler, and Siddhartha Mishra.
\newblock On universal approximation and error bounds for fourier neural
  operators, 2021.
\newblock URL \url{https://arxiv.org/abs/2107.07562}.

\bibitem[Kovachki et~al.(2023)Kovachki, Li, Liu, Azizzadenesheli, Bhattacharya,
  Stuart, and Anandkumar]{Kovachki-NeuralOperator:20223}
Nikola Kovachki, Zongyi Li, Burigede Liu, Kamyar Azizzadenesheli, Kaushik
  Bhattacharya, Andrew Stuart, and Anima Anandkumar.
\newblock Neural operator: Learning maps between function spaces with
  applications to pdes.
\newblock \emph{Journal of Machine Learning Research}, 24\penalty0
  (89):\penalty0 1--97, 2023.
\newblock URL \url{http://jmlr.org/papers/v24/21-1524.html}.

\bibitem[Kovachki et~al.(2024)Kovachki, Lanthaler, and
  Stuart]{kovachki2024operator}
Nikola~B Kovachki, Samuel Lanthaler, and Andrew~M Stuart.
\newblock Operator learning: Algorithms and analysis.
\newblock \emph{Handbook of Numerical Analysis}, 25:\penalty0 419--467, 2024.

\bibitem[Lanthaler et~al.(2025)Lanthaler, Stuart, and
  Trautner]{lanthaler2025discretizationerrorfourierneural}
Samuel Lanthaler, Andrew~M. Stuart, and Margaret Trautner.
\newblock Discretization error of fourier neural operators, 2025.
\newblock URL \url{https://arxiv.org/abs/2405.02221}.

\bibitem[Li et~al.(2020)Li, Kovachki, Azizzadenesheli, Liu, Stuart,
  Bhattacharya, and Anandkumar]{li2020multipole}
Zongyi Li, Nikola Kovachki, Kamyar Azizzadenesheli, Burigede Liu, Andrew
  Stuart, Kaushik Bhattacharya, and Anima Anandkumar.
\newblock Multipole graph neural operator for parametric partial differential
  equations.
\newblock \emph{Advances in Neural Information Processing Systems},
  33:\penalty0 6755--6766, 2020.

\bibitem[Li et~al.(2021)Li, Kovachki, Azizzadenesheli, Bhattacharya, Stuart,
  Anandkumar, et~al.]{lifourier}
Zongyi Li, Nikola~Borislavov Kovachki, Kamyar Azizzadenesheli, Kaushik
  Bhattacharya, Andrew Stuart, Anima Anandkumar, et~al.
\newblock Fourier neural operator for parametric partial differential
  equations.
\newblock In \emph{International Conference on Learning Representations}, 2021.

\bibitem[Li et~al.(2023{\natexlab{a}})Li, Huang, Liu, and
  Anandkumar]{li2023fourier}
Zongyi Li, Daniel~Zhengyu Huang, Burigede Liu, and Anima Anandkumar.
\newblock Fourier neural operator with learned deformations for pdes on general
  geometries.
\newblock \emph{Journal of Machine Learning Research}, 2023{\natexlab{a}}.

\bibitem[Li et~al.(2023{\natexlab{b}})Li, Kovachki, Choy, Li, Kossaifi, Otta,
  Nabian, Stadler, Hundt, Azizzadenesheli, and Anandkumar]{Li23-gino}
Zongyi Li, Nikola Kovachki, Chris Choy, Boyi Li, Jean Kossaifi, Shourya Otta,
  Mohammad~Amin Nabian, Maximilian Stadler, Christian Hundt, Kamyar
  Azizzadenesheli, and Animashree Anandkumar.
\newblock Geometry-informed neural operator for large-scale 3d pdes.
\newblock In A.~Oh, T.~Naumann, A.~Globerson, K.~Saenko, M.~Hardt, and
  S.~Levine, editors, \emph{Advances in Neural Information Processing Systems},
  volume~36, pages 35836--35854. Curran Associates, Inc., 2023{\natexlab{b}}.

\bibitem[Li et~al.(2024)Li, Zheng, Kovachki, Jin, Chen, Liu, Azizzadenesheli,
  and Anandkumar]{li2024physics}
Zongyi Li, Hongkai Zheng, Nikola Kovachki, David Jin, Haoxuan Chen, Burigede
  Liu, Kamyar Azizzadenesheli, and Anima Anandkumar.
\newblock Physics-informed neural operator for learning partial differential
  equations.
\newblock \emph{ACM/IMS Journal of Data Science}, 1\penalty0 (3):\penalty0
  1--27, 2024.

\bibitem[Lu et~al.(2021)Lu, Jin, Pang, Zhang, and
  Karniadakis]{learning-nonlinear-operators-deeponets}
Lu~Lu, Pengzhan Jin, Guofei Pang, Zhongqiang Zhang, and George~Em Karniadakis.
\newblock Learning nonlinear operators via {DeepONet} based on the universal
  approximation theorem of operators.
\newblock \emph{Nature Machine Intelligence}, 3\penalty0 (3):\penalty0
  218--229, March 2021.

\bibitem[Runst and Sickel(1996)]{RunstSickel+1996}
Thomas Runst and Winfried Sickel.
\newblock \emph{Sobolev Spaces of Fractional Order, Nemytskij Operators, and
  Nonlinear Partial Differential Equations}.
\newblock De Gruyter, Berlin, New York, 1996.
\newblock ISBN 9783110812411.
\newblock \doi{doi:10.1515/9783110812411}.
\newblock URL \url{https://doi.org/10.1515/9783110812411}.

\bibitem[Schmeisser and Triebel(1987)]{SchmeisserTriebel1987}
Hans-J{\"u}rgen Schmeisser and Hans Triebel.
\newblock \emph{Topics in Fourier Analysis and Function Spaces}.
\newblock John Wiley \& Sons, Chichester, 1987.
\newblock ISBN 978-0471910806.

\bibitem[Taylor(2023)]{Taylor2023chap1}
Michael~E. Taylor.
\newblock \emph{Function Space and Operator Theory for Nonlinear Analysis},
  pages 1--106.
\newblock Springer International Publishing, Cham, 2023.
\newblock ISBN 978-3-031-33928-8.
\newblock \doi{10.1007/978-3-031-33928-8_1}.
\newblock URL \url{https://doi.org/10.1007/978-3-031-33928-8_1}.

\bibitem[Tong et~al.(2025)Tong, Trung-Dung, Liu, den Broeck, and
  Niepert]{tong2025learningdiscretizedenoisingdiffusion}
Vinh Tong, Hoang Trung-Dung, Anji Liu, Guy~Van den Broeck, and Mathias Niepert.
\newblock Learning to discretize denoising diffusion odes, 2025.
\newblock URL \url{https://arxiv.org/abs/2405.15506}.

\bibitem[Triebel(1983)]{Triebel1983chap2}
Hans Triebel.
\newblock \emph{Function Spaces on Rn}, pages 33--187.
\newblock Springer Basel, Basel, 1983.
\newblock ISBN 978-3-0346-0416-1.
\newblock \doi{10.1007/978-3-0346-0416-1_2}.
\newblock URL \url{https://doi.org/10.1007/978-3-0346-0416-1_2}.

\end{thebibliography}
\bibliographystyle{plainnat}

\appendix
\onecolumn

\tableofcontents

\clearpage

\iffalse
\section{Fourier Analysis on $\T^d$ and $\T^d_N$}\label{appdx-sec:fourier-analysis-torus}

\adri{brief introductory paragraph + add precise definition of the flat torus $\T^d$, say that it is a compact $\mathscr{C}^{\infty}$ manifold without boundary, identified with its \emph{canonical} fundamental domain $\intff{0}{1}^d$. Also say we identify functions $v \colon \T^d \to \R^H$ with $\Z^d$-periodic functions from $\R^d \to \R^H$}

We recall the definition of the grid $\T^d_N$ used to approximate the integral:
 \[ \T^d_N := \frac{1}{N} [N]^d = \enstq{\left( \frac{i_1}{N}, \ldots, \frac{i_d}{N} \right) \in \N_0^d}{\text{for all $j \in \ens{1, \ldots, d}$, $0 \le i_j < N$}}. \]

\fi

\section{Fourier Analysis on $\T^d$ and $\T^d_N$}\label{appdx-sec:fourier-analysis-torus}

In this section we briefly recall the basic notions of Fourier analysis on the flat torus and on its discrete approximation. The $d$--dimensional flat torus $\T^d$ is defined as the quotient space
\[ \T^d := \R^d /\, \Z^d, \]
that is, $\R^d$ with points differing by an element of $\Z^d$ are identified. It is a compact $\mathscr{C}^{\infty}$ manifold without boundary and can be canonically identified with the fundamental domain
\[ [0,1)^d \subset \R^d. \]
Throughout, we identify functions $v \colon \T^d \to \R^H$ with $\Z^d$--periodic functions $v \colon \R^d \to \R^H$. We also recall the definition of the uniform grid $\T^d_N$ used to approximate integrals over $\T^d$:
\[ \T^d_N \simeq \frac{1}{N}[N]^d := \enstq{\left( \frac{i_1}{N}, \ldots, \frac{i_d}{N} \right) \in \N_0^d}
{\text{for all $j \in \ens{1,\ldots,d}$, $0 \le i_j < N$}}. \]

\subsection{Fourier Transform and Inversion Formula}

\begin{definition}[Fourier Transform on $\T^d$]\label{def:continuous-fourier-transform}
    For $f \colon \T^d \to \C^m$, the continuous Fourier transform on $\T^d$ is defined by:
        \[\widehat{f}(\xi) := \int_{\T^d} f(x) e^{-2 i \pi \xi \cdot x} \odif{x}, \quad \xi \in \Z^d.\]
    Then the inversion formula is
        \[ f(x) = \sum_{\xi \in \Z^d} \widehat{f}(\xi) e^{2 i \pi \xi \cdot x}, \quad x \in \T^d.\]
\end{definition}
 
\begin{definition}[Discrete Fourier Transform (on $(\nicefrac{\Z}{N \Z}^d$)]\label{def:discrete-fourier-transform}
    For $u \colon \T^d_N \to \C^m$, the discrete Fourier transform on $(\nicefrac{\Z}{N \Z})^d$ is defined by:
        \[\widehat{u}_N(\xi) := \dfrac{1}{N^d} \sum_{x \in \T^d_N} u(x) e^{-2 i \pi \xi \cdot x}, \quad \xi \in \ens{0, 1, \ldots, N - 1}^d.\]
    Then the inversion formula is
        \[ u(x) = \sum_{\xi \in \ens{0, 1, \ldots, N - 1}^d} \widehat{u}_N(\xi) e^{2 i \pi \xi \cdot x}, \quad x \in \T^d_N.\]
\end{definition}

\subsection{Convolution Operator and Some Properties}

We state below some properties of the Fourier transform and its discrete counterpart.
\begin{lemma}[Fourier Transform and Convolution]
    For any kernel $\cK \in L^2(\T^d, \R^{m \times n})$ and any function $v \in L^2(\T^d, \R^n)$, we have:
        \[  \widehat{(\cK * v)}(\xi) = \widehat{\cK} (\xi) \widehat{v}(\xi), \quad \text{for all $\xi \in \Z^d$}. \]
\end{lemma}

%\adri{si si, on préfère avoir plus ou moins ça à porter de main au cas où? Pour le moment c'est encore un peu brouillon la preuve du lemme 3.10 et il est possible que suivant le noyeu de convolution $\cK_t$ utilisé, l'erreur liée à la discrétisation de la convolution empire la dépendance en $N$, nous n'en savons pas davantage pour le moment, nous finissions la preuve globale et nous raffinerons si besoin est (et il y en aura besoin de la faire :D)}

\begin{proof}
Let $\xi \in \Z^d$, we have
\begin{align*}
    \widehat{(\cK * v)}(\xi) &= \int_{\T^d} (\cK * v) (x) e^{-2i\pi \xi \cdot x} \odif{x} \\
    &= \int_{\T^d} \int_{\T^d} \cK(x - y) v(y) e^{-2i\pi \xi \cdot x} \odif{y} \odif{x} \\
    \oversetrel{rel:128c6b01-cfb6-4211-a71c-33ec1a585f5d}&{=} \int_{\T^d} \left( \int_{\T^d} \cK(x - y) e^{-2i\pi \xi \cdot x} \odif{x} \right)  v(y) \odif{y} \\
    &= \int_{\T^d}  \left( \int_{\T^d} \cK(x - y) e^{-2i\pi \xi \cdot (x-y)} \odif{x} \right) v(y) e^{-2i \pi \xi \cdot y} \odif{y} \\
    &= \widehat{\cK}(\xi) \int_{\T^d}  v(y) e^{-2i \pi \xi \cdot y}  \odif{y} \\
    &= \widehat{\cK}(\xi) \widehat{v}(\xi),
\end{align*}
as claimed. The permutation of integrals in~\relref{rel:128c6b01-cfb6-4211-a71c-33ec1a585f5d} follows from  the Fubini–Tonelli theorem and the fact that on the compact $\T^d$,
    \[ \norm{\cdot}_{L^1(\T^d)} \le \sqrt{\abs{\mathrm{Vol}(\T^d)}} \norm{\cdot}_{L^2(\T^d)} = \norm{\cdot}_{L^2(\T^d)}, \]
which follows by Cauchy-Schwarz's inequality.
\end{proof}

\begin{lemma}[Fourier Transform and Product]\label{appdx-lem:fourier-transform-product-functions}
    For any kernel $\cK \in L^2(\T^d, \R^{m \times n})$ and any function $v \in L^2(\T^d, \R^n)$, we have:
        \[  \widehat{(\cK \times v)}(\xi) = \left( \widehat{\cK} * \widehat{v} \right)(\xi) = \sum_{\eta \in \Z^d} \widehat{\cK}(\xi - \eta) \widehat{v}(\eta), \quad \text{for all $\xi \in \Z^d$}. \]
\end{lemma}

\begin{proof}
    We have:
    \begin{align*}
        \widehat{(\cK \times v)}(\xi) &= \int_{\T^d} \cK (x) v(x) e^{- 2 i \pi \xi \cdot x} \odif{x} \\
        &= \int_{\T^d} \cK (x) \sum_{\eta \in \Z^d} \widehat{v}(\eta) e^{2i \pi \eta \cdot x} e^{- 2 i \pi \xi \cdot x} \odif{x} \\ 
        &= \sum_{\eta \in \Z^d} \int_{\T^d} \cK(x) e^{-2 i \pi (\xi - \eta) \cdot x} \odif{x} \, \widehat{v}(\eta) \\
        &= \sum_{\eta \in \Z^d} \widehat{\cK}(\xi - \eta) \widehat{v}(\eta) \\
        &=  \left( \widehat{\cK} * \widehat{v} \right)(\xi),
    \end{align*}
    which concludes the proof.
\end{proof}

\begin{lemma}\label{appdx-lem:sum-over-grid}
    For any $\xi \in \Z^d$, we have
        \[ \sum_{k \in \ens{0, 1, \ldots, N - 1}^d} e^{\frac{2 i \pi}{N}\,  \xi \cdot k} = \begin{cases}
            N^d, & \text{if $\xi \equiv 0\, [N]$;} \\
            0, & \text{otherwise.}
        \end{cases} \]
\end{lemma}

\begin{proof}
    Let us fix $\xi = (\xi_1, \ldots, \xi_d) \in \Z^d$ we have
        \begin{align*}
            \sum_{k \in \ens{0, 1, \ldots, N - 1}^d} e^{\frac{2 i \pi}{N}\,  \xi \cdot k} &= \sum_{k = (k_1, \ldots, k_d) \in \ens{0, \dots, N-1}^d} \prod_{r = 1}^d e^{\frac{2 i \pi}{N} \, \xi_r k_r} \\
            &= \prod_{r = 1}^d \sum_{k_r = 0}^{N - 1} e^{\frac{2 i \pi}{N} \, \xi_r k_r},
        \end{align*}
    and, for $r \in \ens{1, 2, \ldots, d}$ fixed, we have
        \[ \sum_{k_r = 0}^{N - 1} e^{\frac{2 i \pi}{N} \, \xi_r k_r} = \begin{cases}
            N, & \text{ if } \xi_r \equiv 0\, [N]; \\
            0, & \text{ otherwise;} \\
        \end{cases} \]
    hence, this gives
        \begin{align*}
            \sum_{k \in \ens{0, 1, \ldots, N - 1}^d} e^{\frac{2 i \pi}{N}\,  \xi \cdot k} &= \prod_{r = 1}^N N \mathbbm{1}_{\ens{\xi_r \equiv 0\, [N]}} \\
            &= N^d \mathbbm{1}_{\ens{\xi \equiv 0\, [N]}},
        \end{align*}
    as claimed.
\end{proof}

\begin{lemma}
    For any function $f \in L^2(\T^d)$, we have
        \[ \dfrac{1}{N^d} \sum_{y \in \T^d_N} f(y) = \sum_{m \in \Z^d} \widehat{f}(m N), \]
    and
        \[ \int_{\T^d} f(y) \odif{y} = \widehat{f}(0). \]
\end{lemma}

\begin{proof}
    Let $f \in L^2(\T^d)$, then, we have
    \[ f(y) = \sum_{\xi \in \Z^d} \widehat{f}(\xi) e^{2 i \pi \xi \cdot y},\]
and it follows that:
    \begin{align}
        \dfrac{1}{N^d} \sum_{y \in \T^d_N} f(y) &= \dfrac{1}{N^d} \sum_{y \in \T^d_N}  \sum_{\xi \in \Z^d} \widehat{f}(\xi) e^{2 i \pi \xi \cdot y} \\   
        &= \sum_{\xi \in \Z^d} \widehat{f}(\xi) \left[ \dfrac{1}{N^d} \sum_{k \in \ens{0, \dots, N-1}^d} e^{\frac{2 i \pi}{N} \, \xi \cdot k} \right] \\
        \oversetref{Lem.}{\ref{appdx-lem:sum-over-grid}}&{=} \sum_{\xi \in \Z^d} \widehat{f}(\xi) \mathbbm{1}_{\ens{\xi \equiv 0\ [N]}} \\
        &= \sum_{m \in \Z^d} \widehat{f}(m N).
    \end{align}
Moreover, notice that
    \[ \int_{\T^d} f(y) \odif{y} = \int_{\T^d} \left( \sum_{\xi \in \Z^d} \widehat{f}(\xi) e^{2i \pi \xi \cdot y} \right) \odif{y} = \sum_{\xi \in \Z^d} \widehat{f}(\xi) \underbrace{\left( \int_{\T^d} e^{2i\pi \xi \cdot y} \odif{y} \right)}_{= \begin{cases}
        0, & \text{if $\xi \neq 0$;} \\
        1, & \text{if $\xi = 0$;}
    \end{cases}} = \widehat{f}(0), \]
after switching the sum of the integral. This achieves the proof of the lemma.
\end{proof}

\begin{remark}
    Notably, the above lemma shows that the evaluation of a function $f \in L^2(\T^d)$ on the grid can be rewritten as
        \[ \dfrac{1}{N^d} \sum_{y \in \T^d_N} f(y) - \int_{\T^d} f(y) \odif{y} = \sum_{m \in \Z^d \setminus \{0\}} \widehat{f}(m N). \label{eq:formula-difference} \]
\end{remark}

\begin{lemma}[Continuous Parseval's Identity on $\T^d$]
    Let $f \in L^2(\T^d)$ with Fourier coefficients
    \[ \widehat{f}(\xi) = \int_{\T^d} f(x) e^{-2 \pi i \xi \cdot x} \odif{x}, 
        \quad \xi \in \Z^d. \]
    Then Parseval's identity holds:
    \[  \norm{f}_{L^2(\T^d)}^2 = \sum_{\xi \in \Z^d} \abs{\widehat{f}(\xi)}^2.  \]
\end{lemma}

\begin{lemma}[Discrete Parseval's Identity on $\T^d$]\label{lem:parseval-identity-on-grid}
    For any function $f \colon \T^d \to \C^m$, we have
        \[ \normgridvec{f}^2 := \sum_{x \in \T^d_N} \abs{f(x)}^2 = N^d \sum_{\xi \in \ens{0, 1, \ldots, N - 1}^d} \abs{\sum_{m \in \Z^d} \widehat{f}(\xi + m N)}^2. \] 
\end{lemma}

\begin{proof}
    We have
    \begin{align}
        \sum_{x \in \T^d_N} \abs{f(x)}^2 &= \sum_{x \in \T^d_N} \ps{f(y)}{\overline{f(y)}} \\
        &= \sum_{x \in \T^d_N} \ps{\sum_{\xi \in \Z^d} \widehat{f}(\xi) e^{2 i \pi \xi \cdot x}}{\overline{\sum_{\eta \in \Z^d} \widehat{f}(\eta) e^{2 i \pi \eta \cdot x}}} \\
        &= \sum_{x \in \T^d_N} \sum_{\xi \in \Z^d} \sum_{\eta \in \Z^d} \widehat{f}(\xi) e^{2 i \pi \xi \cdot x} \overline{\widehat{f}(\eta)} e^{-2 i \pi \eta \cdot x} \\
        &= \sum_{\xi \in \Z^d} \sum_{\eta \in \Z^d} \widehat{f}(\xi) \overline{\widehat{f}(\eta)} \sum_{x \in \T^d_N} e^{2 i \pi (\xi - \eta) \cdot x} \\
        \oversetref{Lem.}{\ref{appdx-lem:sum-over-grid}}&{=} N^d \sum_{\xi \in \Z^d} \sum_{\eta \in \Z^d} \widehat{f}(\xi) \overline{\widehat{f}(\eta)} \mathbbm{1}_{\xi \equiv \eta \, [N]} \\
        &= N^d \sum_{r \in \ens{0, 1, \ldots, N - 1}^d} \sum_{\xi, \eta \in (r + N\Z^d)} \widehat{f}(\xi) \overline{\widehat{f}(\eta)} \\
        &= N^d \sum_{r \in \ens{0, 1, \ldots, N - 1}^d} \abs{\sum_{m \in \Z^d} \widehat{f}(r + m N)}^2,
    \end{align}
    as claimed.
\end{proof}

\newpage
\section{Function Spaces and Embeddings}\label{appdx-sec:function-spaces}

This appendix collects notation, definitions, and standard results on function spaces used throughout the paper.

\subsection{Notation and Preliminaries}

\subsubsection{Multi-Index Notation}\label{appdx-subsec:multi-index}

Following standard notations, we let $\alpha = (\alpha_1, \dots, \alpha_d) \in \mathbb{N}_0^d$ be a \emph{multi-index}, and we define its \emph{length} as $\smash{\abs{\alpha} := \sum\limits_{i=1}^d \alpha_i}$. Moreover, given $x = (x_1, \ldots, x_d) \in \R^d$, we define the monomial of degree $\abs{\alpha}$ as
    \[ x^{\alpha} := \prod_{i = 1}^d x_i^{\alpha_i}. \]

For a sufficiently regular function $f \colon \R^d \to \K$ or $f \colon \T^d \to \K$ with $\K$ being $\R$ or $\C$, the partial derivative of order $\abs{\alpha}$ corresponding to the multi-index $\alpha$ is denoted by
\[ \partial^\alpha f := \frac{\partial^{|\alpha|} f}{\partial x_1^{\alpha_1} \cdots \partial x_d^{\alpha_d}} = \partial_{x_1}^{\alpha_1} \cdots \partial_{x_d}^{\alpha_d} f = \partial_1^{\alpha_1} \cdots \partial_d^{\alpha_d} f. \]
and, when convenient, we write $D^\alpha f$ for $\partial^\alpha f$.

This notation is used in the definition of the functions spaces manipulated in this work.

%We also adopt the standard shorthand:
%\[ \int_{\mathbb{R}^d} f(x)\,dx := \int_{\mathbb{R}} \cdots \int_{\mathbb{R}} f(x_1, \dots, x_d)\,dx_1 \dots dx_d, \]

\subsection{Function Spaces}

Since we work with functions defined either on the $d$-dimensional torus $\T^d$ (for instance, the input functions of the neural operator) or on $\R^d$ (such as activation functions, including \textsc{ReLU}, \textsc{GELU} and others), we formulate the function spaces below in both settings, for clarity and self-containedness. For simplicity, we assume all functions to be \emph{real-valued}.

Before introducing the various function spaces, we recall the notion of the support of a function, which will be used repeatedly throughout this subsection.
\begin{definition}[Support of a Function]\label{appdx-def:support-function}
    Let $E$ denote either $\R^d$ or $\T^d$, the support of a function $f \colon E \to \R$ is defined as
        \[ \supp(f) := \adh{\enstq{x \in E}{f(x) \neq 0}}, \]
    which is the \emph{closure} (in $E$, with respect to its canonical topology) of points $x \in E$ where $f$ does not vanish.
\end{definition}

\subsubsection{Smooth Compactly Supported Functions, Schwartz Space and Tempered Distributions}

\paragraph{Case of $\R^d$.}

\begin{definition}[Smooth Functions on $\R^d$]\label{appdx-def:smooth-functions}
    We define the space $\mathscr{C}^{\infty}(\R^d)$ of \emph{real-valued} and \emph{smooth} functions on $\R^d$ as
        \[ \mathscr{C}^{\infty}(\R^d) := \enstq{f \colon \R^d \to \R}{ D^\alpha f \text{ exists and is \emph{continuous} on } \R^d \text{ for every multi-index }\alpha \in \N^d_0}. \]
\end{definition}
Equivalently, $f \in \mathscr{C}^{\infty}(\R^d)$ if all partial derivatives of every order exist (pointwise) and are continuous on $\R^d$. We endow $\mathscr{C}^{\infty}(\R^d)$ with the Fréchet topology generated by the family of seminorms
        \[ \norm{f}_{\alpha, n} := \sup_{\norm{x} \le n} \abs{D^\alpha f(x)}, \]
where $f \in \mathscr{C}^{\infty}(\R^d)$, $\alpha \in \N_0^d$, and $n \in \N$.

\begin{definition}[Smooth Compactly Supported Functions on $\R^d$]\label{appdx-def:smooth-compactly-supported-functions}
    The space $\mathscr{C}^{\infty}_c(\R^d)$ of real-valued, \emph{smooth} and \emph{compactly supported} functions is
        \[ \mathscr{C}^{\infty}_c(\R^d) := \enstq{f \in \mathscr{C}^{\infty}(\R^d)}{\supp(f)\text{ is compact in }\R^d}. \]
\end{definition}
For notational convenience, when needed, we write $\mathscr{C}^{\infty}_K(\R^d)$ the space of smooth, real-valued functions $f$ whose support lies in the compact subset $K \subset \R^d$. It follows that
    \[ \mathscr{C}^{\infty}_c(\R^d) = \bigcup_{K \subset \R^d \text{ compact}} \mathscr{C}^{\infty}_K(\R^d). \]

\begin{definition}[{Schwartz Space $\mathcal{S}(\R^d)$~\citep[Definition~1.18]{Bahouri2011}}]\label{appdx-def:schwartz-space}
    The Schwartz space $\mathcal{S}(\R^d)$ is the set of smooth functions $f \in \mathscr{C}^{\infty}(\R^d)$ such that for any $k \in \N_0$ we have
        \[ \norm{f}_{k, \mathcal{S}} := \sup_{\substack{\alpha \in \N_0^d \\ \abs{\alpha} \le k}}\, \sup_{x \in \R^d} \left( 1 + \norm{x}_2 \right)^k \abs{D^{\alpha} f(x)} < +\infty. \]
\end{definition}

\begin{definition}[{Tempered Distributions on $\R^d$~\citep[Definition~1.20]{Bahouri2011}}]
    The set $\mathcal{S}'(\R^d)$ of tempered distributions consists of all continuous linear functional on $\mathcal{S}(\R^d)$. More precisely, $u \in \mathcal{S}'(\R^d)$ if, and only if there exists a constant $C > 0$, and an integer $k \in \N_0$ such that
        \[ \abs{\ps{u}{\phi}} \le C \norm{\phi}_{k, \mathcal{S}}, \]
    for all $\phi \in \mathcal{S}(\R^d)$.
\end{definition}
Following~\cite{Bahouri2011}, a sequence $(u_n)_{n \in \N}$ of tempered distributions is said to converge to $u \in \mathcal{S}'(\R^d)$ if, and only if
    \[ \ps{u_n}{\phi} \xrightarrow[n \to +\infty]{} \ps{u}{\phi}, \]
for all $\phi \in \mathcal{S}(\R^d)$.

\paragraph{Case of $\T^d$.} In the case of the $d$-dimensional torus, which is compact, any real-valued smooth function $f \in \mathscr{C}^{\infty}(\T^d)$ is automatically compactly supported, i.e.,~\Cref{appdx-def:smooth-functions,appdx-def:smooth-compactly-supported-functions} are equivalent. Consequently, the Schwartz space $\mathcal{S}(\T^d)$ (which in the Euclidean setting lies strictly between $\mathscr{C}^{\infty}(\R^d)$ and $\mathscr{C}^{\infty}_c(\R^d)$) coincides in this case with both $\mathscr{C}^{\infty}_c(\T^d)$ and $\mathscr{C}^{\infty}(\T^d)$. That is,
\[ \mathscr{C}^{\infty}_c(\T^d) = \mathcal{S}(\T^d) = \mathscr{C}^{\infty}(\T^d). \]
As a result, in the case of $\T^d$, the space $\mathcal{S}'(\T^d)$ of tempered distributions coincides with the space of distributions
\[ \left( \mathscr{C}^{\infty}_c(\T^d) \right)', \]
namely the topological dual of $\mathscr{C}^{\infty}_c(\T^d)$.

\subsubsection{Lebesgue Spaces}
Throughout, we let $E$ denote either $\R^d$ or $\T^d$. We equip $E$ with its canonical Borel $\sigma$-algebra and with the Lebesgue measure $\odif{x}$ (in the case $E = \R^d$) or the normalized Haar (Lebesgue) measure on $\T^d$ (so that $\abs{\T^d} = 1$). All functions are identified up to equality almost everywhere (a.e.).

\begin{definition}[Measurable Functions]
    A function $f \colon E \to \R$ is said to be (Lebesgue) measurable if $f^{-1}(U)$ is measurable for every open set $U \subset \R$.
\end{definition}

\begin{definition}[Lebesgue Spaces $L^p(E)$, \citep{folland1999real}]\label{appdx-def:lebesgue}
    Let $0 < p < \infty$. The Lebesgue space $L^p(E)$ is the set of (equivalence classes of a.e. equal) measurable functions $f \colon E \to \R$ such that
        \[ \norm{f}_{L^p(E)} := \left( \int_E |f(x)|^p \odif{x} \right)^{\nicefrac{1}{p}} < +\infty. \]
    
    For $p=\infty$, we define
        \[ L^\infty(E) := \enstq{f \colon E \to \R \text{ measurable}}{\norm{f}_{L^\infty(E)} := \sup_{x\in E} |f(x)| < \infty}. \]
\end{definition}

It is well-known that for $p \ge 1$, the normed space $(L^p(E), \norm{\cdot}_{L^p(E)})$ is a Banach space.% For $0 < p < 1$, the functional $\norm{\cdot}_{L^p(E)}$ is a quasi-norm (triangle inequality holds up to a constant), and $L^p(E)$ is a quasi-Banach space; we will nevertheless use the same notation.

\paragraph{Periodic viewpoint.} As said in~\Cref{appdx-sec:fourier-analysis-torus}, a function $f \colon \T^d \to \R$ can be identified with a $\Z^d$--periodic measurable function on $\R^d$, and we define 
    \[ \int_{\T^d} f(x) \odif{x} = \int_{\intfo{0}{1}^d} f(x) \odif{x}. \]

\subsubsection{Sobolev Spaces}\label{appdx-subsec:sobolev}

Let $E$ be either $\R^d$ or $\T^d$. We define the following spaces, which we used throughout this paper.
\begin{definition}[Integer--order Sobolev Spaces]\label{appdx-def:sobolev-space-integer}
    For $m \in \N_0$ and $1 \le p \le \infty$ the Sobolev space $W^{m, p}(E)$ is the space of (equivalence classes of a.e.\ equal) measurable functions $f \colon E \to \R$ whose \emph{weak derivatives} $D^\alpha f$ exist and belong to $L^p(E)$ for all $\abs{\alpha} \le m$, where $\alpha \in \N_0^d$.
\end{definition}
Notably, the space $W^{m, p}(E)$ is endowed with the norm
    \[ \norm{f}_{W^{m,p}(E)} := \left( \sum_{|\alpha|\le m}\norm{D^\alpha f}_{L^p(E)}^p \right)^{\nicefrac{1}{p}} < +\infty,
    \]
for any real number $ 1 \le p < +\infty$, and
    \[ \norm{f}_{W^{m, \infty}(E)} := \sum_{|\alpha|\le m}\norm{D^\alpha f}_{L^\infty(E)} < +\infty, \]
when $p = +\infty$.

In particular, we write $H^m(E) := W^{m,2}(E)$. For $1 \le p \le \infty$ and $m \in \N_0$, the normed space $(W^{m, p}(E), \norm{\cdot}_{W^{m, p}})$ is a Banach space, and $H^m(E)$ is a Hilbert space.

Following~\cite{Triebel1983chap2,SchmeisserTriebel1987,RunstSickel+1996}, we define the Sobolev space of fractional index as follows:
\begin{definition}[Fractional--order Sobolev Spaces]\label{appdx-def:sobolev-space-fractional}
Let $s \in \R_+^*$, the (inhomogeneous) fractional Sobolev space $H^s(\R^d)$ is defined by
    \[ H^s(\R^d) := \enstq{ f \in \mathcal{S}'(\R^d)}{\int_{\R^d} ( 1 + \norm{\xi}^2 )^s |\widehat f(\xi)|^2 \odif\xi < +\infty }, \]
where $\widehat f$ denotes the Fourier transform of $f$ in the sense of tempered distributions.
\end{definition}

The space $H^s(\R^d)$ is endowed with the norm
    \[ \norm{f}_{H^s(\R^d)} := \left( \int_{\R^d} (1+\norm{\xi}^2)^s | \widehat f(\xi) |^2 \odif\xi \right)^{\nicefrac{1}{2}} < +\infty \]

For the case of the $d$--dimensional torus, the fractional Sobolev space $H^s(\T^d)$ is defined as
\[ H^s(\T^d) := \enstq{ f \in \mathcal{D}(\T^d)}{  \sum_{k\in\Z^d} (1+\norm{k}^2)^s |\widehat f(k)|^2 < \infty }, \]
where $\mathcal{D}(\T^d)$ denotes the space of distribution over $\T^d$. The space $H^s(\T^d)$ is endowed with the norm
\[ \norm{f}_{H^s(\T^d)} := \left( \sum_{k\in\Z^d} (1+ \norm{k}^2)^s |\widehat f(k)|^2 \right)^{\nicefrac{1}{2}} < +\infty. \]

For integer $s = m\in\N_0$, the space $H^m(E)$ coincides with the classical Sobolev space
$W^{m,2}(E)$ from~\Cref{appdx-def:sobolev-space-integer}, and the norms are equivalent. This holds both on $\R^d$ and on the torus $\T^d$.

\subsubsection{H\"older Spaces}\label{appdx-subsec:holder}

Let $E$ denote either the Euclidean space $\mathbb{R}^d$ or the torus $\mathbb{T}^d$.

\begin{definition}[H\"older spaces $\mathscr{C}^{k, \alpha}$]
    Let $k \in \mathbb{N}_0$ and $\alpha \in [0, 1]$. The H\"older space $\mathscr{C}^{k,\alpha}(E)$ is defined as the set of all functions $f \colon E \to \mathbb{R}$ whose partial derivatives $D^\beta f$ are continuous and bounded for all multi-indices $|\beta| \le k$, and whose $k$-th order derivatives are H\"older continuous with exponent $\alpha$.
    
    For any $g \colon E \to \mathbb{R}$, we define the H\"older semi-norm as
    \[ [g]_{\mathscr{C}^{0,\alpha}(E)} := \sup_{x \neq y} \frac{|g(x) - g(y)|}{\norm{x - y}^\alpha}. \]
    
    The space $\mathscr{C}^{k,\alpha}(E)$ is a Banach space equipped with the norm
    \[ \|f\|_{\mathscr{C}^{k,\alpha}(E)} := \sum_{|\beta| \le k} \sup_{x \in E} \norm{D^\beta f(x)} + \sum_{|\beta| = k} [D^\beta f]_{\mathscr{C}^{0,\alpha}(E)}. \]
    
    In the case $\alpha = 0$, we write $\mathscr{C}^k(E) = \mathscr{C}^{k,0}(E)$ to denote the space of $k$-times continuously differentiable functions with bounded derivatives.
\end{definition}

\subsection{Embeddings and Inclusion Results}

\iffalse
We have the following correspondences
\begin{lemma}[Some Equality for Functions Spaces{\cite{Triebel1983chap2,SchmeisserTriebel1987,RunstSickel+1996,doi:10.1142/S0219530519500234,Taylor2023chap1}}]\label{appdx-lem:function-spaces-equality}
    We have the following results:
    \begin{itemize}
        \item for any $s \in \R$, $H^s(\T^d)$ can be identified with $B^s_{2, 2}(\T^d)$,

        \item for any real number $s > 0$, if $s \notin \N_0$ then $\mathscr{C}^s(\R^d)$ can be identified with $B^s_{\infty, \infty}(\R^d)$.
    \end{itemize}
\end{lemma}

\begin{lemma}[{Some Embeddings Results~\citep{Triebel1983chap2,SchmeisserTriebel1987,,Taylor2023chap1}}]\label{appdx-lem:function-spaces-embeddings}
    We have the following embeddings:
    \begin{itemize}
        %\item for any integer $k \in \N_0$, $\mathscr{C}^k(\R^d)$ can be embedded into $B^k_{\infty, \infty}(\R^d)$,

        %\item for any real number $s > t$, $B^s_{2, \infty}(\T^d)$ can be embedded into $B^t_{2, 2}(\T^d)$, and also $H^t(\T^d)$ by~\Cref{appdx-lem:function-spaces-equality},

        \item for any real number $s > \frac{d}{2}$, we have the embeddings $H^s(\T^d) \hookrightarrow \mathscr{C}^0(\T^d)$ and $H^s(\T^d) \hookrightarrow L^{\infty}(\T^d)$,

        %\item for any real numbers $s > 0$, and $0 < p, q_0, q_1 \le +\infty$, we have $B^s_{\infty, \infty}(\R^d) \hookrightarrow B^0_{\infty, 1}(\R^d) \hookrightarrow \mathscr{C}^0(\R^d)$.
    \end{itemize}
\end{lemma}
\fi

\begin{lemma}[{Some Embeddings Results~\citep{Triebel1983chap2,SchmeisserTriebel1987,RunstSickel+1996,doi:10.1142/S0219530519500234,Taylor2023chap1}}]\label{appdx-lem:function-spaces-embeddings}
    Let $E$ denote either $\mathbb{R}^d$ or the torus $\mathbb{T}^d$. The following embeddings hold:

    \begin{enumerate}
        \item \textbf{Sobolev--H\"older Embedding:} Let $s \in \mathbb{R}$, $k \in \mathbb{N}_0$, and $\alpha \in (0,1)$. If 
            \[ s > k + \frac{d}{2}, \]
        there exists a continuous embedding $H^s(\mathbb{T}^d) \hookrightarrow \mathscr{C}^{k, \alpha}(\mathbb{T}^d)$. In particular as a consequence, for $s > \frac{d}{2}$, we have:
            \[ H^s(\mathbb{T}^d) \hookrightarrow \mathscr{C}^0(\mathbb{T}^d) \hookrightarrow L^\infty(\mathbb{T}^d). \]

        \item \textbf{Rellich--Kondrachov Theorem:} For any $s' > s$ and with either $E = \R^d$ or $E = \T^d$, the embedding 
            \[ H^{s'}(E) \hookrightarrow H^s(E) \]
        is continuous. Furthermore, if the domain is compact (e.g., $E = \mathbb{T}^d$), the embedding is compact.

        %\item \textbf{H\"older Inclusions:} For $k \in \mathbb{N}_0$ and either $E = \R^d$ or $E = \T^d$, if $0 \le \beta < \alpha \le 1$, there is a continuous embedding
        %    \[ \mathscr{C}^{k, \alpha}(E) \hookrightarrow \mathscr{C}^{k, \beta}(E). \]
    \end{enumerate}
\end{lemma}

\subsection{Some Composition Results in Sobolev Spaces}

%\adri{WARNING!!!! PLEASE SOMEONE CAN DEFINE $T_f$, REASON: I am too lazy for now. PLEASE THIS IS VERY IMPORTANT TO PREVENT ANY CONFUSION FOR THE KIND REVIEWERS. PLEASE CAN SOMEONE DO IT BEFORE SUBMITTING THE PRESENT MANUSCRIPT. PLEASE.}

\textbf{Composition operator.}
Let $E \in \{\R^d,\T^d\}$ and let $f : \R \to \R$ be Borel measurable. We define
    \[ T_f(h):= f \circ h \qquad \text{(a.e. on $E$)}.\]
We say that $T_f$ \emph{acts on} a space $X(E)$ if
    \[ h\in X(E)\ \Longrightarrow\ T_f(h)\in X(E). \]

\begin{theorem}[{\citet[Theorem~1]{Bourdaud2014}}]\label{appdx-thm:bourdaud-sobolev-s=1}
    Let $d \ge 1$ be an integer and $1 \le p \le +\infty$ be a real number. Assume $f \colon \R \to \R$ is a Borel measurable function such that $f(0) = 0$. Then, the following assertions are equivalent:
    \begin{enumerate}
        \item the operator $T_f$ acts on $W_p^1(\R^d)$,
        \item in the sense of distributions, $f' \in L^{\infty}_{\textnormal{loc}}(\R)$ if $p > n$ or $p = 1 = n$ or $f' \in L^{\infty}(\R)$ otherwise
    \end{enumerate}
\end{theorem}

Notably, in the case $p = 2$, and the function $f \colon \R \to \R$ is globally Lipschitz, condition $2.$ above is satisfied and and for any function $h \in W^1_2(\R^d) = H^1(\R^d)$ we have $f \circ h \in H^1(\R^d)$. The same conclusion holds for the $d$-dimensional torus $\T^d$.

We now focus on the low smoothness regime.
\begin{theorem}[Special Case of~{\citet[Theorem~6]{Bourdaud2014}}]\label{appdx-thm:bourdaud-sobolev-0<s<1}
    Let $d \ge 1$ be an integer and $0 < s < 1$ be a real number. Assume $f \colon \R \to \R$ is a Borel measurable function such that $f(0) = 0$. Then, the following assertions are equivalent:
    \begin{enumerate}
        \item the operator $T_f$ acts on $H^s(\R^d)$,
        \item in the sense of distributions, either $f' \in L^{\infty}_{\textnormal{loc}}(\R)$ if $H^s(\R^d) \hookrightarrow L^{\infty}(\R^d)$ or $f' \in L^{\infty}(\R)$ if $H^s(\R^d) \not\subseteq L^{\infty}(\R^d)$
    \end{enumerate}
\end{theorem}

Again, assuming the function $f \colon \R \to \R$ to be globally Lipschitz, we deduce that for any function $h \in H^s(\R^d)$ we have $f \circ h \in H^s(\R^d)$. The same conclusion holds for the $d$-dimensional torus $\T^d$.

We now deal with the case where the Sobolev regularity $s \in \intoo{1}{\frac{3}{2}}$.
\begin{theorem}[{Special Case of~\citet[Theorem~1]{Bourdaud1992}}]\label{appdx-thm:bourdaud-sobolev-1<s<3/2}
    Let $d \ge 1$ be an integer and $1 < s < \frac{3}{2}$ be a real number. Assume $f \colon \R \to \R$ is a Borel measurable function such that $f(0) = 0$ and (in the sense of distribution) $f''$ is a bounded measure on $\R$. Then, the operator $T_f$ acts on $H^s(\R^d)$.
\end{theorem}

\newpage

\section{Omitted Proofs of the Stability}\label{apx:stability-proofs}

\begin{restate-lemma}{\ref{lem:stability-layer-ssm}}[Stability of a Single SS-NO layer]
    Assume that $g \colon z \mapsto \normop{\cK^{\textnormal{SS-NO}} (z)} \in L^1(\T^d, \R)$, then for any input functions $v, w \in L^2(\T^d, \R^H)$, we have
        \begin{align}
             \norm{\mathcal{L}v - \mathcal{L}w}_{L^2} \leq C_{\sigma} \norm{v - w}_{L^2}, 
        \end{align}
    where 
     \[ C_{\sigma} := L_\sigma \left( \normop{W} + \int_{\T^d} \normop{\cK^{\textnormal{SS-NO}}(z)} \odif{z} \right), \]
    with $\normop{\cdot}$ the operator norm of a matrix, and $\mathcal{L}$ is the map defined in~\eqref{def:ssno-layer} (we omit the layer index $t$)
        \[ \mathcal{L} \colon v \mapsto \sigma\left( W v + \cK^{\textnormal{SS-NO}} * v + b \right). \]
\end{restate-lemma}

\begin{proof}[Proof of \Cref{lem:stability-layer-ssm}]
Let $v, w \colon \T^d \to \R^H$ in $L^2(\T^d)$. By $L_\sigma$-Lipschitz continuity of $\sigma$, we have:
    \begin{align}
        \forall x \in \T^d, \quad \abs{(\mathcal{L}v) (x) - (\mathcal{L}w) (x)} \leq L_\sigma \abs{W (v(x) - w(x)) + (\cK * (v - w)) (x)},
    \end{align}
hence by taking the square and integrating over $\T^d$:
    \begin{align}
        \norm{\mathcal{L}v - \mathcal{L}w}_{L^2} &\leq L_\sigma \norm{W (v - w) + \cK * (v - w)}_{L^2} \\
        & \leq L_\sigma (\norm{W( v - w)}_{L^2} + \norm{\cK * (v - w)}_{L^2}).
    \end{align} 
The first term of the sum is bounded by $\normop{W} \norm{v - w}_{L^2}$. For the second term, we have for all $x \in \T^d$:

    \begin{align}
        \norm{\cK * (v - w)(x)} &= \norm{\int_{\T^d} \cK ( x - y) (v - w) (y) \odif{y}} \\
        & \leq \int_{\T^d} \norm{\cK ( x - y) (v - w) (y) } \odif{y} \\
        & \leq \int_{\T^d} \normop{\cK ( x - y)} \norm{(v - w) (y)} \odif{y} \\
        & = (\normop{\cK(\cdot)} * \norm{(v - w)(\cdot)})(x),
    \end{align}
    where the convolution product between two functions $f$ and $g$ is defined by $(f * g)(x) = {\displaystyle \int f(x-y) g(y) \odif{y}}$.
    Hence, 
    \begin{align}
        \norm{\cK * (v-w)}_{L^2}^2 
        &= \int_{\T^d} \norm{\cK * (v - w)(x)}^2 \odif{x} \\
        & \leq \int_{\T^d}  (\normop{\cK(\cdot)} * \norm{(v - w)(\cdot)})(x)^2 \odif{x} \\
        &= \norm{\normop{\cK(\cdot)} * \norm{(v - w)(\cdot)}}_{L^2}^2,
    \end{align}
    and using \emph{Young's convolution inequality} (see~\Cref{appdx-lem:young-convolution-inequality}), it follows that,
    \begin{align}
        \norm{\normop{\cK(\cdot)} * \norm{(v - w)(\cdot)}}_{L^2} &\leq \norm{\normop{\cK(\cdot)}}_{L^1} \norm{\norm{(v - w)(\cdot)}}_{L^2} \\
        &= \left(\int_{\T^d} \normop{\cK(z)} \odif{z} \right) \norm{v - w}_{L^2}.
    \end{align}
    This implies that
    \[ \norm{\cK * (v-w)}_{L^2} \leq \left(\int_{\T^d} \normop{\cK(z)} \odif{z} \right) \norm{v - w}_{L^2}.\]

    Combining the previous inequalities yields
        \[ \norm{\mathcal{L}v - \mathcal{L}w}_{L^2} \leq L_\sigma \left( \normop{W} + \int_{\T^d} \normop{\cK(z)} \odif{z} \right) \norm{v - w}_{L^2}. \]
\end{proof}

%Before proving \Cref{lem:finite-bound-L1-op-norm}, we start by stating and proving the following elementary linear algebra result.

\begin{restate-lemma}{\ref{lem:finite-bound-L1-op-norm}}
    For SS-NO kernels defined in \Cref{def:ss-no-kernels}, 
    the assumption of \Cref{lem:stability-layer-ssm} on $\cK^{\textnormal{SS-NO}}$ holds, and we have
        \begin{align*}
            &\int_{\T^d} \normop{\cK^{\textnormal{SS-NO}}(z)} \odif{z} \leq \sum_{i = 1}^d \sum_{k=1}^K \left[ \dfrac{1 - e^{-\rho_{k, i}}}{\abs{\rho_{k, i}}} \norm{C_+^{(k)}}_2 \norm{B^{(k)}_+}_2 \right],
        \end{align*}
    where we omit the layer index $t$ for simplicity.
    \iffalse
    more precisely, if $D = [-a, b] \subset \R$ with $a, b \geq 0$ then
        \begin{align}
            \int_{\T^d} \norm{\cK(z)}_{\mathrm{op}} \odif{z} &\leq \sum_{k=1}^K \left[ \dfrac{2 - e^{r_k b} - e^{r_k a}}{\abs{r_k}} \norm{C_k}_2 \norm{B_k}_2 \right] \\
            &\le \begin{cases} 
            \displaystyle\diam(D) \sum_{k = 1}^K \norm{C_k}_2 \norm{B_k}_2  < \infty, \\[7pt]
            \displaystyle 2 \sum_{k = 1}^K \dfrac{\norm{C_k}_2 \norm{B_k}_2}{\abs{r_k}}
        \end{cases}  \numberthis\label{30ba552b-cb1a-44c5-b9d2-732306caae3d} 
        \end{align} 
    where $\diam(D) := b + a$ is the diameter of the space domain.
    \fi
\end{restate-lemma}

\begin{proof}[Proof of \Cref{lem:finite-bound-L1-op-norm}]
    
    For any $z \in \T^d$, we have by the triangle inequality,
        \begin{align}
            \normop{\cK^\textnormal{SS-NO}(z)} &\leq \normop{\sum_{i = 1}^d \left( \cK^{(i)}_{t, +}(z_i) + \cK^{(i)}_{t, -}(z_i) \right)} \\
            & \le \sum_{i = 1}^d \normop{\cK^{(i)}_{t, +}(z_i)} + \normop{\cK^{(i)}_{t, -}(z_i)} \\
            & = \sum_{i = 1}^d \sum_{\eps \in \ens{+, -}} \normop{\mathbbm{1}_{\ens{\eps z_i \geq 0}} \sum_{k = 1}^K c_{k, i} e^{-\rho_{k, i} \abs{z_i}} e^{i \omega_{k, i} z_i} C_{\eps}^{(k)} \left( B_{\eps}^{(k)} \right)^{\top}} \\
            & \leq \sum_{i = 1}^d \sum_{\eps \in \ens{+, -}} \mathbbm{1}_{\ens{\eps z_i \geq 0}} \sum_{k = 1}^K c_{k, i} e^{-\rho_{k, i} \abs{z_i}} \normop{ C_{\eps}^{(k)} \left( B_{\eps}^{(k)} \right)^{\top}} \\
            & = \sum_{i = 1}^d \sum_{\eps \in \ens{+, -}} \mathbbm{1}_{\ens{\eps z_i \geq 0}} \sum_{k = 1}^K c_{k, i} e^{-\rho_{k, i} \abs{z_i}} \norm{ C_{\eps}^{(k)}}_2 \norm{B_{\eps}^{(k)} }_2. \numberthis\label{c3488880-3a7a-4d63-b951-3dca4658bad0}
        \end{align}
    where in~\eqref{c3488880-3a7a-4d63-b951-3dca4658bad0} we use~\Cref{lem:op-norm-matrix-rank-1} which gives $\normop{C_{\eps}^{(k)} \left( B_{\eps}^{(k)} \right)^{\top}} = \norm{C_{\eps}^{(k)}}_2 \norm{B_{\eps}^{(k)}}_2$. Integrating~\eqref{c3488880-3a7a-4d63-b951-3dca4658bad0} over $\T^d$ gives:
        \begin{align}
            &\int_{\T^d} \normop{\cK^\textnormal{SS-NO}(z)} \odif{z} \\
            &\qquad\leq \sum_{i=1}^d \sum_{k=1}^K c_{k, i}  \left( \norm{ C_{+}^{(k)}}_2 \norm{B_{+}^{(k)} }_2 \int_{\T^d} \mathbbm{1}_{\ens{z_i \geq 0}} e^{-\rho_{k, i} \abs{z_i}} \odif{z} \right. \\
            &\qquad\qquad\qquad\qquad\qquad\qquad\qquad\qquad \left.+  \norm{ C_{-}^{(k)}}_2 \norm{B_{-}^{(k)} }_2 \int_{\T^d} \mathbbm{1}_{\ens{z_i \leq 0}} e^{-\rho_{k, i} \abs{z_i}} \odif{z} \right),
            %\left( \int_{\T^d} \mathbbm{1}_{\{ z \geq 0 \}} e^{r_k \abs{z}} \odif{z} + \int_{\T^d} \mathbbm{1}_{\{ z \leq 0 \}} e^{r_k \abs{z}} \odif{z} \right), \\
        \end{align}
    and, using the fact that $\T^d$ is the unit torus, we have
        \[ \int_{\T^d} \mathbbm{1}_{\ens{\pm z_i \geq 0}} e^{-\rho_{k, i} \abs{z_i}} \odif{z} = \abs{\mathrm{Vol}(\T^{d - 1})} \int_{\T} \mathbbm{1}_{\ens{\pm z_i \geq 0}} e^{-\rho_{k, i} \abs{z_i}} \odif{z_i} = \int_{\T} \mathbbm{1}_{\ens{\pm z_i \geq 0}} e^{-\rho_{k, i} \abs{z_i}} \odif{z_i}, \]
    then
        \[ \int_{\T} \mathbbm{1}_{\ens{z_i \geq 0}} e^{-\rho_{k, i} \abs{z_i}} \odif{z_i} = \int_{0}^1 e^{-\rho_{k, i} z_i} \odif{z_i} = \dfrac{1 - e^{-\rho_{k, i}}}{\abs{\rho_{k, i}}}\]
    and    
        \[ \int_{\T} \mathbbm{1}_{\ens{z_i \leq 0}} e^{-\rho_{k, i} \abs{z_i}} \odif{z_i} = 0.\]
    Hence,
        \begin{align*}
            \int_{\T^d} \normop{\cK^{\textnormal{SS-NO}}(z)} \odif{z} \leq \sum_{i = 1}^d \sum_{k=1}^K \left[ \dfrac{1 - e^{-\rho_{k, i}}}{\abs{\rho_{k, i}}} \norm{C_+^{(k)}}_2 \norm{B^{(k)}_+}_2 \right],
        \end{align*}
        %\[ \int_{\T^d} \norm{\cK(z)}_{\mathrm{op}} \odif{z} \leq \sum_{k=1}^K \left[ \dfrac{2 - e^{r_k b} - e^{r_k a}}{\abs{r_k}} \norm{C_k}_2 \norm{B_k}_2 \right]. \numberthis\label{8e90ea11-fd02-43aa-86d4-bd9b87596237} \]

    %Moreover, using the well-known inequality $e^x \ge 1 + x$ (valid for all $x \in \R$) we obtain
       % \[ 2 - e^{r_k a} - e^{r_k b} \le 2 - (1 + r_k a) - (1 + r_k b) = - r_k (a + b) = \abs{r_k} \diam(D), \]
    %and plugging this in~\eqref{8e90ea11-fd02-43aa-86d4-bd9b87596237} leads to
     %   \[ \int_{\T^d} \norm{\cK(z)}_{\mathrm{op}} \odif{z} \leq \sum_{k=1}^K \left[ \dfrac{2 - e^{r_k b} - e^{r_k a}}{\abs{r_k}} \norm{C_k}_2 \norm{B_k}_2 \right] \le \diam(D) \sum_{k = 1}^K \norm{C_k}_2 \norm{B_k}_2 < +\infty, \]
    %as claimed.
\end{proof}

\begin{restate-theorem}{\ref{thm:global_stability_layers}}[Global Stability of Discretized SS-NO]
Let $\mathcal{L}_N^{(T)} = \mathcal{L}_{N, T-1} \circ \dots \circ \mathcal{L}_{N, 0}$ be the operator representing the composition of $T$ discretized SS-NO layers. For any latent states $v, w \in L^2(\mathbb{T}_N^d, \mathbb{R}^{H})$, we have:
\begin{equation}
    \|\mathcal{L}_N^{(T)}(v) - \mathcal{L}_N^{(T)}(w)\|_{\ell^2(\mathbb{T}_N^d)} \le \mathbf{C}_{N, T} \|v - w\|_{\ell^2(\mathbb{T}_N^d)},
\end{equation}
where the global discrete Lipschitz constant $\mathbf{C}_{N, T}$ is given by:
\begin{equation}
    \mathbf{C}_{N, T} := \prod_{t=0}^{T-1} L_{\sigma_t} \left( \|W_t\| + \|\mathcal{K}_{N,t}\|_{\ell^2(\mathbb{T}_N^d)} \right).
\end{equation}
\end{restate-theorem}

\begin{proof}[Proof of~\Cref{thm:global_stability_layers}]
To establish the global stability of the deep architecture, we first characterize the Lipschitz property of a single layer. Let $v, w \colon \mathbb{T}^d \to \mathbb{R}^H$ in $L^2(\mathbb{T}^d)$, we have by the $L_{\sigma}$-Lipschitz continuity of $\sigma$:
\begin{align*}
    \normgridvec{\mathcal{L}_N v - \mathcal{L}_N w} &\leq L_{\sigma} \normgridvec{W (v - w) + \frac{1}{N^d} \sum_{y \in \mathbb{T}^d_N} \mathcal{K}(\cdot - y) (v - w)(y)} \\
    & \leq L_\sigma \left( \normop{W} \normgridvec{v - w} + \normgridvec{\frac{1}{N^d} \sum_{y \in \mathbb{T}^d_N} \mathcal{K}(\cdot - y) (v - w)(y)} \right) \\
    \oversetrel{rel:91753ab9-8a09-4813-9674-cc06b67b6b7f}&{\le} L_\sigma \left( \normop{W} \normgridvec{v - w} + \normgridmatone{\mathcal{K}} \normgridvec{v - w} \right) \\
    & \le C_{N, \sigma} \normgridvec{v - w},
\end{align*}
with $C_{N, \sigma} := L_\sigma \left( \normop{W} + \normgridmatone{\mathcal{K}} \right)$. Inequality~\relref{rel:91753ab9-8a09-4813-9674-cc06b67b6b7f} follows from the \emph{discrete Young's convolution inequality} (see~\Cref{appdx-lem:young-convolution-inequality}) on the finite $d$-dimensional torus $\mathbb{T}^d_N$.

Now, we extend this result to the full sequence of $T$ layers by induction. Let $\mathcal{L}_N^{(T)} = \mathcal{L}_{N, T-1} \circ \dots \circ \mathcal{L}_{N, 0}$. For any $t \in \{0, \dots, T-1\}$, each layer $\mathcal{L}_{N,t}$ satisfy:
\begin{equation*}
    \|\mathcal{L}_{N,t}(v) - \mathcal{L}_{N,t}(w)\|_{\ell^2(\mathbb{T}_N^d)} \le C_{N, \mathcal{L}}^{(t)} \|v - w\|_{\ell^2(\mathbb{T}_N^d)},
\end{equation*}
where $C_{N, \mathcal{L}}^{(t)} = L_{\sigma_t} ( \|W_t\| + \|\mathcal{K}_{N,t}\|_{\ell^2(\mathbb{T}_N^d)} )$. By the property of composition of Lipschitz maps, the Lipschitz constant of the sequence is the product of the individual layer constants:
\begin{equation*}
    \|\mathcal{L}_N^{(T)}(v) - \mathcal{L}_N^{(T)}(w)\|_{\ell^2(\mathbb{T}_N^d)} \le \left( \prod_{t=0}^{T-1} C_{N, \mathcal{L}}^{(t)} \right) \|v - w\|_{\ell^2(\mathbb{T}_N^d)},
\end{equation*}
yielding the global constant $\mathbf{C}_{N, T}$.
\end{proof}

\newpage
\section{Bounds on the FNO and SS-NO Kernels}\label{apx:bounds-fno-ss-no}

\subsection{Fourier Transform of FNO and SS-NO Kernels}

\begin{restate-lemma}{\ref{lem:fourier-fno}}[Fourier Transform of FNO Kernel]
    For any $\xi \in \Z^d$, we have:
    \begin{equation}
        \widehat{\cK}_t^{\textnormal{FNO}}(\xi) = \sum_{k \in \ens{-K, \ldots, K}^d} P_t^{(k)} \delta_k(\xi) = P_t^{(\xi)} \mathbbm{1}_{\ens{\xi = k}},
    \end{equation}
    where we let $P_t^{(\xi)} = 0$ for any $\xi \in \Z^d \setminus \ens{-K, \ldots, K}^d$.
\end{restate-lemma}

\begin{proof}[Proof of~\Cref{lem:fourier-fno}]
    By linearity of the Fourier transform, we have, for any $\xi \in \Z^d$
        \[ \widehat{\cK}_t^{\textnormal{FNO}}(\xi) = \sum_{k \in \ens{-K, \ldots, K}^d} \widehat{\cK}_t^{(k)}(\xi), \]
    and, by definition of $\cK_t^{(k)}$ in~\eqref{def:kernel-fnos}, for any $k \in \ens{-K, \ldots, K}^d$ we have
        \[ \widehat{\cK}_t^{(k)}(\xi) = P_t^{(k)} \int_{\T^d} e^{2 i \pi k \cdot z} e^{-2 i \pi \xi \cdot z} \odif{z}.= P_t^{(k)} \int_{\T^d} e^{2 i \pi (k - \xi) \cdot z} \odif{z} = P_t^{(k)} \mathbbm{1}_{\ens{\xi = k}}. \]
    Therefore,
        \begin{equation}
            \widehat{\cK}_t^{\textnormal{FNO}}(\xi) = \sum_{k \in \ens{-K, \ldots, K}^d} P_t^{(k)} \mathbbm{1}_{\ens{\xi = k}} = P_t^{(\xi)} \mathbbm{1}_{\ens{\xi = k}},
        \end{equation}
    as claimed.
\end{proof}

\begin{restate-lemma}{\ref{lem:fourier-ss-no}}[Fourier Transform of SS-NO Kernel]
    For any $\xi = (\xi_1, \dots, \xi_d) \in \Z^d$, we have:
    \begin{align*}
        &\widehat{\cK}_t^{\textnormal{SS-NO}}(\xi) \\
        &\qquad= \sum_{i=1}^d \sum_{k=1}^K c_{t, k, i} \left[ F_{+, k, i}(\xi_i) A_{k, +} + F_{-, k, i}(\xi_i) A_{k, -} \right],
    \end{align*}
    where $A_{k, \eps} := C_{t, \eps}^{(k)} \left( B_{t, \eps}^{(k)} \right)^\top$ for $\eps \in \ens{+, -}$, and
%\johan{\\typo ? $C_{t, \eps}^{(k)} \left( B_{t, \eps}^{(k)} \right)$}    
    \begin{align}
         F_{+, k, i}(\xi_i) &:= \frac{1 - e^{-(\rho_{t, k, i} - i(\omega_{t, k, i} - 2 \pi \xi_i) ) / 2}}{\rho_{t, k, i} - i(\omega_{t, k, i} - 2 \pi \xi_i)}  \\
         F_{-, k, i}(\xi_i) &:= \frac{1 - e^{-(\rho_{t, k, i} + i(\omega_{t, k, i} - 2 \pi \xi_i) ) / 2}}{\rho_{t, k, i} + i(\omega_{t, k, i} - 2 \pi \xi_i)}.
         %&\int_{-\frac{1}{2}}^0 e^{- \rho_{t, k} \abs{z}} e^{i z (\omega_{t, k} - 2 n \pi)} \odif{z} = \int_{-\frac{1}{2}}^0 e^{z (\rho_{t, k} + i (\omega_{t, k} - 2 n \pi))} \odif{z} = \frac{e^{z (\rho_{t, k} + i (\omega_{t, k} - 2 n \pi))}}{\rho_{t, k} + i (\omega_{t, k} - 2 n \pi)} \Bigg]_{z = -\frac{1}{2}}^{z = 0},
    \end{align}
\end{restate-lemma}

\begin{proof}[Proof of~\Cref{lem:fourier-ss-no}]
    Let's compute the exact Fourier transform of the SS-NOs kernel (defined in \eqref{eq:kernel-ssnos}. 

%    \abde{I will give the clean Fourier coefficient in 1D then I will give the coefficient in arbitrary dimension.}
    
    In $1$D, the kernel is given by:
        \[ \cK_t(z) = \sum_{k = 1}^K c_{t, k} \left[ \mathbbm{1}_{z \geq 0} e^{- \rho_{t, k} \abs{z}} e^{i \omega_{t, k} z} C_{t, +}^{(k)} (B_{t, +}^{(k)})^\top + \mathbbm{1}_{z \leq 0} e^{- \rho_{t, k} \abs{z}} e^{i \omega_{t, k} z} C_{t, -}^{(k)} (B_{t, -}^{(k)})^\top \right],  \]
    periodized on $\T^1 = [-\frac{1}{2}, \frac{1}{2})$. For integer Fourier mode $n \in \Z$ the coefficient is:
    \begin{equation}
        \widehat{\cK_t}(n) = \int_{- \frac{1}{2}}^{\frac{1}{2}} \cK_t(z) e^{-2 \pi i n z}  \odif{z} = \sum_{k=1}^K c_{t, k} \left[ F_{+, k}(n) C_{t, +}^{(k)} (B_{t, +}^{(k)})^\top + F_{-, k}(n) C_{t, -}^{(k)} (B_{t, -}^{(k)})^\top, \right]
    \end{equation}
    where the scalar factors are:
    \begin{align}
         F_{+, k}(n) &:= \int_{0}^{\frac{1}{2}} e^{- (\rho_{t, k} - i \omega_{t ,k }) z} e^{-2i \pi n z}  \odif{z} = \frac{1 - e^{-(\rho_{t, k} - i(\omega_{t, k} - 2 \pi n) ) / 2}}{\rho_{t, k} - i(\omega_{t, k} - 2 \pi n)}  \\
         F_{-, k}(n) &:= \int_{- \frac{1}{2}}^{0} e^{ (\rho_{t, k} + i \omega_{t ,k }) z} e^{-2i \pi n z}  \odif{z} = \frac{1 - e^{-(\rho_{t, k} + i(\omega_{t, k} - 2 \pi n) ) / 2}}{\rho_{t, k} + i(\omega_{t, k} - 2 \pi n)} \\
         %&\int_{-\frac{1}{2}}^0 e^{- \rho_{t, k} \abs{z}} e^{i z (\omega_{t, k} - 2 n \pi)} \odif{z} = \int_{-\frac{1}{2}}^0 e^{z (\rho_{t, k} + i (\omega_{t, k} - 2 n \pi))} \odif{z} = \frac{e^{z (\rho_{t, k} + i (\omega_{t, k} - 2 n \pi))}}{\rho_{t, k} + i (\omega_{t, k} - 2 n \pi)} \Bigg]_{z = -\frac{1}{2}}^{z = 0},
    \end{align}
    Notice that, when $\abs{n} \to \infty$, we have:
        \[ F_{+, k}(n) = \dfrac{1}{- 2 i \pi n} (1 + O(\abs{n}^{-1}))  \]
    and:
        \[  F_{-, k}(n) = \dfrac{1}{- 2 i \pi n} (1 + O(\abs{n}^{-1})) \]
    so $\abs{\widehat{\cK_t}(n)} \sim \Theta \left(\dfrac{1}{n} \right)$ and hence $\abs{\widehat{\cK_t}}$ is not $L^1$ (unless we truncate).

    Now for the $d$-dimensional version, with a multi-index $\xi = (\xi_1, \dots \xi_d) \in \Z^d$, 
    \begin{equation}
        \widehat{\cK_t}(\xi) = \sum_{i=1}^d \sum_{k=1}^K c_{t, k, i} \left[ F_{+, k, i}(\xi_i) C_{t, +}^{(k)} (B_{t, +}^{(k)})^\top + F_{-, k, i}(\xi_i) C_{t, -}^{(k)} (B_{t, -}^{(k)})^\top \right],
    \end{equation}
    where 
    \begin{align}
         F_{+, k, i}(\xi_i) &:= \frac{1 - e^{-(\rho_{t, k, i} - i(\omega_{t, k, i} - 2 \pi \xi_i) ) / 2}}{\rho_{t, k, i} - i(\omega_{t, k, i} - 2 \pi \xi_i)}  \\
         F_{-, k, i}(\xi_i) &:= \frac{1 - e^{-(\rho_{t, k, i} + i(\omega_{t, k, i} - 2 \pi \xi_i) ) / 2}}{\rho_{t, k, i} + i(\omega_{t, k, i} - 2 \pi \xi_i)}.
         %&\int_{-\frac{1}{2}}^0 e^{- \rho_{t, k} \abs{z}} e^{i z (\omega_{t, k} - 2 n \pi)} \odif{z} = \int_{-\frac{1}{2}}^0 e^{z (\rho_{t, k} + i (\omega_{t, k} - 2 n \pi))} \odif{z} = \frac{e^{z (\rho_{t, k} + i (\omega_{t, k} - 2 n \pi))}}{\rho_{t, k} + i (\omega_{t, k} - 2 n \pi)} \Bigg]_{z = -\frac{1}{2}}^{z = 0},
    \end{align}
\end{proof}

\subsection{Bounding the FNO and SS-NO Kernels}

\begin{restate-lemma}{\ref{lem:bound-fno}}[Bound on the $L^2$-norm of $\cK_t$ for FNO]
     For the FNO kernel $\cK_t^{\textnormal{FNO}}$ in~\eqref{def:kernel-fnos}, there exists a finite constant $C_d \ge 0$ such that
        \[ \normgridmattwo{\cK_t^{\textnormal{FNO}}} \le C_d (N K)^{\frac{d}{2}} \sup_{k \in \ens{1, 2, \ldots, K}^d} \normop{P_t^{(k)}}. \]
\end{restate-lemma}

\begin{proof}[Proof of~\Cref{lem:bound-fno}]
    The \emph{Parseval's identity on the grid} (\Cref{lem:parseval-identity-on-grid}) yields,
        \[ \normgridmattwo{\cK_t^{\textnormal{FNO}}}^2 = N^d \sum_{\xi \in \ens{0, 1, \ldots, N - 1}^d} \normop{\sum_{m \in \Z^d} \widehat{\cK}_t^{\textnormal{FNO}}(\xi + m N)}^2. \]
    Using \Cref{lem:fourier-fno} and the fact that $-N < -\frac{N}{2} \le - K \le K \le \frac{N}{2} < N$ we have
        \begin{align}
            \normgridmattwo{\cK_t^{\textnormal{FNO}}}^2 &= N^d \sum_{\xi \in \ens{-K, \ldots, K}^d} \normop{\widehat{\cK}_t^{\textnormal{FNO}}(\xi)}^2 \\
            &= N^d \sum_{\xi \in \ens{-K, \ldots, K}^d} \normop{P_t^{(\xi)}}^2  \\
            &\leq N^d (2 K + 1)^d \sup_{\xi \in \ens{-K, \ldots, K}^d} \normop{P_t^{(\xi)}}^2,
        \end{align}
    hence the result, with $C_d = 3^{\frac{d}{2}}$ since $K \ge 1$.
\end{proof}

\begin{restate-lemma}{\ref{lem:bound-ssno}}[Bound on the $L^2$-norm of $\cK_t$ for SS-NO]
    For the SS-NO kernel $\cK_t^{\textnormal{SS-NO}}$ in~\eqref{eq:kernel-ssnos} we have
     %   \[ \normgridmattwo{\cK_t^{\textnormal{SS-NO}}} \leq K d N^{\frac{d}{2}} \sup_{i \in [d], k \in [K]} \left\{ \abs{c_{t, k, i}} \alpha_{t, k, i}  \left(  \norm{C_{t, +}^{(k)}}^2  \norm{B_{t, +}^{(k)}}^2 + \norm{C_{t, -}^{(k)}}^2  \norm{B_{t, -}^{(k)}}^2 \right)^{\nicefrac{1}{2}} \right\} .  \]
    %Alternatively,
        \begin{align*}
            \normgridmattwo{\cK_t^{\textnormal{SS-NO}}} \le C_d Kd N^{\frac{d}{2}},
        \end{align*}
    where
        \[ C_d := \sup_{i \in [d], \, k \in [K]} \ens{\abs{c_{t, k, i}} \left(  A_{k, +} + A_{k, -} \right)}, \]
    and for $\eps \in \ens{+, -}$, $A_{k, \eps} := \normop{C_{t, \eps}^{(k)}} \, \normop{B_{t, \eps}^{(k)}}$.
\end{restate-lemma}

\begin{proof}[Proof of~\Cref{lem:bound-ssno}]
    We have:
        \begin{align*}
            \normgridmattwo{\cK_t^{\textnormal{SS-NO}}}^2  &= \sum_{x \in \T^d_N} \normop{\cK_t^{\textnormal{SS-NO}}(x)}^2  \\
            &\leq \sum_{x \in \T^d_N} K d \sum_{i \in [d]} \sum_{k \in [K]} \abs{c_{t, k, i}}^2 \left(  \normop{C_{t, +}^{(k)}} \, \normop{B_{t, +}^{(k)}} + \normop{C_{t, -}^{(k)}} \, \normop{B_{t, -}^{(k)}} \right)^2 \\
            &\leq (K d)^2 N^d \sup_{i \in [d], k \in [K]}\abs{c_{t, k, i}}^2 \left(  \normop{C_{t, +}^{(k)}} \, \normop{B_{t, +}^{(k)}} + \normop{C_{t, -}^{(k)}} \, \normop{B_{t, -}^{(k)}} \right)^2,
        \end{align*} 
    as desired

\end{proof}

\newpage
\section{Proof of the Discretizations Error}\label{apx:proof-discretization}
\subsection{Proof Strategy: Error Decomposition}
We start by decomposing the error at each SS-NO layer, let $\cE_t^{(0)} \colon \T^d_N \to \R^{d_t}$ define by
    \[ \cE_t^{(0)}(x) := v_t^N(x) - v_t(x), \]
for all $x \in \T^d_N$. We let $\cE_t^{(1)} \colon \T^d_N \to \R^{d_{t + 1}}$ be the function defined as
    \begin{align}
        \cE_t^{(1)}(x) := \frac{1}{N^d} &\sum_{y \in \T^d_N} \cK_t(x - y) v_t(y) - \int_{\T^d} \cK_t(x - y) v_t(y) \odif{y}, \numberthis\label{b06ec7af-e2f9-46dc-ac2e-244f90be7c6d-1}
%\adri{can we bound this differently?}
    \end{align} 
for all $x \in \T^d_N$. Additionally, we consider $\cE_t^{(2)} \colon \T^d_N \to \R^{d_{t + 1}}$ such that
    \[ \cE_t^{(2)}(x) := \frac{1}{N^d} \sum_{y \in \T^d_N} \cK_t(x - y) \cE_t^{(0)}(y), \numberthis\label{b06ec7af-e2f9-46dc-ac2e-244f90be7c6d-2} \]
for all $x \in \T^d_N$. The term $\cE_t^{(2)}$ corresponds to the error in the inputs at the entry of layer $t$, after the discrete Fourier transform is applied. We have
    \begin{align}
        \cE_t^{(1)}(x) &+ \cE_t^{(2)}(x) = \frac{1}{N^d} \sum_{y \in \T^d_N} \cK_t(x - y) v_t^N(y)  - \int_{\T^d} \cK_t(x - y) v_t(y) \odif{y}, 
    \end{align} 
which is the discretization error related to the convolution. Also, note that
    \begin{align*}
        \cE_{t + 1}^{(0)}(x) :\!&= v_{t + 1}^N(x) - v_{t + 1}(x) \\
        &= \sigma_t\left( W_t v_t(x) + (\cK_t * v_t)(x) +  b_t + W_t \cE_t^{(0)}(x) + \cE_t^{(1)}(x) + \cE_t^{(2)}(x) \right) \\
        &\qquad\qquad- \sigma_t\left( W_t v_t(x) + (\cK_t * v_t)(x) + b_t \right). \numberthis\label{b06ec7af-e2f9-46dc-ac2e-244f90be7c6d-3}
    \end{align*}

In the following parts, we bound each component $\cE_t^{(1)}$, $\cE_t^{(2)}$ and $\cE_{t + 1}^{(0)}$ defined in~\eqref{b06ec7af-e2f9-46dc-ac2e-244f90be7c6d-1},~\eqref{b06ec7af-e2f9-46dc-ac2e-244f90be7c6d-2} and~\eqref{b06ec7af-e2f9-46dc-ac2e-244f90be7c6d-3}.

\subsection{Bound on $\cE_t^{(1)}$}

This next lemma provides a first general upper bound on the error $\cE_t^{(1)}$ (in the $t^{\textnormal{th}}$ layer), and then we will specify for every special case.

\begin{lemma}[A Bound with Sobolev Norm]\label{appdx-lem:bound-sobolev}
    Let $s > \frac{d}{2}$ be a real number, and take $v \in H^s(\T^d)$ then, for any integer $N > 0$ and any $r \in \Z^d$, if $m^*(r) \in \Z^d$ denote the unique vector such that $r + m N \in \ens{-\Floor{\frac{N}{2}}, \ldots, \Floor{\frac{N}{2}}}$, we have
        \[ \sum_{m \in \Z^d \setminus \ens{m^*(r)}} \norm{\widehat{v}(r + m N)} \le C_{d, s} N^{-s} \normsobolev{v}, \]
    where $C_{d, s}$ is a universal constant depending only on the Sobolev exponent $s$ and the dimension $d$.
\end{lemma}

\begin{proof}[Proof of~\Cref{appdx-lem:bound-sobolev}]
    By the Cauchy-Schwarz inequality, for any $r \in \Z^d$ and any integer $N > 0$ we have
        \begin{align*}
            \sum_{m \in \Z^d \setminus \ens{m^*(r)}} & \norm{\widehat{v}(r + m N)} \\
            &= \sum_{m \in \Z^d \setminus \ens{m^*(r)}} \left( 1 + \norm{r + m N} \right)^{-s} \left( 1 + \norm{r + m N} \right)^s \norm{\widehat{v}(r + m N)} \\
            &\le \left( \sum_{m \in \Z^d \setminus \ens{m^*(r)}} \left( 1 + \norm{r + m N} \right)^{-2 s} \right)^{\frac{1}{2}} \left( \sum_{m \in \Z^d \setminus \ens{m^*(r)}} \left( 1 + \norm{r + m N} \right)^{2 s} \norm{\widehat{v}(r + m N)}^2 \right)^{\frac{1}{2}} \\
            \oversetrel{rel:08e1b7ca-cffa-4638-b389-b5d022a0e8bc}&{\le} \left( \sum_{m \in \Z^d \setminus \ens{m^*(r)}} \left( 1 + \norm{r + m N} \right)^{-2 s} \right)^{\frac{1}{2}} \left( 2 \sum_{m \in \Z^d \setminus \ens{m^*(r)}} \left( 1 + \norm{r + m N}^2 \right)^s \norm{\widehat{v}(r + m N)}^2 \right)^{\frac{1}{2}} \numberthis\label{08e1b7ca-cffa-4638-b389-b5d022a0e8bc} \\
            \oversetrel{rel:08e1b7ca-cffa-4638-b389-b5d022a0e8bc-2}&{\le} \sqrt{2} \normsobolev{v} \left( \sum_{m \in \Z^d \setminus \ens{m^*(r)}} \left( 1 + \norm{r + m N} \right)^{-2 s} \right)^{\frac{1}{2}}, \numberthis\label{08e1b7ca-cffa-4638-b389-b5d022a0e8bc-2}
        \end{align*}
    where in~\relref{rel:08e1b7ca-cffa-4638-b389-b5d022a0e8bc} we use the well-known inequality $(a + b)^2 \le 2 (a^2 + b^2)$, for any $a, b \in \R$ (and the fact that $s > 0$). In~\relref{rel:08e1b7ca-cffa-4638-b389-b5d022a0e8bc-2} we bound second sum in~\eqref{08e1b7ca-cffa-4638-b389-b5d022a0e8bc}, which run over points on the grid $r + N\Z^d$, by the sum over all $\Z^d$, resulting in the Sobolev norm of $v$. We now need to bound the remaining sum from~\eqref{08e1b7ca-cffa-4638-b389-b5d022a0e8bc-2}. First, observe by the choice of $m^*(r)$ that for any $m \in \Z^d \setminus \ens{m^*(r)}$ we have
        \[ \norm{r + m N} \ge N, \,\, \text{ i.e., } \,\, \norm{\frac{r}{N} + m} \ge 1, \]
    then, we have
        \begin{align*}
            \sum_{m \in \Z^d \setminus \ens{m^*(r)}} \left( 1 + \norm{r + m N} \right)^{-2 s} \oversetref{Lem.}{\ref{lem:norm-power-alpha-inequality-2}}&{\le} 2 \sum_{m \in \Z^d \setminus \ens{m^*(r)}} \frac{1}{1 + \norm{r + m N}^{2s}} \\
            &= 2 N^{-2 s} \sum_{m \in \Z^d \setminus \ens{m^*(r)}} \frac{1}{N^{-2 s} + \norm{\frac{r}{N} + m}^{2 s}} \\
            &\le 2 N^{-2 s} \sum_{m \in \Z^d \setminus \ens{m^*(r)}} \norm{\frac{r}{N} + m}^{-2 s} \\
            \oversetrel{rel:b77ac201-0357-41e1-9d87-e4ac25d1b496}&{=} 2 N^{-2 s} \sum_{m \in \Z^d \setminus \ens{0}} \norm{\frac{r_0}{N} + m}^{-2 s} \\
            &\le 2 N^{-2 s} \sum_{m \in \Z^d \setminus \ens{0}} \norm{\frac{r_0}{N} + m}_{\infty}^{-2 s}, \numberthis\label{56156669-6d41-4424-8a1e-2cdc4a42b48e}
        \end{align*}
    where in~\relref{rel:b77ac201-0357-41e1-9d87-e4ac25d1b496} we define $r_0$ as the unique vector of $\Z^d$ such that $r_0 \in \ens{-\Floor{\frac{N}{2}}, \ldots, \Floor{\frac{N}{2}}}$ and $r_0 \equiv r\, [N]$ and, therefore, $m^*(r) = r_0$ and $m^*(r_0) = 0$. Then, using the triangle inequality, for any $m \in \Z^d \setminus \ens{0}$ we have $\norm{m}_{\infty} \ge 1$ and 
        \[ 2 \norm{\frac{r_0}{N} + m}_{\infty} \ge 2 \left( \norm{m}_{\infty} - \norm{\frac{r_0}{N}}_{\infty} \right) \oversetrel{rel:1af52881-18f6-43c3-8138-99a355930f55}{\ge} 2 \left( \norm{m}_{\infty} - \frac{1}{2} \right) \ge \norm{m}_{\infty}, \numberthis\label{56156669-6d41-4424-8a1e-2cdc4a42b48e-2} \]
    where in~\relref{rel:1af52881-18f6-43c3-8138-99a355930f55} we use the fact that $\frac{r_0}{N} \in \intff{-\frac{1}{2}}{\frac{1}{2}}$. Injecting~\eqref{56156669-6d41-4424-8a1e-2cdc4a42b48e-2} in~\eqref{56156669-6d41-4424-8a1e-2cdc4a42b48e} we obtain
    \begin{align*}
        \sum_{m \in \Z^d \setminus \ens{m^*(r)}} \left( 1 + \norm{r + m N} \right)^{-2 s} \oversetrel{rel:25fa446c-9f95-4c45-9de5-71dc12d8d14e}&{\le} 2^{2s + 1} N^{-2s} \sum_{m \in \Z^d \setminus \ens{0}} \frac{1}{\left( 2\norm{\frac{r_0}{N} + m}_{\infty} \right)^{2 s}} \\
        \oversetlab{\eqref{56156669-6d41-4424-8a1e-2cdc4a42b48e-2}}&{\le} 2 \left( \dfrac{N}{2} \right)^{-2s} \sum_{m \in \Z^d \setminus \ens{0}} \norm{m}_\infty^{-2s} \\
        \oversetrel{rel:25fa446c-9f95-4c45-9de5-71dc12d8d14e-2}&{\le} 2 \left( \dfrac{N}{2 \sqrt{d}} \right)^{-2s} \sum_{m \in \Z^d \setminus \ens{0}} \norm{m}^{-2s} \\
        &= \underbrace{2 \cdot (4 d)^s \sum_{m \in \Z^d \setminus \ens{0}} \norm{m}^{-2s}}_{:= c_{d, s} < \infty} \, N^{-2s} \\
        &\le c_{d, s} N^{- 2 s}, \numberthis\label{baa88d6f-23d5-4964-ae73-57895f7bd1ee}
    \end{align*}
    where in~\relref{rel:25fa446c-9f95-4c45-9de5-71dc12d8d14e} and~\relref{rel:25fa446c-9f95-4c45-9de5-71dc12d8d14e-2} we use the fact that $\norm{\cdot}_{\infty} \le \norm{\cdot}_2 \le \sqrt{d} \norm{\cdot}_{\infty}$ (in $\R^d$). The bound~\eqref{baa88d6f-23d5-4964-ae73-57895f7bd1ee} follows from $s > \frac{d}{2}$. This gives
        \[ \sum_{m \in \Z^d \setminus \ens{m^*(r)}} \norm{\widehat{v}(r + m N)} \le C_{d, s} N^{-s} \normsobolev{v}, \]
    as desired, with constant $C_{d, s} := \sqrt{2 c_{d, s}} < +\infty$.
\end{proof}

\begin{lemma}[A First General Result on $\cE_t^{(1)}$]\label{lem:general-bound-error-E1}
    We have 
    \begin{align}
        \normgridvec{\cE_t^{(1)}}^2  = N^d \sum_{r \in \ens{0, \dots, N-1}^d} \left\| \sum_{m \in \Z^d} \widehat{\cK_t}(-(r + m N)) \sum_{q \in \Z^d \setminus \{ -m \}} \widehat{v_t}(q N -r )\right\|^2. 
    \end{align} 
\end{lemma}

\begin{proof}[Proof of~\Cref{lem:general-bound-error-E1}] Let us define for every $x \in \T^d$, the periodic function
    \[ f_x(y) := \cK(x - y) v_t(y), \quad y \in \T^d,\]
such that the quantity $\cE_t^{(1)}(x)$ (with $x \in \T^d_N$) becomes
    \[ \cE_t^{(1)}(x) = \dfrac{1}{N^d} \sum_{y \in \T^d_N} f_x(y) - \int_{\T^d} f_x(y) \odif{y}, \]
and, using~\Cref{eq:formula-difference}, it follows
    \[ \cE_t^{(1)}(x) = \sum_{m \in \Z^d \setminus \{0 \}} \widehat{f_x}(m N). \numberthis\label{8587e697-4130-486c-b3ec-9a251578ca90} \]

Now, if we denote $g_x(y) := \cK(x -y)$ we have
    \begin{align}
        g_x(y) = \sum_{\xi \in \Z^d} \widehat{\cK}_t(\xi) e^{2i \pi \xi \cdot (x - y)},
    \end{align}
and so
    \begin{align}
        \widehat{g_x}(\xi) &= \int_{\T^d} g_x(y) e^{-2 i \pi \xi \cdot y} \odif{y} \\
        &= \int_{\T^d} \left( \sum_{\eta \in \Z^d} \widehat{\cK}_t(\xi) e^{2i \pi \eta \cdot (x - y)} \right) e^{-2 i \pi \xi \cdot y} \odif{y} \\
        &= \sum_{\eta \in \Z^d} \widehat{\cK}_t(\eta) e^{2i \pi \eta \cdot x} \int_{\T^d} e^{-2i \pi (\eta + \xi) \cdot y} \odif{y} \\
        &= \widehat{\cK}_t(-\xi) e^{-2i\pi \xi \cdot x},
    \end{align}
then, using the fact that the Fourier transform of a product of two functions is the convolution of their individual Fourier transforms (see~\Cref{appdx-lem:fourier-transform-product-functions}), it follows:
    \begin{align}
        \widehat{f_x}(\xi) &= \sum_{\eta \in \Z^d} \widehat{g_x}(\eta) \widehat{v_t}(\xi - \eta) \\
        &= \sum_{\eta \in \Z^d} e^{-2i\pi \eta \cdot x} \widehat{\cK_t}(-\eta) \widehat{v_t}(\xi - \eta), \numberthis\label{aa6b26ab-9327-4c67-85cc-2108ab11672a}
    \end{align}
hence
    \begin{align}
        \cE_t^{(1)}(x) \oversetlab{\eqref{8587e697-4130-486c-b3ec-9a251578ca90}}&{=} \sum_{m \in \Z^d \setminus \{0 \}} \widehat{f_x}(m N) \\
        \oversetlab{\eqref{aa6b26ab-9327-4c67-85cc-2108ab11672a}}&{=} \sum_{m \in \Z^d \setminus \{0 \}} \sum_{\eta \in \Z^d} e^{-2i\pi \eta \cdot x} \widehat{\cK_t}(-\eta) \widehat{v_t}(m N - \eta) \\
        &= \sum_{\eta \in \Z^d} e^{-2i\pi \eta \cdot x} \underbrace{\widehat{\cK_t}(-\eta)  \sum_{m \in \Z^d \setminus \{0 \}} \widehat{v_t}(m N - \eta)}_{:= c_\eta} \numberthis\label{cc30072b-bb12-4b42-8cd2-b9c5bc021c59-2} \\
        &= \sum_{\eta \in \Z^d} e^{-2i\pi \eta \cdot x} c_{\eta}, \numberthis\label{aa6b26ab-9327-4c67-85cc-2108ab11672a-2}
    \end{align}
where $c_{\eta}$ is a vector in $\C^d$. Therefore, the $\ell^2$--norm of $\cE_t^{(1)}$ on the grid $\T^d_N$ can be expressed as
    \begin{align}
        \normgridvec{\cE_t^{(1)}}^2 &= \sum_{x \in \T^d_N} \norm{\cE_t^{(1)}(x)}^2 \\
        \oversetlab{\eqref{aa6b26ab-9327-4c67-85cc-2108ab11672a-2}}&{=} \sum_{x \in \T^d_N} \norm{\sum_{\eta \in \Z^d} e^{-2i \pi \eta \cdot x} c_\eta}^2, \numberthis\label{aa6b26ab-9327-4c67-85cc-2108ab11672a-3}
    \end{align}
and, decomposing each $\eta \in \Z^d$ uniquely as $\eta = r + m N$ with $r \in \ens{0, \dots, N-1}^d$ and $m \in \Z^d$, the inner sum in~\eqref{aa6b26ab-9327-4c67-85cc-2108ab11672a-3} is equal to
    \begin{align}
        \sum_{\eta \in \Z^d} e^{-2i \pi \eta \cdot x} c_\eta &= \sum_{r \in \ens{0, \dots, N-1}^d} \sum_{m \in \Z^d} e^{-2i \pi (r + m N) \cdot x} c_{r + m N} \\
        &= \sum_{r \in \ens{0, \dots, N-1}^d} e^{-2i \pi r \cdot x} \underbrace{\left( \sum_{m \in \Z^d} c_{r + m N} \right)}_{:= a_r} \numberthis\label{cc30072b-bb12-4b42-8cd2-b9c5bc021c59} \\
        &= \sum_{r \in \ens{0, \dots, N-1}^d} e^{-2i \pi r \cdot x} a_r. \numberthis\label{dad4f517-f55b-4294-b1c2-25c151386a27}
    \end{align}
Now, using the fact that for $v \in \C^d$ we have $\norm{v}^2 = \sum\limits_{i = 1}^n \abs{v_i}^2 = \ps{v}{\overline{v}}$, then
    \begin{align}
         \normgridvec{\cE_t^{(1)}}^2 \oversetlab{\eqref{dad4f517-f55b-4294-b1c2-25c151386a27}}&{=} \sum_{x \in \T^d_N} \ps{\sum_{r \in \ens{0, \dots, N-1}^d} e^{-2i \pi r \cdot x} a_r }{ \overline{\sum_{s \in \ens{0, \dots, N-1}^d} e^{-2i \pi s \cdot x} a_s }} \\
         &= \sum_{r, s \in \ens{0, \dots, N-1}^d} a_r \overline{a_s} \sum_{x \in \T^d_N} e^{- 2 \pi i (r - s) \cdot x} \\
         \oversetref{Lem.}{\ref{appdx-lem:sum-over-grid}}&{=} N^d \sum_{r, s \in \ens{0, \dots, N-1}^d} a_r \overline{a_s}\, \mathbbm{1}_{\ens{r \equiv s \, [N]}} \\
         &= N^d \sum_{r \in \ens{0, \dots, N-1}^d} \norm{a_r}^2 \\
         \oversetlab{\eqref{cc30072b-bb12-4b42-8cd2-b9c5bc021c59}}&{=} N^d \sum_{r \in \ens{0, \dots, N-1}^d} \norm{\sum_{m} c_{r + m N}}^2 \\
         \oversetlab{\eqref{cc30072b-bb12-4b42-8cd2-b9c5bc021c59-2}}&{=} N^d \sum_{r \in \ens{0, \dots, N-1}^d} \norm{\sum_{m \in \Z^d} \widehat{\cK_t}(-(r + m N))  \sum_{m' \in \Z^d \setminus \{0 \}} \widehat{v_t}(m' N - (r + m N) )  }^2 \\
         &= N^d \sum_{r \in \ens{0, \dots, N-1}^d} \norm{\sum_{m \in \Z^d} \widehat{\cK_t}(-(r + m N)) \sum_{q \in \Z^d \setminus \{ -m \}} \widehat{v_t}(q N -r )}^2,
    \end{align}
which is the claimed result.
\end{proof}

\begin{lemma}[Polynomial Decay of Fourier Coefficients $\widehat{\cK}_t(\xi)$]\label{lem:bounding-E1-polynomial-decay}
    Let $d \ge 1$ be a integer and $s > \frac{d}{2}$ a real number. Assume there exists some real number $\alpha > d$ and a finite real constant $C_{d, \alpha} \ge 0$ such that
        \[ \normop{\widehat{\cK}_t(\xi)} \le C_{d, \alpha} (1 + \norm{\xi})^{-\alpha}, \numberthis\label{97835486-5f3e-4046-85f8-fc76dff72f94} \]
    for all $\xi \in \Z^d$. Then, there exists a finite real constant $C_{d, s, \alpha} \ge 0$ such that the inequality
        \[ \normgridvec{\cE_t^{(1)}} \le C_{d, s, \alpha} N^{\max\ens{\frac{d}{2} - s, d - \alpha}} \normsobolev{v_t}, \]
    holds.
\end{lemma}

\begin{proof}[Proof of~\Cref{lem:bounding-E1-polynomial-decay}]
    By~\Cref{lem:general-bound-error-E1}, we have
        \[ \normgridvec{\cE_t^{(1)}}^2 \leq N^d \sum_{r \in \ens{0, \dots, N-1}^d} \norm{\sum_{m \in \Z^d} \widehat{\cK_t}(-(r + m N)) \sum_{q \in \Z^d \setminus \{ -m \}} \widehat{v_t}(q N -r )}^2. \numberthis\label{4f774a5f-e1b3-4c08-a0a4-c68197924cd1-2} \]
    Now, for any $r \in \ens{0, 1, \ldots, N - 1}^d$, let $m^*(r) \in \Z^d$ be the unique vector such that $r + m^*(r) N \in \ens{-\Floor{\frac{N}{2}}, \ldots, \Floor{\frac{N}{2}}}$. 
    Then, the inner expression in~\eqref{4f774a5f-e1b3-4c08-a0a4-c68197924cd1-2} reads
        \begin{align*}
            \sum_{m \in \Z^d} & \widehat{\cK_t}(-(r + m N)) \sum_{q \in \Z^d \setminus \{ -m \}} \widehat{v_t}(q N -r ) \\
            &= \underbrace{\widehat{\cK}_t(- (r + m^*(r) N)) \sum_{q \in \Z^d \setminus \ens{m^*(r)}} \widehat{v}_t(q N - r)}_{:= \, A} \\
            &\qquad\qquad\qquad+ \underbrace{\sum_{m \in \Z^d \setminus \ens{m^*(r)}} \widehat{\cK_t}(-(r + m N)) \sum_{q \in \Z^d \setminus \{ -m \}} \widehat{v_t}(q N -r)}_{:= \, B} \\
            &= A + B, \numberthis\label{277aeb6d-7ab5-4951-94d8-523d924c32d9}
        \end{align*}
    and, using the well-known inequality $(a + b)^2 \le 2 (a^2 + b^2)$ we obtain
        \[ \norm{\sum_{m \in \Z^d} \widehat{\cK_t}(-(r + m N)) \sum_{q \in \Z^d \setminus \{ -m \}} \widehat{v_t}(q N -r )}^2 \oversetlab{\eqref{277aeb6d-7ab5-4951-94d8-523d924c32d9}}{=} \norm{A + B}^2 \le 2 (\norm{A}^2 + \norm{B}^2), \]
    and we need to bound both $A$ and $B$. For $A$ we have
        \begin{align*}
            \norm{A}^2 \oversetlab{\eqref{277aeb6d-7ab5-4951-94d8-523d924c32d9}}&{=} \norm{\widehat{\cK}_t(- (r + m^*(r) N)) \sum_{q \in \Z^d \setminus \ens{m^*(r)}} \widehat{v}_t(q N - r)}^2 \\
            &\le \normop{\widehat{\cK}_t(- (r + m^*(r) N))}^2 \norm{\sum_{q \in \Z^d \setminus \ens{m^*(r)}} \widehat{v}_t(q N - r)}^2 \\
            &\le \normop{\widehat{\cK}_t(- (r + m^*(r) N))}^2 \left( \sum_{q \in \Z^d \setminus \ens{m^*(r)}} \norm{\widehat{v}_t(q N - r)} \right)^2 \\
            \oversetref{Lem.}{\ref{appdx-lem:bound-sobolev}}&{\le} C_{d, s}^2 N^{-2 s} \normsobolev{v_t}^2 \normop{\widehat{\cK}_t(- (r + m^*(r) N))}^2, \numberthis\label{7c90e298-abb5-4f11-8065-a69e5221bc78}
        \end{align*}
    and, for the term $B$, we rely on the decay property of the Fourier coefficients $\widehat{\cK}_t(\cdot)$, i.e.,
        \begin{align*}
            \norm{B}^2 &= \norm{\sum_{m \in \Z^d \setminus \ens{m^*(r)}} \widehat{\cK_t}(-(r + m N)) \sum_{q \in \Z^d \setminus \{ -m \}} \widehat{v_t}(q N -r)}^2 \\
            &\le \left( \sum_{m \in \Z^d \setminus \ens{m^*(r)}} \normop{\widehat{\cK_t}(-(r + m N))} \sum_{q \in \Z^d \setminus \{ -m \}} \norm{\widehat{v_t}(q N -r)} \right)^2 \\
            \oversetlab{\eqref{97835486-5f3e-4046-85f8-fc76dff72f94}}&{\le} \left( \sum_{m \in \Z^d \setminus \ens{m^*(r)}} C_{d, \alpha} (1 + \norm{r + m N})^{-\alpha} \sum_{q \in \Z^d \setminus \{ -m \}} \norm{\widehat{v_t}(q N -r)} \right)^2 \\
            &= C_{d, \alpha}^2 \left( \sum_{m \in \Z^d \setminus \ens{m^*(r)}} (1 + \norm{r + m N})^{-\alpha} \norm{\widehat{v_t}(- (r + m^*(r) N)} \right. \\
                &\qquad\qquad\qquad + \left. \sum_{m \in \Z^d \setminus \ens{m^*(r)}} (1 + \norm{r + m N})^{-\alpha} \sum_{q \in \Z^d \setminus \{ -m, -m^*(r) \}} \norm{\widehat{v_t}(q N -r)} \right)^2,
        \end{align*}
    then, by~\Cref{appdx-lem:bound-sobolev} we have
        \[ \sum_{q \in \Z^d \setminus \{ -m, -m^*(r) \}} \norm{\widehat{v_t}(q N -r)} \le C_{d, s} N^{-s} \normsobolev{v_t}, \]
    and $\norm{\widehat{v_t}(- (r + m^*(r) N)} \le \norm{\widehat{v}_t}_{\infty}$ therefore, we obtain
        \begin{align*}
            \norm{B}^2 &\le C_{d, \alpha}^2 \left( \norm{\widehat{v}_t}_{\infty}^2 + C_{d, s}^2 N^{-2s} \normsobolev{v_t}^2 \right) \left( \sum_{m \in \Z^d \setminus \ens{m^*(r)}} (1 + \norm{r + m N})^{-\alpha} \right)^2. \numberthis\label{7c90e298-abb5-4f11-8065-a69e5221bc78-2}
        \end{align*}
    Now, since $\alpha > d$ and $\norm{r + m N} \ge N$ for all $m \in \Z^d \setminus \ens{m^*(r)}$, by~\Cref{appdx-lem:bound-sobolev} (and more precisely~\eqref{baa88d6f-23d5-4964-ae73-57895f7bd1ee}), there exists a constant $D_{d, \alpha} \ge 0$ such that
        \[ \sum_{m \in \Z^d \setminus \ens{m^*(r)}} (1 + \norm{r + m N})^{-\alpha} \oversetref{Lem.}{\ref{lem:norm-power-alpha-inequality-2}}{\le} 2 \sum_{m \in \Z^d \setminus \ens{m^*(r)}} (1 + \norm{r + m N}^{-\alpha} ) \le D_{d, \alpha} N^{-\alpha}, \]
    hence
        \[ \norm{B}^2 \le C_{d, \alpha}^2 D_{d, \alpha}^2 N^{-2 \alpha} \left( \norm{\widehat{v}_t}_{\infty}^2 + C_{d, s}^2 N^{-2s} \normsobolev{v_t}^2 \right). \]
    Next, from~\eqref{4f774a5f-e1b3-4c08-a0a4-c68197924cd1-2}, and using the bounds~\eqref{7c90e298-abb5-4f11-8065-a69e5221bc78} and~\eqref{7c90e298-abb5-4f11-8065-a69e5221bc78-2} we have
        \begin{align*}
            \normgridvec{\cE_t^{(1)}}^2 \oversetlab{\eqref{4f774a5f-e1b3-4c08-a0a4-c68197924cd1-2}}&{\le} N^d \sum_{r \in \ens{0, \dots, N-1}^d} \norm{\sum_{m \in \Z^d} \widehat{\cK_t}(-(r + m N)) \sum_{q \in \Z^d \setminus \{ -m \}} \widehat{v_t}(q N -r )}^2 \\
            \oversetlab{\eqref{7c90e298-abb5-4f11-8065-a69e5221bc78}+\eqref{7c90e298-abb5-4f11-8065-a69e5221bc78-2}}&{\le} 2 N^d \sum_{r \in \ens{0, \dots, N-1}^d} \left( C_{d, s}^2 N^{-2 s} \normsobolev{v_t}^2 \normop{\widehat{\cK}_t(- (r + m^*(r) N))}^2 \right. \\ 
            & \qquad \qquad \qquad \qquad \qquad \qquad \qquad \qquad \qquad \left. + C_{d, \alpha}^2 D_{d, \alpha}^2 N^{-2 \alpha} \left( \norm{\widehat{v}_t}_{\infty}^2 + C_{d, s}^2 N^{-2s} \normsobolev{v_t}^2 \right) \right),
        \end{align*}
    from where
        \begin{align*}
            2 N^d &\sum_{r \in \ens{0, \dots, N-1}^d} \! \left( C_{d, s}^2 N^{-2 s} \normsobolev{v_t}^2 \normop{\widehat{\cK}_t(- (r + m^*(r) N))}^2 \right.\\
            &\qquad\qquad\qquad\qquad\qquad\qquad\qquad\qquad\qquad\left. \vphantom{ C_{d, s}^2 N^{-2 s} \normsobolev{v_t}^2 \normop{\widehat{\cK}_t(- (r + m^*(r) N))}^2} + C_{d, \alpha}^2 D_{d, \alpha}^2 N^{-2 \alpha} \left( \norm{\widehat{v}_t}_{\infty}^2 + C_{d, s}^2 N^{-2s} \normsobolev{v_t}^2 \right) \right) \\
            &= 2 C_{d, s}^2 N^{d -2 s} \normsobolev{v_t}^2 \sum_{r \in \ens{0, \dots, N-1}^d} \normop{\widehat{\cK}_t(- (r + m^*(r) N))}^2 \\
            &\qquad\qquad\qquad\qquad\qquad\qquad\qquad\qquad\qquad+ 2 C_{d, \alpha}^2 D_{d, \alpha}^2 N^{2 (d - \alpha)} \left( \norm{\widehat{v}_t}_{\infty}^2 + C_{d, s}^2 N^{-2s} \normsobolev{v_t}^2 \right),
        \end{align*}
    and, thanks to~\eqref{97835486-5f3e-4046-85f8-fc76dff72f94}, we can bound the first term (the sum) as
        \begin{align*}
            2 C_{d, s}^2 N^{d -2 s} \normsobolev{v_t}^2 & \sum_{r \in \ens{0, \dots, N-1}^d} \normop{\widehat{\cK}_t(- (r + m^*(r) N))}^2  \\ & \oversetlab{\eqref{97835486-5f3e-4046-85f8-fc76dff72f94}}{\le} 2 C_{d, s}^2 C_{d, \alpha}^2 N^{d -2 s} \normsobolev{v_t}^2 \sum_{r \in \ens{0, \dots, N-1}^d} \left( 1 + \norm{r + m^*(r) N}\right)^{-2\alpha} \\
            &= 2 C_{d, s}^2 C_{d, \alpha}^2 N^{d -2 s} \normsobolev{v_t}^2 \sum_{\xi \in \ens{-\Floor{\frac{N}{2}}, \ldots, \Floor{\frac{N}{2}}}^d} \left( 1 + \norm{\xi}\right)^{-2\alpha} \\
            &\le 2 C_{d, s}^2 C_{d, \alpha}^2 N^{d -2 s} \normsobolev{v_t}^2 \underbrace{\sum_{\xi \in \Z^d} \left( 1 + \norm{\xi}\right)^{-2\alpha}}_{:= \, c_{d, \alpha} < +\infty} \\
            &= 2 c_{d, \alpha} C_{d, s}^2 C_{d, \alpha}^2 N^{d -2 s} \normsobolev{v_t}^2,
        \end{align*}
    since $2 \alpha > \alpha > d$. Thus, we have
        \begin{align*}
            \normgridvec{\cE_t^{(1)}}^2 &\le 2 c_{d, \alpha} C_{d, s}^2 C_{d, \alpha}^2 N^{d -2 s} \normsobolev{v_t}^2 + 2 C_{d, \alpha}^2 D_{d, \alpha}^2 N^{2 (d - \alpha)} \left( \norm{\widehat{v}_t}_{\infty}^2 + C_{d, s}^2 N^{-2s} \normsobolev{v_t}^2 \right) \\
            \oversetref{Lem.}{\ref{lem:sobolev-norm-and-sup-norm}}&{\le} 2 c_{d, \alpha} C_{d, s}^2 C_{d, \alpha}^2 N^{d -2 s} \normsobolev{v_t}^2 + 2 C_{d, \alpha}^2 D_{d, \alpha}^2 N^{2 (d - \alpha)} \left( 1 + C_{d, s}^2 N^{-2s} \right) \normsobolev{v_t}^2 \\
            &\le C_{d, s, \alpha}^2 N^{2 \max\ens{\frac{d}{2} - s, d - \alpha}} \normsobolev{v_t}^2,
        \end{align*}
    where $C_{d, s, \alpha}^2 := 2 \max\ens{2 c_{d, \alpha} C_{d, s}^2 C_{d, \alpha}^2, 2 C_{d, \alpha}^2 D_{d, \alpha}^2 \left( 1 + C_{d, s}^2 \right)}$, since $N \ge 1$ we have $N^{-2s} \le 1$.
\end{proof}

In the following lemma, we consider the case where we truncate the Fourier coefficients of the kernel.
\begin{lemma}[Truncating Frequency for Fourier Coefficients $\widehat{\cK}_t(\xi)$]\label{lem:bounding-E1-truncation}
    Let $d \ge 1$ be a integer and $s > \frac{d}{2}$ a real number. Suppose that there exists an integer $K_{\textnormal{cutoff}} > 0$ such that
        \[   \widehat{\cK}_t(\xi) = 0, \numberthis\label{50b70afa-6148-4eb1-9656-b7247bf5b24d} \]
    for all $\xi \in \Z^d$ with $\norm{\xi}_{\infty} \ge K_{\textnormal{cutoff}}$. Then, there exists a finite real constant $C_{d, s} \ge 0$ such that the inequality
        \[\normgridvec{\cE_t^{(1)}} \leq C_{d, s} K_{\textnormal{cutoff}}^{\frac{d}{2} + s} N^{\frac{d}{2} - s} \normop{\widehat{\cK}_t}_{\infty} \normsobolev{v_t}, \]
    holds.
\end{lemma}

\begin{proof}[Proof of~\Cref{lem:bounding-E1-truncation}]
    Note that, since $\widehat{\cK}_t(\xi) = 0$ for all $\xi \in \Z^d$ such that $\norm{\xi}_{\infty} \ge K_{\textnormal{cutoff}}$ then
        \[ \normop{\widehat{\cK}_t(\xi)} \le \normop{\widehat{\cK}_t}_{\infty} (2 K_{\textnormal{cutoff}} \sqrt{d})^{\frac{d}{2} + s} \left( 1 + \norm{\xi} \right)^{-\frac{d}{2} - s}, \numberthis\label{d36c462b-51ae-45ca-974e-f3ffa21b7c1b} \]
    since for any $\xi \in \Z^d$, either $\norm{\xi}_{\infty} \ge K_{\textnormal{cutoff}}$ for which $\widehat{\cK}_t(\xi) = 0$ and the bound~\eqref{d36c462b-51ae-45ca-974e-f3ffa21b7c1b} holds, or $\norm{\xi}_{\infty} < K_{\textnormal{cutoff}}$ which implies $\norm{\xi} \le K_{\textnormal{cutoff}} \sqrt{d}$ and, as $d \ge 1$ and $K_{\textnormal{cutoff}} \ge 1$ then $2 K_{\textnormal{cutoff}} \sqrt{d} \ge 1 + K_{\textnormal{cutoff}} \sqrt{d} \ge 1 + \norm{\xi}$ so
        \[ \normop{\widehat{\cK}_t}_{\infty} (2 K_{\textnormal{cutoff}} \sqrt{d})^{\frac{d}{2} + s} \left( 1 + \norm{\xi} \right)^{-\frac{d}{2} - s} \ge \normop{\widehat{\cK}_t}_{\infty} \ge \normop{\widehat{\cK}_t(\xi)}, \]
    as desired. Then, by~\Cref{lem:bounding-E1-polynomial-decay} with $\alpha = \frac{d}{2} + s > d$ as $s > \frac{d}{2}$ we obtain the bound
        \[ \normgridvec{\cE_t^{(1)}} \leq C_{d, s} K_{\textnormal{cutoff}}^{\frac{d}{2} + s} N^{\frac{d}{2} - s} \normop{\widehat{\cK}_t}_{\infty} \normsobolev{v_t}, \]
    for some finite constant $C_{d, s} \ge 0$; this achieves the proof of the lemma.
\end{proof}

\subsection{Bound on $\cE_t^{(2)}$}

\begin{lemma}[Bounding $\cE_t^{(2)}$]\label{lem:bounding-E2}
    The following holds:
        \begin{align}
            \normgridvec{\cE_t^{(2)}} \le N^{-\frac{d}{2}} \normgridvec{\cE_t^{(0)}} \normgridmattwo{\cK_t}. \numberthis 
        \end{align} 
\end{lemma}

\begin{proof}[Proof of~\Cref{lem:bounding-E2}]
To bound the term $\cE_t^{(2)}$ we have, by definition
    \[ \normgridvec{\cE_t^{(2)}}^2 = \sum_{x \in \T^d_N} \norm{\cE_t^{(2)}(x)}^2 = \sum_{x \in \T^d_N} \norm{\frac{1}{N^d} \sum_{y \in \T^d_N} \cK_t(x - y) \cE_t^{(0)}(y)}^2, \]
then by the triangular inequality, this gives
    \begin{align*}
         \normgridvec{\cE_t^{(2)}}^2 &\le \frac{1}{N^{2 d}} \sum_{x \in \T^d_N} \left( \sum_{y \in \T^d_N} \norm{\cK_t(x - y) \cE_t^{(0)}(y) }\right)^2 \\
         &\le \frac{1}{N^{2 d}} \sum_{x \in \T^d_N} \left( \sum_{y \in \T^d_N} \normop{\cK_t(x - y)} \norm{\cE_t^{(0)}(y)} \right)^2 \\
         &= \frac{1}{N^{2 d}} \sum_{x \in \T^d_N} \left( \left(\normop{\cK_t(\cdot)} * \norm{\cE_t^{(0)}(\cdot)} \right) (x) \right)^2,
    \end{align*}
where $*$ denotes the discrete convolution. Then, by the \emph{Young's convolution inequality} (applied on the discrete additive group $(\frac{1}{N} (\nicefrac{\Z}{N \Z})^d, +)$, see~\Cref{appdx-lem:young-convolution-inequality}) we obtain
    \begin{align*}
        \normgridvec{\cE_t^{(2)}}^2 &\le \frac{1}{N^{2 d}} \sum_{x \in \T^d_N} \left( \left(\normop{\cK_t(\cdot)} * \norm{\cE_t^{(0)}(\cdot)} \right) (x) \right)^2 \\
        &= \frac{1}{N^{2 d}} \normgridvec{\normop{\cK_t(\cdot)} * \norm{\cE_t^{(0)}(\cdot)}}^2 \\
        &\le \frac{1}{N^{2 d}} \normgridmatone{\cK_t}^2 \normgridvec{\cE_t^{(0)}}^2 \\
        &= \frac{\normgridvec{\cE_t^{(0)}}^2}{N^{2 d}} \left( \sum_{x \in \T^d_N} \normop{\cK_t(x)} \right)^2 \\
        &\le \frac{\normgridvec{\cE_t^{(0)}}^2}{N^{2 d}} N^d \sum_{x \in \T^d_N} \normop{\cK_t(x)}^2 \numberthis\label{29016375-7b43-45a6-ac18-2f3ccc939bb1} \\
        & = N^{-d} \normgridvec{\cE_t^{(0)}}^2 \normgridmattwo{\cK_t}^2,
    \end{align*}
where in~\eqref{29016375-7b43-45a6-ac18-2f3ccc939bb1} we use the Cauchy-Schwarz inequality. Hence
    \[ \normgridvec{\cE_t^{(2)}} \le N^{-\frac{d}{2}} \normgridvec{\cE_t^{(0)}} \normgridmattwo{\cK_t}. \numberthis\label{ce67aa76-3273-4e19-abc6-7f2af8b1f272-2} \]
Note that $\normgridmattwo{\cK_t} = \Theta(N^{\frac{d}{2}})$.
\end{proof}

\subsection{Bound on $\cE_{t + 1}^{(0)}$}
Finally, we upper bound $\cE_{t + 1}^{(0)}$. The following lemma encompasses both cases~\Cref{lem:bounding-E1-polynomial-decay,lem:bounding-E1-truncation}.
\begin{lemma}[Generic Bound for $\cE_{t + 1}^{(0)}$]\label{lem:generic-bound-E0}
    Suppose that there exists finite real constants $C(d, s, \cK_t), D(d, \cK_t) \ge 0$ and a real number $\beta$ (which can depends on $d$ and $s$) such that
        \begin{align}
            \normgridvec{\cE_{t}^{(1)}} &\leq C(d, s, \cK_t) N^{\beta} \normsobolev{v_t}, \numberthis\label{ce67aa76-3273-4e19-abc6-7f2af8b1f272} 
        \end{align}
    and $\normgridmattwo{\cK_t} \le D(d, \cK_t) N^{\frac{d}{2}}$. Then, the inequality
    \begin{align}
        \normgridvec{\cE_{T}^{(0)}} &\leq \begin{cases}
             N^\beta B T, & \text{if $A = 1$} \\[10pt]
             N^\beta B \dfrac{A^T - 1}{A - 1}, & \text{otherwise} \\
        \end{cases}
    \end{align} 
    holds, where 
        \[ A := L_\sigma \sup_{t \in \ens{0, 1, \ldots T - 1}} \left( \normop{W_t} +D(d, \cK_t)  \right) \,\, \text{ and } \,\,  B := L_\sigma \sup_{t \in \ens{0, 1, \ldots T - 1}} C(d, s, \cK_t) \normsobolev{v_t}, \]
    
    Note that, in general, $A > 0$ and $A \neq 1$.
\end{lemma}

\begin{proof}[Proof of~\Cref{lem:generic-bound-E0}]
     By definition of $\cE_{t + 1}^{(0)}$, we have
        \begin{align*}
            &\normgridvec{\cE_{t + 1}^{(0)}}^2  \\ & \quad \oversetlab{\eqref{b06ec7af-e2f9-46dc-ac2e-244f90be7c6d-3}}{=} \sum_{x \in \T^d_N} \left\| \sigma_t\left( W_t v_t(x) + (\cK_t * v_t)(x) +   b_t + W_t \cE_t^{(0)}(x) + \cE_t^{(1)}(x) + \cE_t^{(2)}(x) \right) \right.\\
            &\qquad\qquad\qquad\qquad\qquad\qquad\qquad\qquad\qquad\qquad\qquad\left. \vphantom{\sigma_t\left( W_t v_t(x) + (\cK_t * v_t)(x) +   b_t + W_t \cE_t^{(0)}(x) + \cE_t^{(1)}(x) + \cE_t^{(2)}(x) \right)} - \sigma_t\left( W_t v_t(x) + (\cK_t * v_t)(x) + b_t \right)\right\|^2 \\ & \quad% relou
            \oversetref{Def.}{\ref{def:neural-operator}}{\le} L_{\sigma}^2 \sum_{x \in \T^d_N} \norm{W_t \cE_t^{(0)}(x) + \cE_t^{(1)}(x) + \cE_t^{(2)}(x)}^2,
        \end{align*}
    hence
        \begin{align*}
            \normgridvec{\cE_{t + 1}^{(0)}} &\le L_{\sigma} \normgridvec{W_t \cE_t^{(0)} + \cE_t^{(1)} + \cE_t^{(2)}} \\
            &\le L_{\sigma} \left( \normgridvec{W_t \cE_t^{(0)}} + \normgridvec{\cE_t^{(1)}} + \normgridvec{\cE_t^{(2)}} \right) \\
            &\le L_{\sigma} \left( \normop{W_t} \normgridvec{\cE_t^{(0)}} + \normgridvec{\cE_t^{(1)}} + \normgridvec{\cE_t^{(2)}} \right) \\
            \oversetlab{\eqref{ce67aa76-3273-4e19-abc6-7f2af8b1f272-2} + \eqref{ce67aa76-3273-4e19-abc6-7f2af8b1f272}}&{\le} L_{\sigma} \left( \normop{W_t} \normgridvec{\cE_t^{(0)}} + C(d, s, \cK_t) N^{\beta} \normsobolev{v_t} + N^{-\frac{d}{2}} \normgridvec{\cE_t^{(0)}} \normgridmattwo{\cK_t}  \right) \\
            &= L_\sigma \left( \normop{W_t} + N^{-\frac{d}{2}} \normgridmattwo{\cK_t}   \right) \normgridvec{\cE_t^{(0)}} + L_\sigma C(d, s, \cK_t) N^{\beta} \normsobolev{v_t} \\
            \oversetlab{\eqref{ce67aa76-3273-4e19-abc6-7f2af8b1f272}}&{\le} L_\sigma \left( \normop{W_t} + D(d, \cK_t) \right) \normgridvec{\cE_t^{(0)}} + L_\sigma C(d, s, \cK_t) N^{\beta} \normsobolev{v_t}
        \end{align*}
    
    \textbf{Summing up all Bounds:} For any $t \in \ens{0, 1, \ldots, T - 1}$, let $A_t \ge 0$ the (finite) constant
        \[ A_t := L_\sigma \left( \normop{W_t} +D(d, \cK_t)  \right), \]
    and denote 
        \[ B_t := L_\sigma C(d, s, \cK_t) \normsobolev{v_t}.  \]
    With these notations, we have:
        \[  \normgridvec{\cE_{t + 1}^{(0)}} \leq A_t \normgridvec{\cE_{t}^{(0)}} + B_t N^\beta, \]
    and by induction, we it follows
        \[  \normgridvec{\cE_{T}^{(0)}} \leq \left( \prod_{j=0}^{T - 1} A_j \right) \normgridvec{\cE_{0}^{(0)}} + N^{\beta} \sum_{t = 0}^{T-1} \left( \prod_{j= t+1}^{T-1} A_j \right) B_t.\]
    Now denote:
    \begin{align}
        A := \max_{0 \leq t \leq T-1} A_j, \quad B := \max_{0 \leq t \leq T-1} A_j,
    \end{align}
    then, this yields
    \begin{align}
        \normgridvec{\cE_{T}^{(0)}} &\leq A^T \normgridvec{\cE_{0}^{(0)}} + N^{\beta} B \sum_{t = 0}^{T - 1} A^{T - 1 - t} \\
        &= \begin{cases}
            \displaystyle\normgridvec{\cE_{0}^{(0)}} + N^\beta B T, & \text{if $A = 1$} \\[10pt]
            \displaystyle A^T \normgridvec{\cE_{0}^{(0)}} + N^\beta B \dfrac{A^T - 1}{A - 1}, & \text{otherwise} \\
        \end{cases}
    \end{align}
    and, if $v_0 = v_0^N$ (meaning that the discretized and continuous architectures receive the same input function) then $\cE_0^{(0)} = 0$ therefore
    \begin{align}
        \normgridvec{\cE_{T}^{(0)}} &\leq A^T \normgridvec{\cE_{0}^{(0)}} + N^{\beta} B \sum_{t = 0}^{T - 1} A^{T - 1 - t} \\
        &= \begin{cases}
             N^\beta B T, & \text{if $A = 1$} \\[10pt]
             N^\beta B \dfrac{A^T - 1}{A - 1}, & \text{otherwise} \\
        \end{cases}.
    \end{align} 
    
\end{proof}

\newpage

\section{Proof of the Regularity Lemmas}\label{appdx-sec:regularity-lemma-proof}

\subsection{Preservation of the Input Regularity for Smooth Activations}

\begin{restate-corollary}{\ref{lem:heredity-H-s}}
    Assume $\sigma$ is a smooth function ($\mathscr{C}^\infty$) and $v_t \in H^s(\T^d)$ with $s > \frac{d}{2}$, then $v_{t+1} \in H^s(\T^d)$.
\end{restate-corollary}

\begin{proof}[Proof of~\Cref{lem:heredity-H-s}]
    Since $\widehat{\cK_t} \in L^\infty$ (\Cref{lem:fourier-ss-no}), we have:
    \begin{align*}
        \normsobolev{\cK_t * v_t}^2 &= \sum_{\xi \in \Z^d} (1 + \norm{\xi}^2)^s \norm{\widehat{\cK_t}(\xi) \widehat{v_t}(\xi)}^2 \\
        & \leq \normop{\widehat{\cK_t}}_{L^\infty}^2 \sum_{\xi \in \Z^d} (1 + \norm{\xi}^{2})^s \norm{\widehat{v_t}(\xi)}^2 \\
        & = \normop{\widehat{\cK_t}}_{L^\infty}^2 \normsobolev{v_t}^2,
    \end{align*}
    and hence $\cK_t * v_t \in H^s(\T^d)$.
    Furthermore, $W_t v_t \in H^s(\T^d)$ since $W_t$ is a constant matrix. Using \Cref{lem:Moser}, we conclude that $v_{t+1} \in H^s(\T^d)$.
\end{proof}

\subsection{Input Regularity For Non-Smooth Activations}

%\abde{First Possibility: H\"older-Sobolev embeddings}

\begin{restate-lemma}{\ref{lem:regularity-lemma}}[A Regularity Lemma]
    Let $d, H$ be positive integers, $s > \frac{d}{2}$ be a real number. Assume $v \in H^s(\T^d, \R^H)$, and the activation function $\sigma$ is globally Lipschitz, $\sigma(0) = 0$ and $\sigma'$ has bounded variations on $\R$. Then, for any $0 < t < \min \ens{ \frac{3}{2}, s }$, we have $\sigma \circ v \in H^t(\T^d)$.
\end{restate-lemma}

%\adri{WARNING!!!! LEMMA HAS TO BE PROVED, PLEASE, SOMEONE HAS TO PROVE THIS LEMMA BEFORE SUBMITTING THE PRESENT MANUSCRIPT. THIS IS ESPECIALLY IMPORTANT AS IT REPRESENTS A BUILDING BLOCK FOR OUR CLAIMS ON THE REGULARITY.}

\begin{proof}
    Without loss of generality, we can assume that $H = 1$. We distinguish three cases:
    \begin{itemize}
        \item if $0 < s \le 1$ then $\min\ens{\frac{3}{2}, s} = s$, and by~\Cref{appdx-thm:bourdaud-sobolev-s=1,appdx-thm:bourdaud-sobolev-0<s<1}, since we assume the function $\sigma$ to be globally Lipschitz and $\sigma(0) = 0$, we deduce that the \emph{left} composition operator by $\sigma$ acts on $H^s(\T^d)$ thus, $\sigma \circ v \in H^s(\T^d)$. By the embedding $H^s(\T^d) \hookrightarrow H^t(\R^d)$ from~\Cref{appdx-lem:function-spaces-embeddings} for any $0 < t < s$, the claim follows,

        \item if $1 < s < \frac{3}{2}$, then $\min\ens{\frac{3}{2}, s} = s$, and by our assumptions on $\sigma$, we know that $\sigma'$ exists a.e. on $\R$ by the Rademacher's theorem, and since $\sigma$ has bounded variations on $\R$, it follows that (in the sense of distributions) $\sigma''$ is a bounded measure on $\R$. Therefore, applying~\Cref{appdx-thm:bourdaud-sobolev-1<s<3/2}, we obtain that $\sigma \circ v \in H^s(\T^d)$ and the embedding $H^s(\T^d) \hookrightarrow H^t(\R^d)$ for any $0 < t < s$ allows to conclude,

        \item if $\frac{3}{2} \le s$, then we have $\min\ens{\frac{3}{2}, s} = \frac{3}{2}$. Now let $t \in \intoo{0}{\frac{3}{2}}$ be a real number, using the embedding $H^s(\T^d) \hookrightarrow H^t(\R^d)$ from~\Cref{appdx-lem:function-spaces-embeddings} we deduce that $v \in H^t(\T^d)$, and using the two previous points we obtain $\sigma \circ v \in H^t(\T^d)$. As this is true for every $t \in \intoo{0}{\frac{3}{2}}$, we deduce that $\sigma \circ v \in H^t(\T^d)$ for every $0 < t < \min \ens{ \frac{3}{2}, s }$, as desired.
    \end{itemize}
\end{proof}

\begin{remark}
    Notice that a function $f \colon \R \to \R$ has bounded variations if, and only if its derivative $f'$ (in the sense of distribution) is a bounded (Radon) measure on $\R$.
\end{remark}

\newpage
\section{Useful Identities and Inequalities}

\begin{lemma}[Jensen's Inequality]\label{lem:jensen-inequality}
    Let $f \colon \R^d \to \R$ be a convex function then
    \begin{enumerate}
        \item (Probabilistic Form) for any random vector $X \in \R^d$ we have
            $$\E{f(X)} \ge f\left(\E{X}\right).$$

        \item (Deterministic Form) for any vectors $v_1, \ldots, v_n \in \R^d$ and scalars $\lbd_1, \ldots, \lbd_n \in \R_+$ we have
            $$\sum_{i = 1}^n \lbd_i f(v_i) \ge f\left( \sum_{i = 1}^n \lbd_i v_i \right),$$
        provided $\lbd_i \ge 0$ for all $i \in [n]$ and $\sum\limits_{i = 1}^n \lbd_i = 1$.
    \end{enumerate}
\end{lemma}

\begin{lemma}\label{lem:norm-power-alpha-inequality}
    For any vectors $a, b \in \R^d$ and any scalar $\alpha \ge 1$ we have
        \[ \norm{a + b}^{\alpha} \le 2^{\alpha - 1} \left( \norm{a}^{\alpha} + \norm{b}^{\alpha} \right). \numberthis\label{e4d36f60-854e-4737-9f49-0ac946fb5b07} \]
\end{lemma}

\begin{proof}
    Note that the function $x \mapsto \norm{x}^{\alpha}$ is convex over $\R^d$ since $\norm{\cdot}$ is convex (as a norm) and $t \mapsto t^{\alpha}$ is convex since $\alpha \ge 1$. Hence, by the Jensen's inequality (\Cref{lem:jensen-inequality}), we have
        \[ \norm{\frac{a + b}{2}}^{\alpha} \le \frac{1}{2} \left( \norm{a}^{\alpha} + \norm{b}^{\alpha} \right), \]
    i.e.,
        \[ \norm{a + b}^{\alpha} \le 2^{\alpha - 1} \left( \norm{a}^{\alpha} + \norm{b}^{\alpha} \right), \]
    as claimed.
\end{proof}

\begin{lemma}\label{lem:norm-power-alpha-inequality-2}
    For any real numbers $x \ge 0$ and $\alpha \ge 0$, we have
        \[ (1 + x)^{\alpha} \ge C_{\alpha} \left( 1 + x^{\alpha} \right), \]
    where $C_{\alpha} = \min\ens{2^{\alpha - 1}, 1} \ge \frac{1}{2} > 0$.
\end{lemma}

\begin{proof}
    Let us define the function $\phi_{\alpha} \colon \R_+ \to \R$ as
        \[ \phi_{\alpha} \colon x \mapsto \dfrac{(1 + x)^\alpha}{1 + x^\alpha}, \]
    . The function $\phi_{\alpha}$ is continuously differentiable over $\R_+$, and its derivative is
        \[ \phi_{\alpha}'(x) = \dfrac{\alpha (1 + x)^{\alpha - 1} (1 - x^{\alpha - 1})}{(1 + x^\alpha)^2}. \numberthis\label{e2ac6203-890a-461a-9131-f326188fa262} \]
    
    From~\eqref{e2ac6203-890a-461a-9131-f326188fa262}, if $\alpha \geq 1$ then $\phi_\alpha$ is increasing on $[0, 1]$ and decreasing on $[1, \infty)$, and $\phi_\alpha(0) = \lim_{x \to \infty} \phi_{\alpha}(x) = 1$, hence $\phi_\alpha \geq 1$. Likewise, if $0 \le \alpha \le 1$, the minimum is attained at $1$, with $\phi_\alpha(1) = 2^{\alpha - 1}$, hence $\phi_\alpha \geq 2^{\alpha - 1}$. Hence the result.
\end{proof}

\begin{lemma}[{Young's Inequality (General Case;~\citet[Lemma~1.4]{Bahouri2011}}]\label{appdx-lem:young-convolution-inequality}
    Let $(G, \cdot)$ be a locally compact, topological group endowed with a left-invariant Haar measure $\mu$. If $\mu$ satisfies $\mu(A^{-1}) = \mu(A)$ for any Borel set $A \subseteq G$ then, for all real numbers $(p, q, r) \in \intff{1}{+\infty}^3$ such that
        \[ \frac{1}{p} + \frac{1}{q} = 1 + \frac{1}{r}, \]
    and any functions $f \in L^p(G, \R, \mu)$ and $g \in L^q(G, \R, \mu)$, the convolution $f * g \colon G \to \R$ (formally) defined by
        \[ f * g \colon x \mapsto \int_G f(y) g(y^{-1} \cdot x) \odif{\mu(y)}, \]
    is well-defined, and $f * g \in L^r(G, \R, \mu)$. Moreover, we have
        \[ \norm{f * g}_{L^r(G, \R, \mu)} \le \norm{f}_{L^p(G, \R, \mu)} \, \norm{g}_{L^q(G, \R, \mu)} \]
\end{lemma}

%\adri{state the special cases we use?}

\iffalse% Not used
\begin{lemma}[{Bernstein-Type Estimates~\citep[Lemma~2.1]{Bahouri2011chap2}}]\label{appdx-lem:Berteins-type-estimates}
    Let $R, r_1, r_2 > 0$, with $r_2 > r_1$ be real numbers, and define the ball $\mathcal{B} := \enstq{\xi \in \R^d}{\norm{\xi}_2 \le R}$ and the annulus $\mathcal{C} := \enstq{\xi \in \R^d}{0 < r_1 \le \norm{\xi}_2 \le r_2}$. Then, there exists a constant $C = C(\mathcal{B}, \mathcal{C}) > 0$ such that, for any multi-index $\alpha \in \N_0^d$, any real numbers $\lbd > 0$ and $1 \le p \le q \le +\infty$, and any function $u \in L^p(\R^d, \R)$, we have
    \begin{itemize}
        \item if $\supp(\widehat{u}) \subseteq \lbd \mathcal{B}$, then
            \[ \norm{\partial^{\alpha} u}_{L^p(\R^d)} \le C^{\abs{\alpha} + 1} \lbd^{\abs{\alpha} + d \left( \frac{1}{p} - \frac{1}{q} \right)} \norm{u}_{L^p(\R^d)}, \]

        \item if $\supp(\widehat{u}) \subseteq \lbd \mathcal{C}$, then 
            \[ C^{-\abs{\alpha} - 1} \lbd^{\abs{\alpha}} \norm{u}_{L^p(\R^d)} \le \norm{\partial^{\alpha} u}_{L^p(\R^d)} \le C^{\abs{\alpha} + 1} \lbd^{\abs{\alpha}} \norm{u}_{L^p(\R^d)}. \]
    \end{itemize}
\end{lemma}

\adri{say it can be extended to $\T^d$. Do we need to show it precisely or not? The proof relies on the Young's convolution inequality, so we only need to make sure it holds analogously on $\T^d$}
\fi

\begin{lemma}[An Inequality Between Sobolev Norm and Sup-Norm]\label{lem:sobolev-norm-and-sup-norm}
    For any integer $s > 0$ and any $v \in H^s(\T^d)$ we have
        \[ \normsobolev{v} \ge \norm{v}_{\infty}. \]
\end{lemma}

\begin{proof}
    By definition of the Sobolev norm, we have
        \begin{align*}
            \normsobolev{v}^2 :\!&= \sum_{\xi \in \Z^d} \left( 1 + \norm{\xi}^{2 s} \right) \norm{\widehat{v}(\xi)}^2 \\
            & \ge \sup_{\xi \in \Z^d} \left[ \left( 1 + \norm{\xi}^{2 s} \right) \norm{\widehat{v}(\xi)}^2 \right] \\
            & \ge \sup_{\xi \in \Z^d} \norm{\widehat{v}(\xi)}^2 \\
            &= \norm{v}_{\infty}^2,
        \end{align*}
    as claimed.
\end{proof}

\begin{lemma}\label{lem:op-norm-matrix-rank-1}
    For any matrix $A \in \R^{H \times H}$ of the form $A = U V^\top$ with $U$ and $V$ two column vectors in $\R^H$, the following holds:
        \[\normop{A} = \norm{U}_2 \norm{V}_2. \]
\end{lemma}

%\adri{Can we extend this lemma to matrices $A \in \R^{m \times n}$ of the form $U V^{\top}$ where $U \in \R^n$ and $V \in \R^m$? Should be straightforward I think...}

\begin{proof}[Proof of \Cref{lem:op-norm-matrix-rank-1}]
    If $U = 0$ or $V = 0$, then $A = 0$ and the identity is trivial. Assume that $U \neq 0$ and $V \neq 0$. For any $x \in \R^H$ with $x \neq 0$, we have $A x = U (V^\top x)$, so 
        \[ \norm{Ax}_2 = \abs{V^\top x} \norm{U}_2,\]
    and by \emph{Cauchy-Schwarz}, $\abs{V^\top x} \leq \norm{V}_2 \norm{x}_2$, therefore,
        \[ \dfrac{\norm{Ax}_2}{\norm{x}_2} \leq \norm{U}_2 \norm{V}_2 \implies \normop{A} \leq \norm{U}_2 \norm{V}_2. \]
    To show equality, evaluate at the unit vector $x^\star = \frac{V}{\norm{V}_2}$, then $V^\top x^\star \norm{V}_2$ and so $\norm{A x^\star}_2 = \norm{U}_2 \norm{V}_2$.
\end{proof}

\clearpage

\section{Results for the Product-Form of the SS-NO}
\label{sec:prodform}

In this section we give the results for the product-form which we omitted in the main text.
For lemma \ref{lem:fourier-ss-no}, for the product-form, we have:
    \begin{align*}
        \widehat{\cK}_t^{\textnormal{SS-NO}}(\xi) = \prod_{i=1}^d \sum_{k=1}^K c_{t, k, i} \left[ F_{+, k, i}(\xi_i) A_{k, +} + F_{-, k, i}(\xi_i) A_{k, -} \right].
    \end{align*}

Next, concerning Lemma \ref{lem:bound-ssno}, for the product-form SS-NO kernel, the multiplicative interaction between one--dimensional kernels leads to a loss of the linear scaling in $d$ observed earlier for the sum-form. Indeed, similar arguments as in the proof of \Cref{lem:bound-ssno} yield
       \begin{align*}
            \normgridmattwo{\cK_t^{\textnormal{SS-NO}}} \le (C_d N K)^{\frac{d}{2}},
       \end{align*}
where $C_d$ is the constant defined in~\Cref{lem:bound-ssno}.

Finally, on the discretization theorem, for the SS-NO in product-form, the constant $A$ in~\Cref{thm:discretization-error-ssnos} reads instead
\begin{equation}
    A = L_{\sigma} \sup_{t \in \ens{0, \ldots, T - 1}} \left( \normop{W_t} + K^{\frac{d}{2}} \, C(\cK_t)^{\frac{d}{2}} \right),
\end{equation}
i.e., the dependency is now in $K^{\frac{d}{2}}$ instead of $K d$ for the sum-form.

\clearpage
\section{Experimental Details}\label{sec:exp-details}

%\subsection{Benchmarks and Datasets}\label{sec:datasets} 

%{\color{blue}
\subsection{1D Burgers' Equation}

We consider the one-dimensional viscous Burgers' equation,
\begin{equation}
    \partial_t u(x,t) + u(x,t)\,\partial_x u(x,t) = \nu\,\partial_{xx} u(x,t), \quad (t, x) \in \intff{0}{T} \times \intff{0}{L},
\end{equation}
where $\nu > 0$ is the viscosity coefficient.

\begin{figure}[h]
    \centering
    \includegraphics[width=0.8\linewidth]{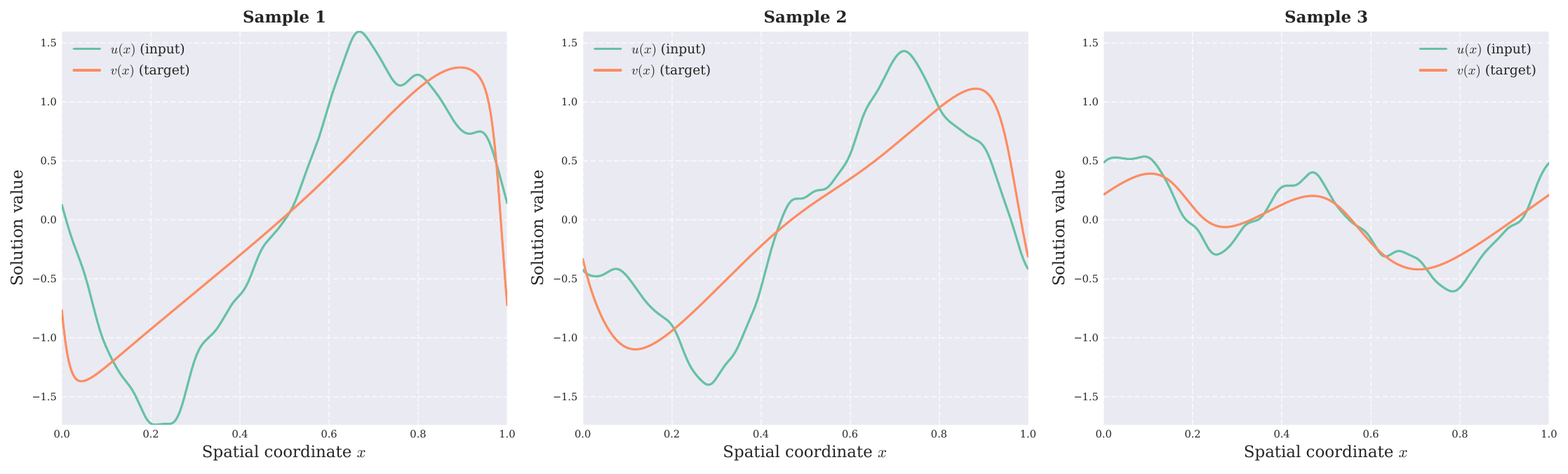}
    \caption{ Representative solutions of the one-dimensional viscous Burgers' equation with viscosity $\nu = 0.1$. Each plot corresponds to a different initial condition randomly sampled from the benchmark distribution.}
    \label{fig:appdx-burgers-1d}
\end{figure}
%}

\subsection{Gaussian Random Fields (GRFs)}\label{sec:datasets}

We consider synthetic Gaussian random field (GRF) benchmarks in one and two spatial dimensions, which provide controlled test cases for studying discretization effects independently of numerical PDE solvers. GRFs are generated on a fine spatial grid and treated as continuous functions; lower-resolution representations are obtained via uniform subsampling.

A GRF is defined on a periodic domain by prescribing its Fourier coefficients. Concretely, suppose $u \colon \T^d \to \R$ admits the Fourier representation
\begin{equation}
    u(x) = \sum_{k \in \Z^d} \widehat{u}_k \, e^{2\pi i k \cdot x},
\end{equation}
where the coefficients $\widehat{u}_k$ are independent complex-valued Gaussian random variables with zero mean and variance proportional to
\begin{equation}
    \E{\norm{\widehat{u}_k}^2} \propto \left( \norm{k}^2 + \eps \right)^{-\nicefrac{\alpha}{2}},
\end{equation}
for a smoothness parameter $\alpha > 0$ and a small regularization constant $\varepsilon > 0$. The zero-frequency mode is removed to enforce mean-zero realizations. Increasing $\alpha$ enhances the suppression of high-frequency modes, yielding smoother samples, as shown in~\Cref{fig:appdx-grf-1d,fig:appdx-grf-2d}. To target a prescribed Sobolev regularity $s$, we choose $\alpha = s + \frac{d}{2}$.

GRFs are generated on a fine uniform spatial grid and treated as realizations of continuous functions. Coarser representations are obtained via uniform subsampling of the fine-grid fields.

\begin{figure}[h]
    \centering
    \includegraphics[width=\linewidth]{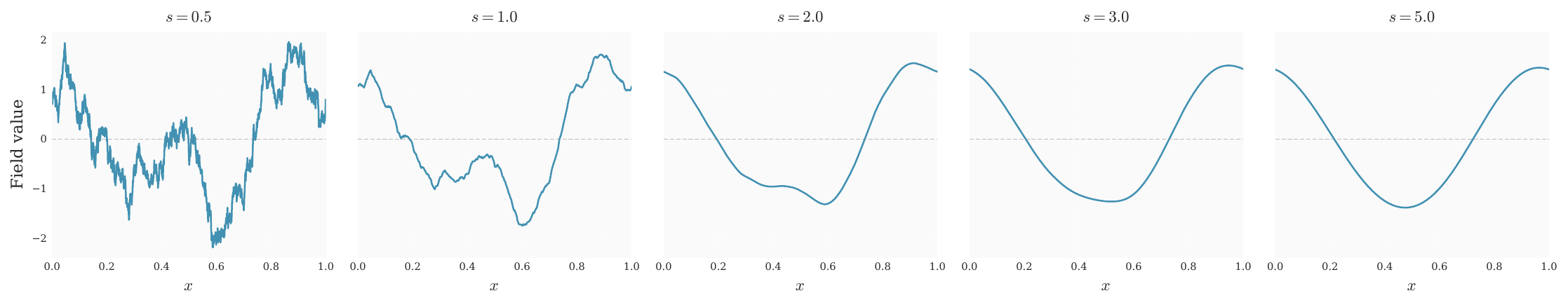}
    \caption{One-dimensional Gaussian random field realizations with varying smoothness parameters. Increasing smoothness corresponds to higher regularity and reduced high-frequency content.}
    \label{fig:appdx-grf-1d}
\end{figure}

\begin{figure}[h]
    \centering
    \includegraphics[width=\linewidth]{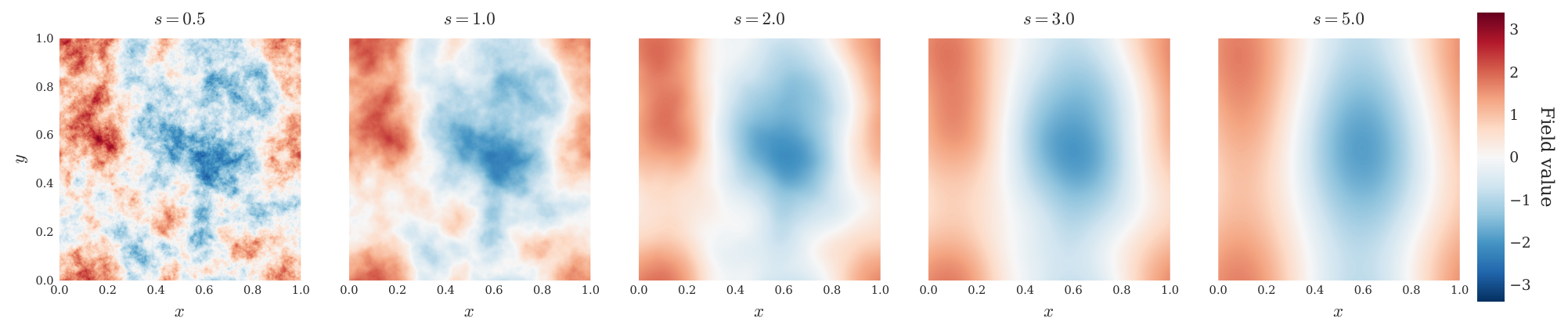}
    \caption{Two-dimensional Gaussian random field realizations with varying smoothness parameters. Fields are shown on a uniform spatial grid and share a common color scale for comparability.}
    \label{fig:appdx-grf-2d}
\end{figure}

\newpage

%{\color{blue}
\section{Additional Experiment: Discretization Scaling for \emph{Trained} SS-NOs on 1D Burgers}
\label{appdx-sec:trained-burgers-depth}

The experiments in the main text were intentionally designed to isolate the \emph{continuous-to-discrete implementation error} from training effects, by probing fixed operators under controlled changes of the spatial resolution. 
In this appendix, we further demonstrate empirically that the same phenomenon remains visible \emph{after training}, through an experiment on a standard operator-learning task: the 1D Burgers benchmark. This experiment serves two purposes. First, it tests whether the discretization-scaling behavior established in our theory is still observable for \emph{trained} SS-NOs. Second, it provides a direct empirical study of how this behavior evolves with the depth of the architecture.

\subsection{Experimental Design}

We consider the 1D SS-NO architecture and train a family of models on the Burgers dataset, varying only the depth of the network. In all cases, we use $64$ channels, $32$ poles, an output resolution of $1024$, and a \textsc{ReLU} activation. We study the four depths $T \in \{1,2,4,8\}$. The corresponding parameter counts are $20993$ for $T=1$, $33473$ for $T=2$, $58433$ for $T=4$, and $108353$ for $T=8$.

All models are trained at the finest available training resolution (no subsampling) for $30$ epochs, with batch size $32$, using \textsc{AdamW} with learning rate $10^{-2}$ and weight decay $10^{-6}$, together with a cosine annealing scheduler. The train/validation/test split is kept fixed across all runs, so that the only changing factor is the number of layers.

After training, we freeze the learned weights and analyze the discretization behavior of the resulting operator. For a given Burgers test input $u$, we first compute the prediction obtained from the full-resolution input, which we use as our reference:
\[
v^{\mathrm{ref}} := \Psi_{\theta}(u).
\]
We then subsample the same input by different factors, evaluate the same trained model on each coarsened version $u^{(\mathrm{sub})}$, and compare the resulting prediction
\[
v^{(\mathrm{sub})} := \Psi_{\theta}\bigl(u^{(\mathrm{sub})}\bigr)
\]
to the full-resolution reference. The quantity we report is therefore the relative discretization error
\[
\mathrm{RelErr}(L)
:=
\frac{\|v^{(\mathrm{sub})}-v^{\mathrm{ref}}\|_{\ell^2}}
{\|v^{\mathrm{ref}}\|_{\ell^2}},
\]
%Malheureusement nous n'avons pas eu le temps de traiter cette remarque avant la deadline des rebuttals
%\mn{est ce que cela ne montre que la architecture est stable mais  est que  ca prouve qu'on a  bien appris l'EDP?  La différence entre notre modèle à basse résolution et celui modèle à haute résolution sera tjrs un pb. ils vont vouloir voir la difference entre notre modèle à basse résolution  et solution analytique (la vraie physique de Burgers) ou element fini ou autre par Appris.  Si tu as on a graphique qui montre que le SS-NO garde une bonne erreur de test même quand on change de résolution par rapport à l'entraînement (N vers N*10 ), ca prouve que la théorie a un impact direct sur le déploiement. je pense que le reviewer penible veut savoir si ca aide de changer la façon d'entraîner mettre plus de couches ou bien si on on entraîne en X et tester en Y grand sans réentraînement}
where $L$ is the spatial resolution after subsampling. In practice, we probe the resolutions $L \in \{16, 32, 64, 128, 256, 512, 1024\}$.

\subsection{Results}

\begin{figure*}[t]
    \centering
    \includegraphics[width=\textwidth]{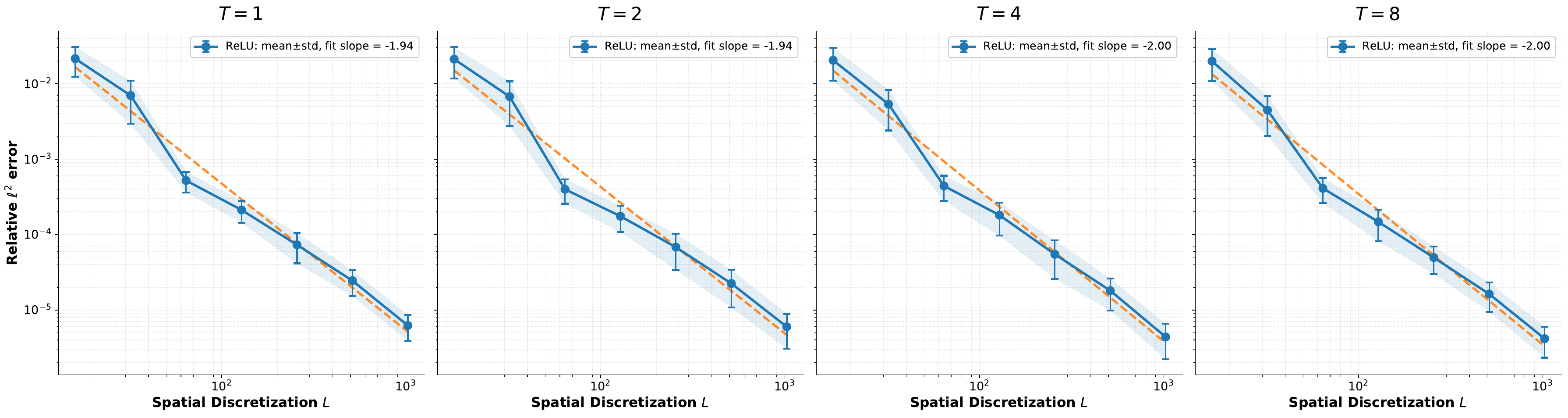}
    \caption{
    Relative discretization error for \emph{trained} 1D ReLU SS-NOs on the Burgers benchmark, for depths $T\in\{1,2,4,8\}$. In each subplot, the same trained operator is evaluated on coarsened versions of the same test input and compared to its full-resolution prediction. The plotted quantity is the relative $\ell^2$ error, shown as mean $\pm$ standard deviation over $30$ test samples, together with a log--log fit.
    }
    \label{fig:trained-burgers-depth-sweep}
\end{figure*}
%\mn{Christian Trouve les réponses sont tres bien sauf il conseille fortement de rapporter l'erreur $L^2$ relative de test réelle par rapport à la vérité terrain à pleine résolution à côté de la figure pour montrer que les modèles entraînés résolvent véritablement la tâche. Il dit il faut impérativement une conclusion pour les gens d'ML de type nous montrons la stabilité structurelle, mais on confirme par ailleurs que le modèle atteint une erreur de test de tant par rapport à la solution exacte, prouver que le 'scaling' observé s'applique à une solution physiquement vraie.  }

The results are shown in \Cref{fig:trained-burgers-depth-sweep}. 
For each depth, we plot the mean $\pm$ standard deviation over $30$ Burgers test samples, together with a log--log linear fit.
Across all four depths, the discretization error decreases regularly as the input resolution is refined, and the decay is very close to linear in log--log scale. The fitted slopes are $-1.9436$ for $T=1$, $-1.9428$ for $T=2$, $-2.0006$ for $T=4$, and $-1.9976$ for $T=8$, with corresponding $R^2$ values $0.9793$, $0.9698$, $0.9829$, and $0.9821$.

Two points are worth emphasizing. First, the scaling behavior predicted by our discretization analysis clearly remains visible \emph{after training}. This is important because it shows that the continuous-to-discrete gap studied in the paper is not merely a property of random or frozen architectures: it is still present, and still strongly structured, in a realistic learned setting. In that sense, the experiment strengthens the practical significance of the theory.

Second, the depth sweep suggests that increasing the number of layers from $T=1$ to $T=8$ does not qualitatively alter the observed law. All four models exhibit nearly the same slope, close to $-2$, and no instability or breakdown of the scaling appears as depth increases. At least on this 1D Burgers task, optimization therefore does not seem to induce any pathological amplification of the discretization error with depth. Instead, the trained models remain numerically well behaved across all tested depths.
%}

%{\color{blue}

\newpage
\section{Additional Experiment: Depth and Stability on 1D Gaussian Random Fields}
\label{appdx-sec:depth-GRF-stability}

To complement the stability experimental analysis, we additionally study how the empirical stability behavior evolves with the depth of the architecture. In particular, while \Cref{thm:global_stability_layers,thm:global_iss} highlights that the relevant stability constants accumulate across layers, it is also important to understand whether this depth dependence leads to an observable degradation in practice. To this end, we perform an experiment on untrained 1D SS-NOs driven by Gaussian random field (GRF) inputs, with the goal of directly assessing the effect of depth on perturbation amplification.

\subsection{Experimental Setup}

We follow the same general protocol as in the stability experiment of the main text, but now vary the depth of the SS-NO. More precisely, we consider 1D SS-NOs with no positional encoding, \textsc{GELU} activations, and output resolution $8192$, and study the depths
\[
T \in \{1,2,4,8,16,32\}.
\]
The inputs are sampled from 1D Gaussian random fields with smoothness parameter $s_{\mathrm{GRF}}=2.0$, that is, with Fourier decay exponent $\alpha=s_{\mathrm{GRF}}+\frac12=2.5$. For each depth $T$, we evaluate two complementary empirical stability quantities.

\paragraph{Perturbation Response.}
For each of $N_{\mathrm{GRF}}=20$ random GRF inputs $v$, we generate $N_{\mathrm{DIRS}}=20$ independent random unit directions $\xi$, and evaluate the output discrepancy
\[
\bigl\|\mathcal{L}_N^{(T)}(v+\varepsilon \xi)-\mathcal{L}_N^{(T)}(v)\bigr\|_{\ell^2}
\]
over perturbation scales $\varepsilon \in \{0,0.025,\dots,0.8\}$. The reported curves correspond to the mean $\pm$ standard deviation over all sampled inputs and directions.

\paragraph{Empirical Lipschitz factor.}
To summarize the global sensitivity of the full $T$-layer architecture, we also estimate an empirical Lipschitz constant by sampling $200$ random GRF pairs $(x_1,x_2)$ and computing
\[
\max \frac{\|\mathcal{L}_N^{(T)}(x_1)-\mathcal{L}_N^{(T)}(x_2)\|_{\ell^2}}{\|x_1-x_2\|_{\ell^2}}.
\]
In addition, we report the mean output perturbation error at the largest tested scale $\varepsilon=0.8$, in order to compare a worst-case pairwise sensitivity estimate with the average response observed under structured perturbations.

\subsection{Results}

\begin{figure*}[t]
    \centering
    \includegraphics[width=\textwidth]{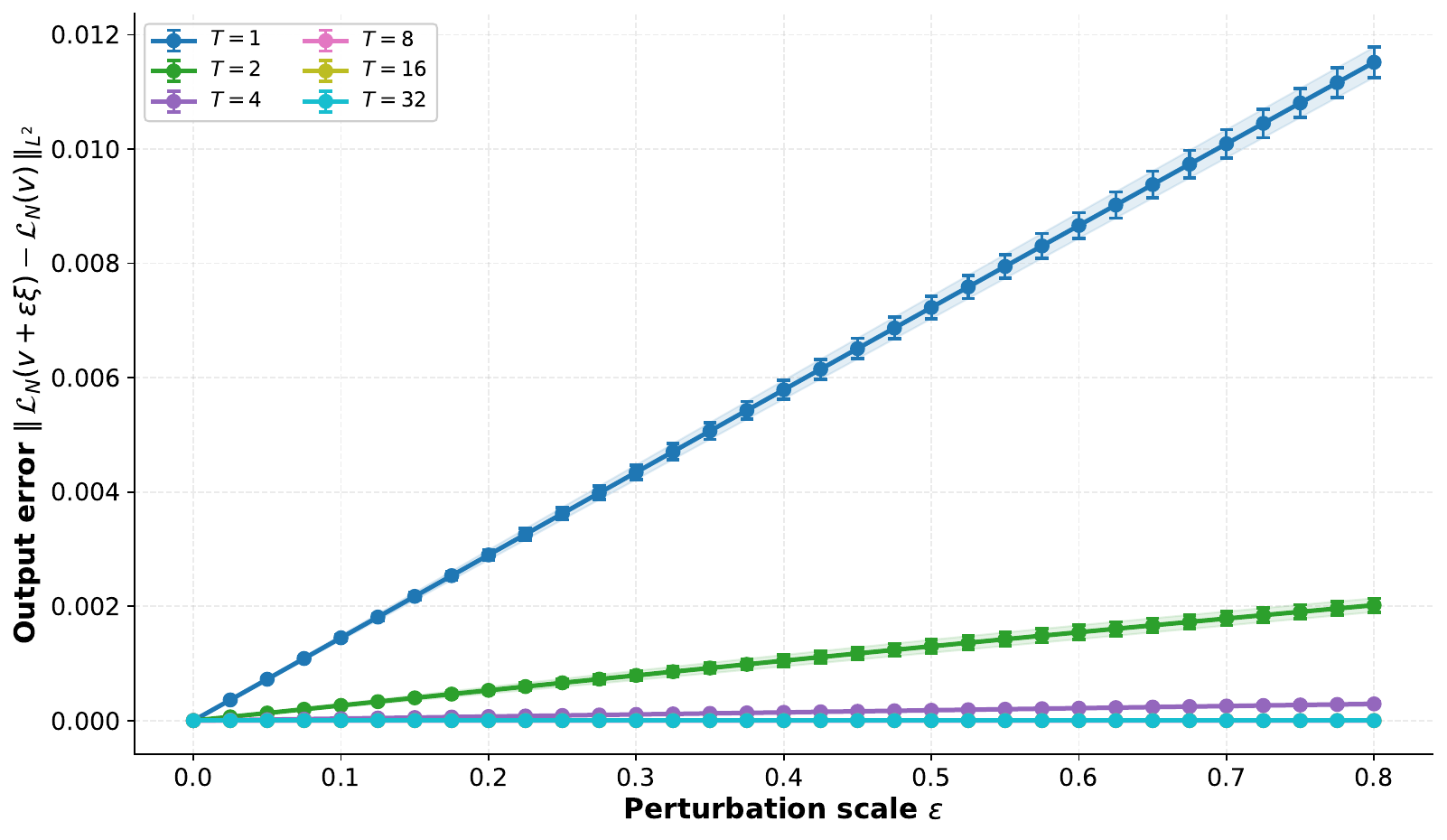}
    \caption{
    Output perturbation error for 1D SS-NOs on GRF inputs, for depths $T\in\{1,2,4,8,16,32\}$. The models use output dimension $8192$, \textsc{GELU} activations, and no positional encoding. For each depth, we report the mean $\pm$ standard deviation over $N_{\mathrm{GRF}}=20$ random fields and $N_{\mathrm{DIRS}}=20$ random perturbation directions per field, across perturbation scales $\varepsilon\in[0,0.8]$.
    }
    \label{fig:grf-output-error-depth}
\end{figure*}

\begin{figure*}[t]
    \centering
    \includegraphics[width=0.74\textwidth]{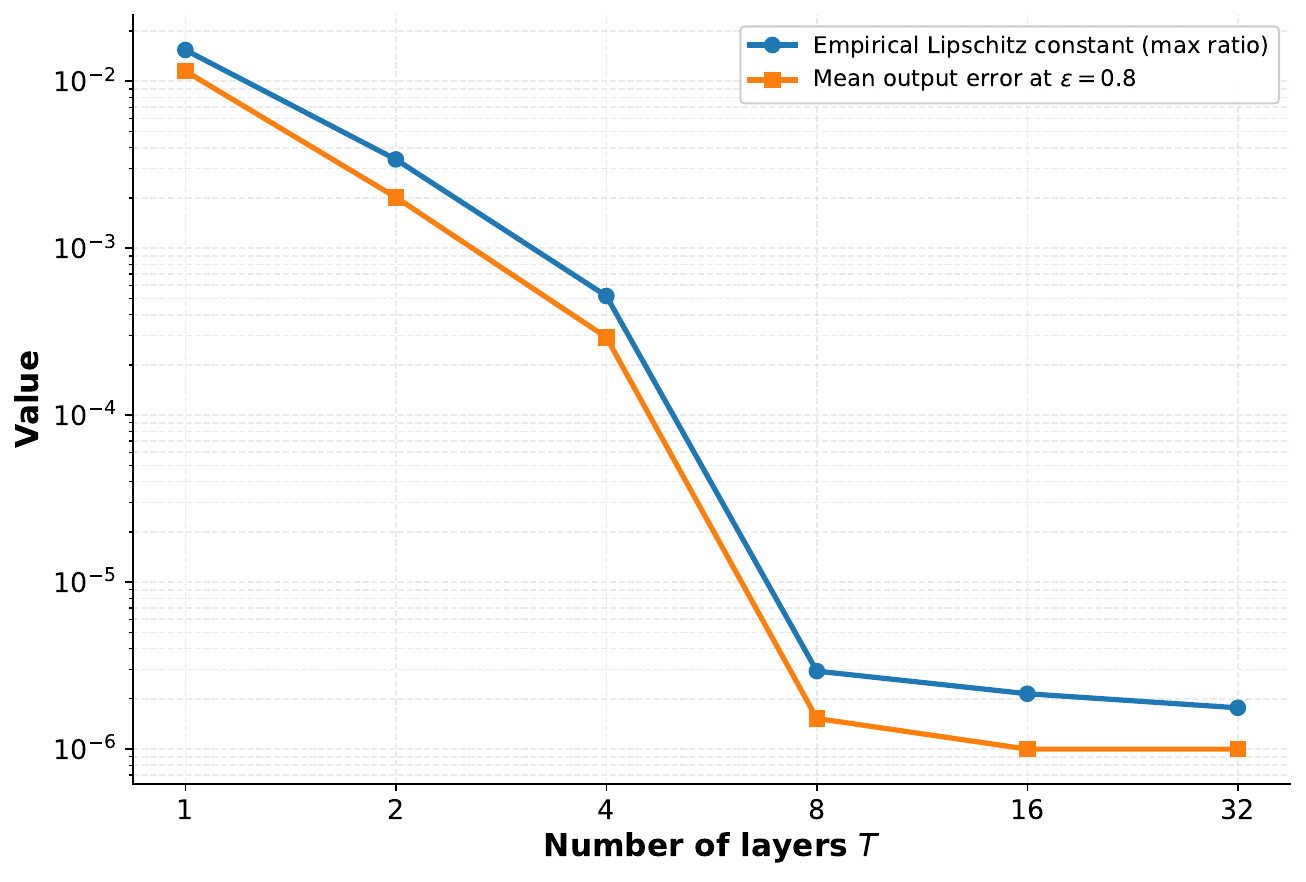}
    \caption{
    Empirical stability versus depth $T$ on a log--log scale. We plot the estimated Lipschitz factor of the full $T$-layer network (maximum output-to-input perturbation ratio over $200$ random GRF pairs) together with the mean output perturbation error at $\varepsilon=0.8$. The same model family and hyperparameters as in \Cref{fig:grf-output-error-depth} are used.
    }
    \label{fig:grf-lipschitz-depth}
\end{figure*}

The results are shown in \Cref{fig:grf-output-error-depth,fig:grf-lipschitz-depth}. Several observations emerge. 

First, \Cref{fig:grf-output-error-depth} shows that, for all tested depths, the output perturbation grows in a nearly linear fashion with the perturbation amplitude $\varepsilon$, which is the qualitative behavior expected from the Lipschitz-type stability bounds derived in the main text. Importantly, increasing the depth from $T=1$ to $T=32$ does not lead to any visible instability or abrupt amplification regime. On the contrary, the perturbation response remains smooth and well controlled throughout the tested range.

Second, \Cref{fig:grf-lipschitz-depth} indicates that both the empirical Lipschitz factor and the mean perturbation response at $\varepsilon=0.8$ decrease as the depth grows, rather than increasing explosively. In this experiment, deeper untrained networks therefore appear \emph{less} sensitive to perturbations. A natural explanation is that, at random initialization, the successive layers remain in a contractive regime, due to the effect of moderate weight magnitudes and the regularizing action of repeated kernel mixing. In particular, although our theoretical upper bounds accumulate layerwise and are therefore necessarily conservative, the global network may operate far from the worst-case regime.

Overall, this additional experiment supports the main theoretical insight of the paper from a complementary angle. The stability bounds in the main text identify the mechanisms through which perturbations may propagate across layers, while the present empirical study shows that, in a representative regime of random SS-NO initializations, this accumulation remains practically benign even for fairly deep architectures. This provides further evidence that the discrete stability framework developed here is not only mathematically meaningful, but also consistent with the observed behavior of the architecture in depth.

%}

\clearpage

%{\color{blue}

\section{Additional Stress Test: Under-Resolved Oscillatory Inputs}
\label{appdx-sec:underresolved-stress-test}

Beyond the empirical well-behaved discretization analysis of the main text, we include here a deliberately \emph{under-resolved} stress test designed to exhibit a regime where discretization error becomes large in a way consistent with our theory. This experiment addresses the question of whether the continuous-to-discrete analysis is only descriptive of well-behaved settings, or whether it also predicts degradation when the operator is evaluated outside its numerically resolvable regime.

\subsection{Experimental design}

We consider a fixed 1D SS-NO with \textsc{GELU} activation and output resolution $8192$, with randomly initialized parameters kept frozen throughout the experiment. No training is involved: the goal is to isolate the numerical implementation gap of the operator itself.

For each frequency
\[
k \in \{4,8,16,32,64,128\},
\]
we construct oscillatory inputs of the form
\[
u_k(x) = \sin(2\pi kx + \phi),
\]
where $\phi$ is a random phase. For each $k$, we evaluate the same SS-NO on a very fine grid, which serves as a reference discretization, and then re-evaluate it on coarser grids of sizes
\[
L \in \{16,32,64,128,256,512,1024\}.
\]
The discretization error is measured as the relative $\ell^2$ difference between the coarse-grid evaluation and the fine-grid reference. We average over multiple random phases.

This setup is intentionally chosen to probe a simple failure mode predicted by the theory: as the input frequency $k$ increases, the relevant oscillations become harder to resolve on a coarse grid. Hence one should expect the coarse-grid implementation to depart more strongly from the fine-grid one when $L$ is too small relative to $k$. In each subplot we indicate the rough Nyquist threshold $L \approx 2k$, which marks the minimal sampling scale required to begin resolving a sinusoid of frequency $k$.

\subsection{Results}

\begin{figure*}[t]
    \centering
    \includegraphics[width=\textwidth]{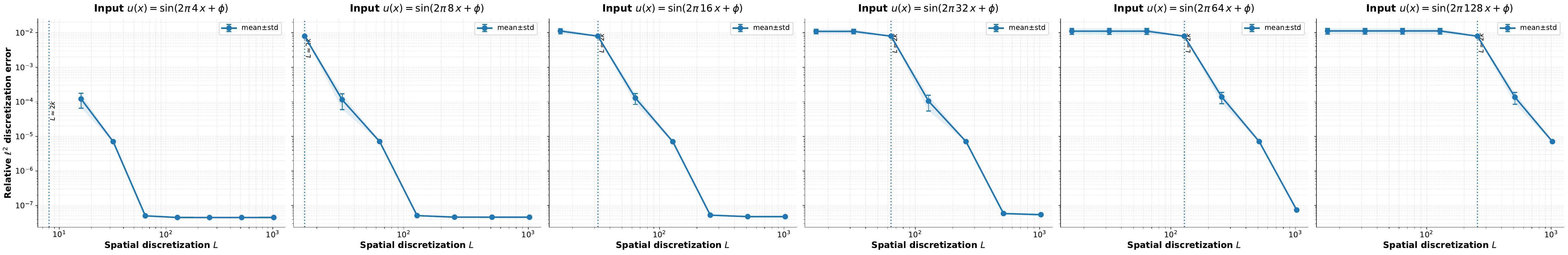}
    \caption{
    Stress test with increasingly oscillatory inputs. For each input frequency $k \in \{4,8,16,32,64,128\}$, we evaluate the same fixed 1D SS-NO on grids of sizes $L \in \{16,32,64,128,256,512,1024\}$ and compare the result to a fine-grid reference. The dashed vertical line indicates the rough Nyquist threshold $L \approx 2k$. As $k$ increases, a clear under-resolved regime appears: the discretization error remains large and may plateau when the grid is too coarse to resolve the input oscillations, before recovering once $L$ becomes sufficiently large.
    }
    \label{fig:underresolved-stress-test}
\end{figure*}

The results are shown in \Cref{fig:underresolved-stress-test}. For low-frequency inputs, such as $k=4$ and $k=8$, the discretization error already decays cleanly as the evaluation resolution increases. By contrast, for larger frequencies the coarse-grid regime becomes markedly worse. For $k=16$ and $k=32$, the first resolutions exhibit much larger errors before the decay regime is recovered. This phenomenon becomes even clearer for $k=64$ and $k=128$: the error essentially remains on a plateau for several coarse resolutions, and only begins to decrease once the grid size crosses the scale needed to resolve the oscillations.

\subsection{Interpretation}

This experiment provides a concrete example of a discretization failure regime that is qualitatively predicted by our theory. The point is not that the SS-NO becomes unstable in an optimization sense, but rather that its \emph{discrete realization ceases to be faithful} when the relevant oscillations are too fine for the grid. In the well-resolved regime, the discrete operator tracks the fine-grid reference accurately and the error decreases with resolution; in the under-resolved regime, the implementation gap becomes large and may even plateau before eventually decaying again.

We stress that this is precisely the type of phenomenon our analysis is meant to capture. The continuous-to-discrete error is small only when the evaluation grid is sufficiently fine relative to the frequency content that the operator must process. Hence \Cref{fig:underresolved-stress-test} shows that our theory is not merely a post-hoc description of favorable cases: it also correctly anticipates a failure mode in which coarse discretization is no longer able to realize the intended operator faithfully.

% 58 pages xD oui
%}

%%%%%%%%%%%%%%%%%%%%%%%%%%%%%%%%%%%%%%%%%%%%%%%%%%%%%%%%%%%%

\end{document}